\documentclass[12pt,letterpaper]{article}
\usepackage[T1]{fontenc}
\usepackage{lmodern}
\usepackage{graphicx,natbib,microtype,parskip,tikz,float,caption,setspace}
\usetikzlibrary{arrows.meta,decorations.pathreplacing}
\usepackage{amsmath,amssymb,amsthm,mathrsfs,mathtools,booktabs,longtable,array,enumitem,xurl}
\usepackage{hyperref}
\setcitestyle{authoryear,round,semicolon,aysep={}}
\hypersetup{hidelinks}
\newcommand{\anon}{1}
\allowdisplaybreaks
\newtheorem{theorem}{Theorem}[section]
\newtheorem{lemma}[theorem]{Lemma}
\newtheorem{proposition}[theorem]{Proposition}
\newtheorem{corollary}[theorem]{Corollary}
\theoremstyle{definition}\newtheorem{assumption}[theorem]{Assumption}
\newtheorem{definition}[theorem]{Definition}
\theoremstyle{definition}
\newtheorem*{remark*}{Remark}
\newcommand{\R}{\mathbb R}\newcommand{\E}{\mathbb E}
\newcommand{\cF}{\mathcal F}

\newcommand{\Holder}{\mathcal H_d^s(H)}
\DeclareMathOperator{\ReLU}{ReLU}

\newcommand{\papertitle}{Statistical Gains from Looped Estimation under Parameter Budgets}

\if0\anon
  \hypersetup{pdfauthor={}}
\else
  \hypersetup{pdfauthor={Xinyu Tian and Xiaotong Shen}}
\fi

\hypersetup{pdftitle={\papertitle}}
\title{\bfseries\papertitle}
\if0\anon
  \author{}
\else
  \author{Xinyu Tian \qquad Xiaotong Shen\\[0.5ex]
  School of Statistics, University of Minnesota}
\fi
\date{}
\begin{document}
\setstretch{1}
\maketitle
\begin{abstract}
Memory constraints in artificial intelligence motivate accurate function
approximation with fewer parameters. We study looping, which repeatedly
composes one update function with shared parameters; each output
becomes the next input. A looped Transformer, for example, reuses one block,
whereas its conventional untied counterpart uses separately parameterized
blocks. We compare their parameter requirements for a given worst-case approximation
accuracy, or equivalently, their approximation accuracy under a common
budget limiting distinct trainable coefficients. We then ask whether this representational parsimony
improves statistical accuracy. For general likelihood models, we establish
an upper squared Hellinger risk bound for looped sieve maximum likelihood
and a minimax lower bound for the jointly tuned untied family. Further loop
iterations improve the approximation bound without adding parameters, while increasing computation
and the fitted-class complexity bound. For targets of known H\"older smoothness,
looped residual feedforward networks and post-layer-normalized
Transformers attain the minimax polynomial rate up to logarithmic factors
with a fixed number of bounded real parameters. At sufficiently large fixed
budgets, looped worst-case risk vanishes while optimal worst-case untied
risk remains bounded away from zero. The loop-to-untied risk ratio also tends to zero
under specified growing-budget conditions. Regression, binary response,
and energy-based generative models illustrate the theory.
\end{abstract}
\medskip
\setstretch{1}
\renewcommand{\topfraction}{0.90}
\renewcommand{\bottomfraction}{0.85}
\renewcommand{\textfraction}{0.08}
\renewcommand{\floatpagefraction}{0.85}
\setcounter{topnumber}{3}
\setcounter{bottomnumber}{2}
\setcounter{totalnumber}{4}
\def\loopedmainloaded{}
% !TEX root = main.tex
\section{Introduction}

Memory constraints in artificial intelligence motivate learning with fewer
trainable parameters. During training, parameters, optimizer states, and
activations compete for memory \citep{rajbhandari2020}; during inference,
parameter transfers and communication can limit prediction throughput
\citep{pope2023}. Looped architectures offer one response to these resource
demands by repeatedly applying a shared module, increasing computational
depth without adding distinct trainable parameters. This approach is used in
Universal Transformers \citep{dehghani2019} and recurrent-depth language
models \citep{geiping2025}, and enables latent reasoning in ByteDance's Ouro
models \citep{zhu2025ouro}. More recently, \citet{wang2026smelt} study
SMELT, a sparse mixture-of-experts Transformer,
under closely matched parameter, computation, and key--value cache budgets.
Looped computation has even been reported in OpenAI's Astra model
\citep{efrati2026astra,raschka2026astra}.
These developments motivate the question of whether repeated composition
with shared parameters can improve approximation accuracy and reduce
statistical risk under a common parameter budget.\looseness=-1

To study these designs statistically, we view them as function
representations. Each representation composes an input function, $T$
update functions, and an output function. Each successive update is one
iteration. A looped representation uses the same update function and
parameters at every iteration; its conventional untied counterpart uses
separately parameterized update functions.
In neural networks, each update function corresponds to a block, so $T$
measures effective depth in blocks. At fixed block width, increasing this
depth reuses the loop's fitted coefficients, whereas the untied network
requires additional coefficients to be estimated.\looseness=-1

To reflect the demand for parameter storage, we impose a parameter
budget, an upper bound on the number of distinct trainable scalar
coefficients. Statistically, this is a way to study representational
parsimony. Over a common target class, we can ask how many coefficients
are needed to attain a prescribed worst-case approximation accuracy or,
equivalently, how accurately each representation can approximate under a
given parameter budget. We adopt the latter view to compare looped and
untied representations. Statistical risk depends on both this approximation
accuracy and the estimation error incurred in fitting the coefficients to data.\looseness=-1

\paragraph{Main finding.}
With a sufficiently large fixed budget, a looped estimator can attain
vanishing worst-case risk. At the same budget, every estimator in the
corresponding untied family has worst-case risk bounded away from zero.
To see the approximation--estimation balance, fix either the residual
feedforward network (FFN) or post-layer-normalized (post-LN) Transformer
family, a likelihood example, and a
H\"older target class of known smoothness $s$ in dimension $d$. Write $a=2s/d$.
For this overview, fix the state or token embedding dimension and choose
the largest admissible feedforward hidden width under budget $B$.
Let $R_L(n,B,T)$ be the worst-case squared Hellinger risk of the resulting
loop sieve maximum likelihood estimator and $R_U(n,B)$ the optimal
worst-case risk among estimators valued in the budget-$B$ untied family.
Our main results give, under suitable conditions,
\[
\begin{aligned}
R_L(n,B,T)&\lesssim
\left\{\frac{\log(eBT)}{BT}\right\}^{a}
+\frac{B\{\log(en)+T\log(eB)\}}{n},\\
R_U(n,B)&\gtrsim\{B^2\log(eB)\}^{-a}.
\end{aligned}
\]
The implicit constants depend on the fixed target-class, likelihood, and
remaining architecture parameters, but not on $n,B$, or $T$.
The bounds follow from Theorem~\ref{thm:model-risk} and
Proposition~\ref{prop:budget-lower}.

For the untied networks considered here, increasing widths or adding blocks
increases the parameter count. Improving approximation through a larger
network therefore incurs the estimation cost of a larger parameterization.
The loop upper bound expresses a different balance: at fixed $B$, its
approximation (first term) and estimation (second term) both vary with $T$,
while the number of fitted coefficients stays fixed.
Additional iterations reduce the approximation bound but increase the
complexity bound for the fitted function class and hence the estimation
term. Sieve-likelihood theory \citep{shenwong1994,wongshen1995} converts
this balance into a risk bound. For any sufficiently large fixed $B$, choosing \(T_n=\lceil n^{1/(a+1)}\{\log(en)\}^{a/(a+1)}\rceil\) gives \(R_L(n,B,T_n)\lesssim_B \{\frac{\log(en)}{n}\}^{2s/(2s+d)}\). The looped estimator therefore attains the classical H\"older minimax rate
up to a logarithmic factor \citep{stone1982}. By contrast, increasing
untied depth requires additional coefficients. Under a fixed budget, even
the jointly tuned untied family retains positive worst-case approximation
error and risk. This lower bound holds uniformly over all admissible
choices of width and depth in the untied family under the same parameter
budget, including wide, shallow networks and narrow, deep networks.
Consequently, $R_L(n,B,T_n)/R_U(n,B)\to0$ for every
sufficiently large fixed budget admitting the loop, as stated precisely in
Theorem~\ref{thm:fixed-budget}.

The same bound displays the tradeoff between budget and iteration count.
With the state or embedding dimension fixed and logarithmic factors
suppressed, the approximation and estimation terms scale as $(BT)^{-a}$
and $BT/n$. A larger budget permits a wider shared block and comparable
approximation with fewer iterations. Theorem~\ref{thm:growing-budget}
extends the comparison to sieves in which both state and hidden widths grow.
Under its stated schedule, it attains the minimax polynomial rate with
$T_n$ of order $n^{d/(2s+d)}/B_n$, up to logarithmic factors.

The untied lower bound gives a contrasting parameter requirement. Attaining
the minimax polynomial rate requires its budget to grow at least as
$n^{d/(4s+2d)}$, up to logarithmic factors; this is a necessary growth order.
For $B_n\asymp n^{\gamma_B}$ with $0<\gamma_B<d/(4s+2d)$,
Theorem~\ref{thm:growing-budget} shows that the loop-to-untied risk ratio
still tends to zero. Looping thus attains comparable statistical accuracy with substantially
fewer parameters, at the cost of additional iterations.

We obtain the required approximation rates by encoding local Taylor
information in bounded real coefficients and repeatedly applying a shared
update to decode and evaluate it. Combining these constructions with
likelihood-specific risk bounds yields the comparisons for Gaussian and
Laplace regression, binary response, and energy-based generative models.

\paragraph{Relation to existing work.}
Parameter reuse is present in Universal Transformers \citep{dehghani2019},
deep equilibrium models \citep{bai2019}, looped residual language models
\citep{ngwang2024}, and recurrent-depth models \citep{geiping2025}, in contrast
to conventional separately parameterized Transformer blocks
\citep{vaswani2017}. Computational analyses study program execution
\citep{giannou2023}, iterative in-context learning
\citep{yang2024,gatmiry2024,chen2025}, time-invariant reasoning generators
\citep{joshi2025}, and programmable looped rectified linear unit (ReLU) networks \citep{liang2025}.
These studies motivate an explicit comparison of statistical risk under a
common parameter constraint.\looseness=-1

For smooth-function estimation, approximation and statistical theory typically
balance increasing width or depth against the complexity of a growing neural
sieve \citep{yarotsky2017,yarotsky2018,lu2021,schmidthieber2020}. Recurrent
networks also attain H\"older approximation and near-minimax least-squares
rates \citep{jiao2024}, while depth-separation results identify representational
gains under prescribed resource constraints
\citep{eldanshamir2016,telgarsky2016}. Work on reuse itself gives quantitative
approximation bounds for repeated composition of a fixed-size ReLU network
\citep{zhang2023}, fixed-dimensional recurrence for univariate continuous
functions \citep{abadie2026}, and looped-Transformer approximation of continuous
permutation-equivariant maps \citep{xu2025}. Existing complexity bounds quantify the effects of depth
\citep{bartlett2019} and recurrence length \citep{chen2020} alongside parameter
count. Model-selection theory accounts for the additional complexity of
choosing among candidate models \citep{barronbirgemassart1999}. The contribution of this article is to establish when the approximation
advantage of parameter reuse yields a statistical risk advantage under a
common parameter budget. We construct increasingly accurate approximations
with a fixed number of bounded coefficients and control estimation over the
entire fitted loop class. The resulting worst-case loop risk is asymptotically smaller than
the best worst-case risk over the corresponding untied family, even when
that family allows joint model selection under the same budget. To our knowledge, our fixed- and
growing-budget theorems (Theorems~\ref{thm:fixed-budget}
and~\ref{thm:growing-budget}) provide the first statistical risk separation
between looped estimation and the corresponding jointly tuned untied family
under a common parameter budget.\looseness=-1

\paragraph{Organization.}
Section~\ref{sec:parameter-budget} formulates looped estimation under a
parameter budget and develops its general likelihood-risk comparison with the
untied benchmark.
Section~\ref{sec:examples-neural} verifies its ingredients for the four
likelihood models and two neural realizations.
Section~\ref{sec:budget-consequences} derives the fixed-budget separation and
its growing-sieve extension. Section~\ref{sec:discussion} discusses the
statistical interpretation and remaining resource questions. Proofs and
constructions are in the supplementary material; Supplementary
Appendix~\ref{sec:numerical} presents the numerical studies and experimental details.

\section{Looped estimation under a parameter budget}
\label{sec:parameter-budget}
This section defines looped and untied estimators under a common parameter
budget and compares their likelihood risks. The budget counts distinct
trainable coefficients, while the iteration count records how often the
update function is composed. We first define the candidate function classes,
then state a general sieve likelihood bound and apply it to compare
looped and untied risks under the same budget.

\paragraph{Notation.}
Let $\mathcal X$ be the input space, $\mathcal V$ a class of functions on
$\mathcal X$, and $\cF\subseteq\mathcal V$ the fixed target family containing
the unknown $f_0$. An estimator is a measurable, data-dependent choice from
a candidate function class.
Write $\mathbb N=\{1,2,\ldots\}$, and let $\log$ denote the natural
logarithm. For real $x,y$, write $x\vee y=\max\{x,y\}$ and
$x\wedge y=\min\{x,y\}$. For functions, $\|\cdot\|_2$ denotes the $L^2$ norm under the specified
measure; for vectors, it denotes the Euclidean norm.

Unless explicitly fixed, $C,c>0$ denote constants that may change from line
to line. They may depend on the fixed target family and the stated likelihood
and computational-class constants, and are uniform in the sample size,
varying parameter budgets and model tuning parameters, and target function. For positive quantities $x,y$,
write $x\lesssim y$ if $x\le Cy$ for such a constant, $x\gtrsim y$ if
$y\lesssim x$, and $x\asymp y$ if both inequalities hold. Subscripts on
these comparison symbols indicate the permitted dependence of the constants.

\subsection{Looped and untied representations}
\label{sec:repeated-classes}
A candidate function is constructed in three stages: an input function forms
an intermediate representation, successive update functions transform it,
and an output function produces a scalar value. A looped representation
shares the update parameters, whereas its conventional untied counterpart
uses separate parameters at each update (Figure~\ref{fig:loop-architecture}).

The model tuning parameter $\lambda$ specifies the forms and dimensions of
the input, update, and output functions and their intermediate
representation space $\mathcal Z_\lambda$, also called the state space.
Neural networks are examples of this general framework. For a fixed
feedforward block design, $\lambda=(q,r)$ specifies the dimension $q$ of
the vector passed between updates and the hidden-layer width $r$, with
$\mathcal Z_\lambda=\mathbb R^q$.
For example, with $q=16$ and $r=64$, each block maps a 16-dimensional
input through 64 hidden neurons to a 16-dimensional output.

With these forms and dimensions fixed, let $\eta$, $\theta$, and $\omega$
denote the coefficient vectors of the
input, update, and output functions, with finite-dimensional parameter
spaces $\Theta_\eta(\lambda)$, $\Theta_\theta(\lambda)$, and
$\Theta_\omega(\lambda)$, respectively. The three functions are
$E_\eta:\mathcal X\to\mathcal Z_\lambda$,
$\Phi_\theta:\mathcal Z_\lambda\to\mathcal Z_\lambda$, and
$O_\omega:\mathcal Z_\lambda\to\mathbb R$.
The iteration count $T$ is the number of updates applied to the
intermediate representation and is specified separately from $\lambda$.
Together, $\lambda$ and $T$ determine the candidate function classes; the
coefficient vectors are fitted to data within the chosen class.

\begin{definition}[Looped and untied representations]\label{def:loop-untied}
Fix a model tuning parameter $\lambda$ and an iteration count $T\ge1$.
Both representations compose an input function, $T$ update functions,
and an output function. For $x\in\mathcal X$, define \(f_L(x;\eta,\theta,\omega) =O_\omega\bigl(\Phi_\theta^{\circ T}(E_\eta(x))\bigr)\), \(f_U(x;\eta,\theta_1,\ldots,\theta_T,\omega) =O_\omega\bigl(\Phi_{\theta_T}\circ\cdots\circ \Phi_{\theta_1}(E_\eta(x))\bigr)\). Here $\Phi_\theta^{\circ T}$ denotes $T$-fold composition of the same
update function. The looped representation shares $\theta$ across these
updates; the untied representation uses separate vectors
$\theta_1,\ldots,\theta_T$. Both use the same input and output function
families.

The classes $\mathcal F_L(\lambda,T)$ and $\mathcal F_U(\lambda,T)$
comprise these functions as their coefficients range over
$\Theta_L(\lambda)=\Theta_\eta(\lambda)\times\Theta_\theta(\lambda)
\times\Theta_\omega(\lambda)$ and
$\Theta_U(\lambda,T)=\Theta_\eta(\lambda)\times
\prod_{t=1}^T\Theta_\theta(\lambda)\times\Theta_\omega(\lambda)$,
respectively.
\end{definition}

Let $\dim\Theta_\eta(\lambda)$, $\dim\Theta_\theta(\lambda)$, and
$\dim\Theta_\omega(\lambda)$ denote the numbers of trainable scalar
parameters in the input, update, and output functions, respectively. The looped
representation fits one shared update vector, whereas the untied
representation fits $T$ separate update vectors. Separate parameter
slots count individually even when their fitted values coincide.
The total parameter counts are therefore \(B_L(\lambda) =\dim\Theta_\eta(\lambda)+\dim\Theta_\theta(\lambda)+\dim\Theta_\omega(\lambda)\), \(B_U(\lambda,T) =\dim\Theta_\eta(\lambda)+T \dim\Theta_\theta(\lambda)+\dim\Theta_\omega(\lambda)\). A parameter budget $B$ bounds the total number of trainable scalar
parameters: an admissible looped or untied representation satisfies
$B_L(\lambda)\le B$ or $B_U(\lambda,T)\le B$, respectively, and need not
exhaust the budget. For fixed $\lambda$, $B_L(\lambda)$ is independent of
$T$ because further iterations reuse the same update parameters;
each new untied update adds $\dim\Theta_\theta(\lambda)$
parameters.

\begin{definition}[Looped estimation under a parameter budget]
\label{def:budgeted-loop-estimation}
For a budget $B$, \emph{looped estimation under $B$} chooses
model tuning parameters $(\lambda,T)$ with $B_L(\lambda)\le B$ and returns a measurable,
data-dependent element of $\mathcal F_L(\lambda,T)$. A single update parameter vector is
fitted and shared across all $T$ applications. Its untied counterpart may
select among model tuning parameters $(\lambda,T)$ satisfying $B_U(\lambda,T)\le B$. Thus the
budget constrains the number of distinct trainable coefficients, while $T$
specifies the number of composed update functions.
\end{definition}

\begin{figure}[!htb]
\centering
\includegraphics[width=\linewidth]{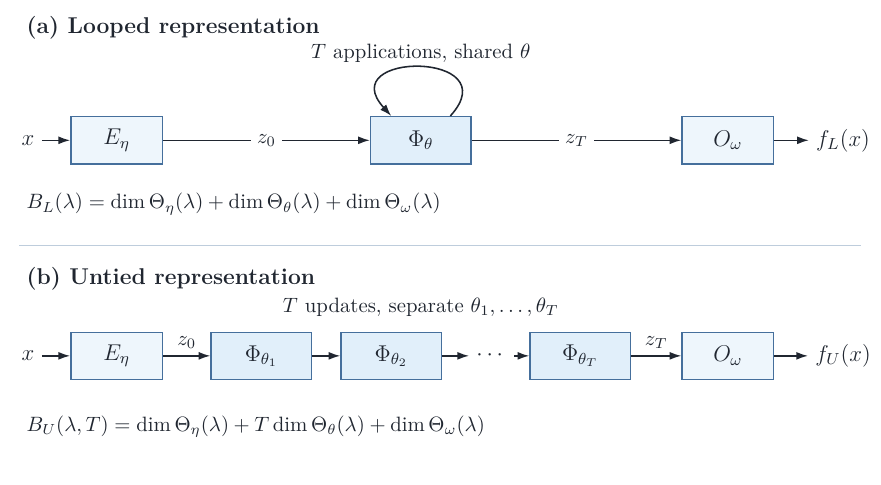}
\caption{Looped and untied representations at a fixed model tuning parameter $\lambda$.
Both use the input function $E_\eta$ and output function $O_\omega$.
Here $z_0=E_\eta(x)$ is the encoded input and $z_T$ is the intermediate representation after $T$
composed updates.
In (a), the same update $\Phi_\theta$ is applied $T$ times; in (b), the
updates $\Phi_{\theta_1},\ldots,\Phi_{\theta_T}$ have separate parameter
vectors. The counts below each panel include the input, update, and output
parameters. Thus $B_L(\lambda)$ is independent of $T$, whereas
$B_U(\lambda,T)$ grows linearly with $T$.}
\label{fig:loop-architecture}
\end{figure}

Let $\Lambda_0$ be the set of candidate values of $\lambda$. Under a
parameter budget $B$, the admissible values of $\lambda$ for looped models
and the admissible pairs $(\lambda,T)$ for untied models are \(\Lambda_L(B)=\{\lambda\in\Lambda_0:B_L(\lambda)\le B\}\), \(\mathcal I_U(B)=\{(\lambda,T)\in\Lambda_0\times\mathbb N: B_U(\lambda,T)\le B\}\). Assume that $\mathcal I_U(B)$ is finite for each $B$. Because
$B_L(\lambda)$ does not depend on $T$, membership in $\Lambda_L(B)$ places no
restriction on the iteration count. Each additional untied iteration instead
adds $\dim\Theta_\theta(\lambda)$ parameters and can carry the pair outside
$\mathcal I_U(B)$. The budgeted untied family includes every admissible
pair of model tuning parameters: \(\mathcal F_U[B]=\bigcup_{(\lambda,T)\in\mathcal I_U(B)} \mathcal F_U(\lambda,T)\).

\subsection{Likelihood risk and sieve maximum likelihood}
\label{sec:statistical-problem}
\label{sec:main-likelihood-risk}\label{sec:approximation-estimation}
Let $\{\mathbb P_v:v\in\mathcal V\}$ be a likelihood model with densities $p_v$
relative to a common measure $\nu$. We observe independent
$Z_1,\ldots,Z_n\sim\mathbb P_{f_0}$ for an unknown $f_0\in\cF$ and use squared
Hellinger distance, \(h^2(\mathbb P_f,\mathbb P_g) =\int(\sqrt{p_f}-\sqrt{p_g})^2\,d\nu\), \(0\le h^2\le2\). We use the same notation for Hellinger distance between laws and between
their densities.

For any deterministic candidate sieve $\mathcal C_n\subseteq\mathcal V$,
possibly depending on $n$, write
$\mathcal P_n=\{p_g:g\in\mathcal C_n\}$ for its induced density class.

\begin{definition}[Sieve maximum likelihood estimator]\label{def:main-sieve-mle}
Fix a deterministic sieve $\mathcal C_n$. For a deterministic optimization
tolerance $\varepsilon_n\ge0$ in the average log likelihood, the sieve maximum likelihood estimator (MLE) is any measurable
$\widehat f_n\in\mathcal C_n$ satisfying \(\frac1n\sum_{i=1}^n\log p_{\widehat f_n}(Z_i) \ge \sup_{g\in\mathcal C_n}\frac1n\sum_{i=1}^n\log p_g(Z_i) -\varepsilon_n\). Existence of a measurable approximate maximizer is assumed.
\end{definition}

We study the worst-case squared Hellinger risk of the sieve MLE over
$f_0\in\cF$. Its risk bound separates approximation by the sieve from
the complexity of fitting over it. For a deterministic sieve $\mathcal C_n$,
define the worst-case squared Hellinger approximation error by \(A(\mathcal C_n) =\sup_{f_0\in\cF}\inf_{g\in\mathcal C_n} h^2(\mathbb P_{f_0},\mathbb P_g)\). A bracket $[p^-,p^+]$ contains the densities lying between its nonnegative
lower and upper functions, with width
$\|\sqrt{p^+}-\sqrt{p^-}\|_{L^2(\nu)}$.
For a density class $\mathcal P$, the Hellinger bracketing number
$N(\epsilon,\mathcal P,h)$ is the smallest number of brackets of width at most
$\epsilon$ needed to contain the class. Its logarithm,
$H(\epsilon,\mathcal P,h)=\log N(\epsilon,\mathcal P,h)$, is the bracketing
entropy. Supplementary Section~\ref{sec:hellinger-complexity} gives the precise pointwise conventions.

Likelihood estimation also requires an approximating candidate that is close
in the one-sided likelihood quantity
\begin{equation}\label{eq:main-rho-gamma}
 \rho_\gamma(p_0,p_1)=\frac1\gamma\int p_0
 \left\{\left(\frac{p_0}{p_1}\right)^\gamma-1\right\}\,d\nu,
 \qquad 0<\gamma\le1.
\end{equation}
For the selected sieve $\mathcal C_n$, set
$\delta(f_0;\mathcal C_n)=\inf_{g\in\mathcal C_n}\rho_{1/2}(p_{f_0},p_g)$.
The next result combines complete-sieve bracketing with a comparison of
likelihood and Hellinger approximation.
\begin{corollary}[Sieve-MLE oracle inequality]\label{cor:main-sieve-oracle}
Let $C_1,C_2,C_3>0$ be the absolute constants in the
sieve-likelihood bound in Supplementary Theorem~\ref{thm:ws-hellinger}. Let $e_n>0$ be a
deterministic critical radius for $\mathcal P_n$ such that $n^{-1}\le e_n^2\le2$ and, for
every $e_n\le t\le\sqrt2$,
\begin{equation}\label{eq:main-ws-integral}
 \int_{t^2/2^8}^{\sqrt2t}
 \sqrt{H(u/C_1,\mathcal P_n,h)}\,du
 \le C_2\sqrt n\,t^2.
\end{equation}
Thus $e_n$ depends on the selected sieve through the entropy of
$\mathcal P_n$. If, for a fixed constant $C_4>0$
independent of $n$ and $f_0$,
\begin{equation}\label{eq:main-rho-bridge}
 \sup_{f_0\in\cF}\delta(f_0;\mathcal C_n)\le C_4A(\mathcal C_n),
\end{equation}
then there is a constant $C_5$, depending only on $C_3,C_4$
and independent of $n$ and $f_0$, such
that
\begin{equation}\label{eq:main-general-oracle}
 \sup_{f_0\in\cF}\E_{f_0}
 h^2(\mathbb P_{f_0},\mathbb P_{\widehat f_n})
 \le C_5\left\{
 \underbrace{A(\mathcal C_n)}_{\text{sieve approximation}}
 +\underbrace{e_n^2}_{\text{fitted-sieve estimation}}
 +\underbrace{\varepsilon_n}_{\text{optimization}}
 \right\}.
\end{equation}
\end{corollary}
The result follows from the one-sided truncated likelihood-ratio argument of
\citet{wongshen1995}; Supplementary Theorem~\ref{thm:ws-hellinger} gives the bound and derivation. Its two
conditions play different roles. The bracketing integral controls fitting over
the complete sieve, while the likelihood bridge certifies that an approximating
member is close to the truth in the one-sided likelihood quantity as well as in
Hellinger distance. An accurate estimator therefore needs a small approximation term, a controlled
estimation term, and a suitably small optimization tolerance.

\subsection{Looped and untied risk comparison}
\label{sec:untied-obstruction}
We compare the loop sieve MLE with estimators taking values in the
untied union $\mathcal F_U[B]$, using the same likelihood model, target
family, squared Hellinger risk, and parameter budget. Every function in
the looped and untied classes is required to belong to $\mathcal V$.

\paragraph{Looped upper bound.}
For looped estimation at sample size $n$, choose a deterministic parameter
budget $B=B_n$ and model tuning parameters
$\lambda=\lambda_n\in\Lambda_L(B_n)$ and $T=T_n\ge1$ before fitting.
These choices may depend on $n$; their $n$ subscripts are
suppressed below. The selected loop sieve is $\mathcal F_L(\lambda,T)$.
Taking $\mathcal C_n=\mathcal F_L(\lambda,T)$ in
Definition~\ref{def:main-sieve-mle} defines the loop sieve MLE
$\widehat f_{L,n}$. Write $A_L(\lambda,T)=A(\mathcal F_L(\lambda,T))$
and let $e_{n,L}$ be a critical radius satisfying the hypotheses of
Corollary~\ref{cor:main-sieve-oracle} for the induced density class.
If the likelihood bridge \eqref{eq:main-rho-bridge} holds for this sieve,
then Corollary~\ref{cor:main-sieve-oracle} gives
\begin{equation}\label{eq:mainrisk}
 \sup_{f_0\in\cF}\E_{f_0}
 h^2(\mathbb P_{f_0},\mathbb P_{\widehat f_{L,n}})
 \le C_5\{A_L(\lambda,T)+e_{n,L}^2+\varepsilon_n\}.
\end{equation}
A parameter budget limits the coefficients available for approximation.
For fixed $\lambda$, increasing $T$ leaves $B_L(\lambda)$ unchanged
but changes the loop sieve and can affect both approximation and estimation.
Thus the bound allows us to ask whether repeated composition improves
approximation under a fixed budget while keeping estimation error controlled.

\paragraph{Untied lower bound.}
For a candidate class
$\mathcal C\subseteq\mathcal V$, define its class-constrained worst-case risk by \(\inf_{\widehat f:\,\widehat f\in\mathcal C} \sup_{f_0\in\cF}\E_{f_0} h^2(\mathbb P_{f_0},\mathbb P_{\widehat f})\). The infimum ranges over all measurable $\mathcal C$-valued estimators.
Taking $\mathcal C=\mathcal F_U[B]$ gives the untied benchmark. This benchmark allows data-dependent selection of $\lambda$ and $T$
and any measurable fitting rule valued in $\mathcal F_U[B]$.

We derive the lower bound by comparing the Hellinger packing number of
the target family with the Hellinger covering number of the budgeted
untied family. A packing consists of target laws separated
by a specified Hellinger distance; a cover approximates every law in the
untied family within a specified distance. If the target packing is too large
for such a cover, at least one target remains separated from every member
of the untied family. We state this comparison precisely below and combine
it with the loop risk upper bound.\looseness=-1

\begin{assumption}[Hellinger packing of the target family]
\label{ass:complexity}
There exist an exponent $a>0$ and constants
$c_{\mathrm{pack}},\delta_0>0$ such that, for every
$0<\delta\le\delta_0$, the target family contains
$f_{1,\delta},\ldots,f_{M_\delta,\delta}$ satisfying \(\min_{1\le j\ne k\le M_\delta} h(\mathbb P_{f_{j,\delta}},\mathbb P_{f_{k,\delta}})\ge8\delta\), \(\log M_\delta\ge c_{\mathrm{pack}}\delta^{-2/a}\). \end{assumption}

For a class of probability laws $\mathcal Q$, let
$\mathcal N(\epsilon,\mathcal Q,h)$ denote the smallest cardinality of an
internal Hellinger cover of radius $\epsilon$, with every center belonging to
$\mathcal Q$. Write
$\mathcal P_U[B]=\{p_g:g\in\mathcal F_U[B]\}$ for the laws induced by the
budgeted untied union.

\begin{theorem}[Untied minimax lower bound under a parameter budget]\label{thm:union}
Suppose there is a constant $C>0$ such that, for sufficiently large $B$ and
$0<\epsilon<1$,
\begin{equation}\label{eq:unioncover}
\log\mathcal N(\epsilon,\mathcal P_U[B],h)
\le C\{B^\kappa\log^b(eB)+B\log(1/\epsilon)\},
\qquad \kappa,b\ge1.
\end{equation}
Under Assumption~\ref{ass:complexity}, there is a constant $c_{\mathrm{low}}>0$,
depending only on the fixed packing and covering constants, such that every
$n\ge1$ and all sufficiently large $B$ satisfy
\begin{equation}\label{eq:unionlower}
\inf_{\widehat f:\,\widehat f\in\mathcal F_U[B]}
\sup_{f_0\in\cF}\E_{f_0}
h^2(\mathbb P_{f_0},\mathbb P_{\widehat f})
\ge c_{\mathrm{low}}\{B^\kappa\log^b(eB)\}^{-a}.
\end{equation}
\end{theorem}
At any fixed sufficiently large $B$, the right side of
\eqref{eq:unionlower} is a positive constant valid for every sample size and
every measurable fitting rule. It also covers data-dependent selection of
$(\lambda,T)$ within $\mathcal F_U[B]$. Theorem~\ref{thm:union} follows from the
packing--cover bound in Lemma~\ref{lem:packing-cover} with squared Hellinger
risk.

\paragraph{Risk comparison.}
\begin{corollary}[Loop-versus-untied risk separation]\label{cor:main-separation}
Let $B=B_n$, $\lambda=\lambda_n\in\Lambda_L(B_n)$, and $T=T_n\ge1$
be deterministic sequences of parameter budgets and model tuning parameters
satisfying the hypotheses of the loop
risk bound \eqref{eq:mainrisk} and Theorem~\ref{thm:union}. If \(\{A_L(\lambda,T)+e_{n,L}^2+\varepsilon_n\} \{B^\kappa\log^b(eB)\}^{a}\longrightarrow0\) as \(n\to\infty\), then the ratio of the worst-case squared Hellinger risk of the loop sieve
MLE to the untied benchmark over $\mathcal F_U[B]$ tends to zero.
\end{corollary}

Corollary~\ref{cor:main-separation} gives a sufficient condition for the
loop sieve MLE to attain a worst-case risk of smaller order than the
untied benchmark under a common parameter budget.
Section~\ref{sec:examples-neural} verifies the approximation and complexity bounds
for concrete likelihood and neural examples; Section~\ref{sec:budget-consequences}
establishes separation under fixed and growing budgets.

\section{Likelihood models and neural realizations}
\label{sec:examples-neural}
The general comparison separates the statistical model from the
parameterized representation. A likelihood model determines how function
approximation and covering translate into probability-law distances. A neural
architecture determines what can be approximated and the complexity of its
complete fitted class. This section verifies both sets of conditions for four
likelihood models and two neural realizations.

\subsection{Likelihood models and metric calibration}
\label{sec:likelihood-examples}\label{sec:conditional-models}
We consider three conditional prediction models and one energy-based
generative model, each indexed by a bounded H\"older-smooth function.
Squared Hellinger distance is comparable to squared $L^2(\mu)$ distance
between the conditional model functions, or distance modulo additive
constants for the energy model.

\begin{definition}[H\"older class]\label{def:holder-class}
Fix $d\ge1$, $s>0$ and $H>0$. Write $k=\lfloor s\rfloor$ when $s$ is not an
integer and $k=s-1$ otherwise, and put $\alpha=s-k\in(0,1]$. Let
$\Holder$ consist of functions $f:[0,1]^d\to[-H,H]$ that are $C^k$ on
$(0,1)^d$, whose partial derivatives through order $k$ are
all bounded by $H$ there, and whose order-$k$ partial derivatives satisfy
$|\partial^\nu f(x)-\partial^\nu f(y)|\le H\|x-y\|_\infty^\alpha$
for $|\nu|=k$ and $x,y\in(0,1)^d$.
Here $\nu=(\nu_1,\ldots,\nu_d)$ is a multi-index of nonnegative integers,
$|\nu|=\sum_i\nu_i$, and $\partial^\nu$ denotes the corresponding mixed
partial derivative. Also, $\|x-y\|_\infty=\max_i|x_i-y_i|$.
Boundary values are bounded by $H$;
no boundary derivative condition is imposed.
\end{definition}

Throughout this section,
$\mathcal X=[0,1]^d$, $\mu$ is uniform probability measure on $\mathcal X$,
$\cF=\Holder$, and $a=2s/d$.
For these examples, let $\mathcal V$ be the measurable functions from
$\mathcal X$ into $[-H,H]$; this likelihood index space extends beyond the
H\"older target class $\cF$.

The four likelihood models below use the same candidate space
$\mathcal V$ and target class $\cF$ in the sieve maximum likelihood framework
of Section~\ref{sec:parameter-budget}. In each model, the unknown function
$f_0$ determines the law of the observations.

\noindent\textbf{Conditional prediction models.}
Observe independent pairs $(X_i,Y_i)$ with
$X_i\sim\mu$ and, for $f_0\in\Holder$, consider
\[
Y_i\mid X_i=x\sim
N(f_0(x),\sigma^2),\quad
\operatorname{Laplace}(f_0(x),\tau),\quad\text{or}\quad
\operatorname{Bernoulli}\bigl((1+e^{-f_0(x)})^{-1}\bigr).
\]
The Gaussian standard deviation $\sigma>0$ and Laplace scale $\tau>0$
are fixed and known. Write $m=\mathrm G,\mathrm{La},\mathrm{Be}$ for these three models and
$\mathbb P_f^m,p_f^m$ for the joint law and density indexed by $f$.

\noindent\textbf{Energy-based generative model.}
Observe independent
$X_1,\ldots,X_n\sim\mathbb P_{f_0}^{\mathrm E}$, where
\begin{equation}\label{eq:ebm-model}
Z_f=\int_{\mathcal X}e^{f(u)}\,d\mu(u),\qquad
p_f^{\mathrm E}(x)=\frac{e^{f(x)}}{Z_f},\qquad f_0\in\Holder.
\end{equation}
Here $f$ is a log-density potential, or negative energy \citep{songkingma2021}. Adding a constant to
$f$ leaves the density unchanged, so define
$d_{\mathrm c}(f,g)^2=\inf_{c\in\mathbb R}\|f-g-c\|_{L^2(\mu)}^2$.
Set $d_m(f,g)=\|f-g\|_{L^2(\mu)}$ for $m\in\{\mathrm G,\mathrm{La},\mathrm{Be}\}$
and $d_{\mathrm E}(f,g)=d_{\mathrm c}(f,g)$.
For a function class $\mathcal G$, write
$\mathcal N(\epsilon,\mathcal G,\|\cdot\|_\infty)$ for its internal uniform covering number.

The next proposition connects these models to the general likelihood theory.
Its metric and likelihood bounds transfer function approximation to likelihood
risk, while its bracket conversion and target packing provide the complexity
conditions used in the upper and lower bounds.\looseness=-1

\begin{proposition}[Metric and likelihood bounds for the examples]
\label{prop:model-geometry}
For $m\in\{\mathrm G,\mathrm{La},\mathrm{Be},\mathrm E\}$, let
$\mathcal P^m(\mathcal G)=\{p_g^m:g\in\mathcal G\}$ for a function class
$\mathcal G$, and let $d_m$ be the pseudometric defined above. There are constants
$0<c_m<C_m<\infty$, depending only on $H$ and on the fixed $\sigma$ or
$\tau$, such that the following statements hold uniformly over measurable
$f,g:\mathcal X\to[-H,H]$.

\emph{(i) Metric equivalence.} \(c_m d_m(f,g)^2\le h^2(\mathbb P_f^m,\mathbb P_g^m) \le C_m d_m(f,g)^2\), \(d_m(f,g)\le\|f-g\|_{L^2(\mu)}\).

\emph{(ii) One-sided likelihood bridge.} \(\rho_{1/2}(p_f^m,p_g^m) \le C_m h^2(\mathbb P_f^m,\mathbb P_g^m)\).

\emph{(iii) Bracket and cover conversion.} For every
$\mathcal G\subseteq\{g:\mathcal X\to[-H,H]\}$ and $0<u<1$, \(H\bigl(u,\mathcal P^m(\mathcal G),h\bigr) \le \log\mathcal N\bigl(u/C_m,\mathcal G,\|\cdot\|_\infty\bigr)\). An internal uniform $u$-cover of $\mathcal G$ also induces an internal
Hellinger $C_m u$-cover of $\mathcal P^m(\mathcal G)$.

\emph{(iv) Target packing.} The H\"older target family satisfies
Assumption~\ref{ass:complexity} in every model with exponent
$a=2s/d$.
\end{proposition}

The proof is in Supplementary Section~\ref{app:proofs-section3}, with the likelihood
calculations in Supplementary Appendix~\ref{sec:losses}. In particular,
$d_{\mathrm c}\le\|\cdot\|_{L^2(\mu)}$ means that the same ordinary
$L^2(\mu)$ approximation bound in Section~\ref{sec:neural-approximation}
suffices for all four models. Section~\ref{sec:architecture-unions} applies
the bracket and cover conversion to the neural classes defined next.

\subsection{Looped FFN and Transformer}
\label{sec:neural-estimators}
We specialize Definition~\ref{def:loop-untied} to residual FFNs and
post-LN Transformers, indexed by $j=\mathrm F$ and $j=\mathrm{Tr}$,
respectively. For each architecture, the model tuning parameter is
$\lambda=(q,r)$. Here $q$ is the intermediate representation dimension for
the FFN and the dimension of each token for the Transformer; $r$ is the
hidden width of the update block's feedforward subnetwork. For both architectures, fix $M_{\rm par}>0$ and restrict
every trainable scalar to $[-M_{\rm par},M_{\rm par}]$.
Write $\Pi_H(u)=\min\{H,\max\{-H,u\}\}$ for the fixed output clipping function,
which adds no trainable parameters.\looseness=-1

\noindent\textbf{Looped FFN.}
The input function $E_\eta$ is an affine function from $\mathbb R^d$ to $\mathbb R^q$.
For a state $z\in\mathbb R^q$, the shared residual block is
\[
\Phi_\theta(z)=z+W_2\ReLU(W_1z+b_1)+b_2,\quad
W_1\in\mathbb R^{r\times q},\ b_1\in\mathbb R^r,\quad
W_2\in\mathbb R^{q\times r},\ b_2\in\mathbb R^q.
\]
Here $\ReLU(u)=\max\{u,0\}$ acts coordinatewise.
The output function $O_\omega$ is an affine scalar function on $\mathbb R^q$
followed by $\Pi_H$. Applying the same block $T$ times gives the class
$\mathcal F_{L,\mathrm F}(q,r,T)$. All input-function, block, and output-function
coefficients range over the prescribed parameter box, with their shapes
fixed before the class is formed. Allowing a separate update vector at
each application defines $\mathcal F_{U,\mathrm F}(q,r,T)$.

\noindent\textbf{Looped Transformer.}
Fix a partition $I_1,\ldots,I_M$ of $\{1,\ldots,d\}$ into nonempty sets
and put $N=M+3$. The state $S\in\mathbb R^{N\times q}$ consists of $M$
data tokens, one controller row, and two reference rows. The input function
assigns data token $j\le M$ the row $S_{0,j}=A_jx_{I_j}+a_j$,
where $A_j\in\mathbb R^{q\times |I_j|}$ and $a_j\in\mathbb R^q$.
The controller and reference rows are input-independent free constant
vectors in $\mathbb R^q$.

Writing $\operatorname{Attn}$ for self-attention,
$\operatorname{FFN}$ for the row-wise feedforward subnetwork, and
$\operatorname{LN}_1,\operatorname{LN}_2$ for layer normalization, the
shared post-LN block is \(Y=\operatorname{LN}_1\{S+\operatorname{Attn}(S)\}\), \(\Phi_\theta(S)=\operatorname{LN}_2\{Y+\operatorname{FFN}(Y)\}\). Attention has one $q$-dimensional softmax head with $1/\sqrt q$ scaling,
no mask, and four affine projections from $\mathbb R^q$ to $\mathbb R^q$.
The FFN has one ReLU hidden layer of width $r$. Each layer normalization
has trainable coordinatewise gain and bias and a fixed positive stabilizer
$\varepsilon_{\rm LN}$. The output function $O_\omega$ is an affine scalar function
of $\operatorname{vec}(S)\in\mathbb R^{Nq}$ followed by $\Pi_H$,
where $\operatorname{vec}$ stacks the matrix entries into a vector.

Applying this block $T$ times with shared parameters defines
$\mathcal F_{L,\mathrm{Tr}}(q,r,T)$; using separate update vectors defines
$\mathcal F_{U,\mathrm{Tr}}(q,r,T)$. Both classes use the same token
protocol, parameter box, and endpoint shapes. Reference values and zero
coordinates used by an approximation witness remain free coordinates in
the complete parameterized classes.
Figure~\ref{fig:neural-realizations} displays the two update blocks within
the common structure of Figure~\ref{fig:loop-architecture}.

\noindent\textbf{Estimation and parameter counts.}
For $j\in\{\mathrm F,\mathrm{Tr}\}$, Table~\ref{tab:parameter-counts}
gives the total numbers of trainable scalar coefficients, including the input
and output functions. The formulas count every matrix entry, bias, layer-normalization
gain, and layer-normalization shift in the definitions above.
Supplementary Section~\ref{app:counts} gives the corresponding dense counts for the residual
FFNs used in the numerical studies.

\begin{figure}[p]
\setstretch{1}
\centering
\includegraphics[width=0.94\linewidth]{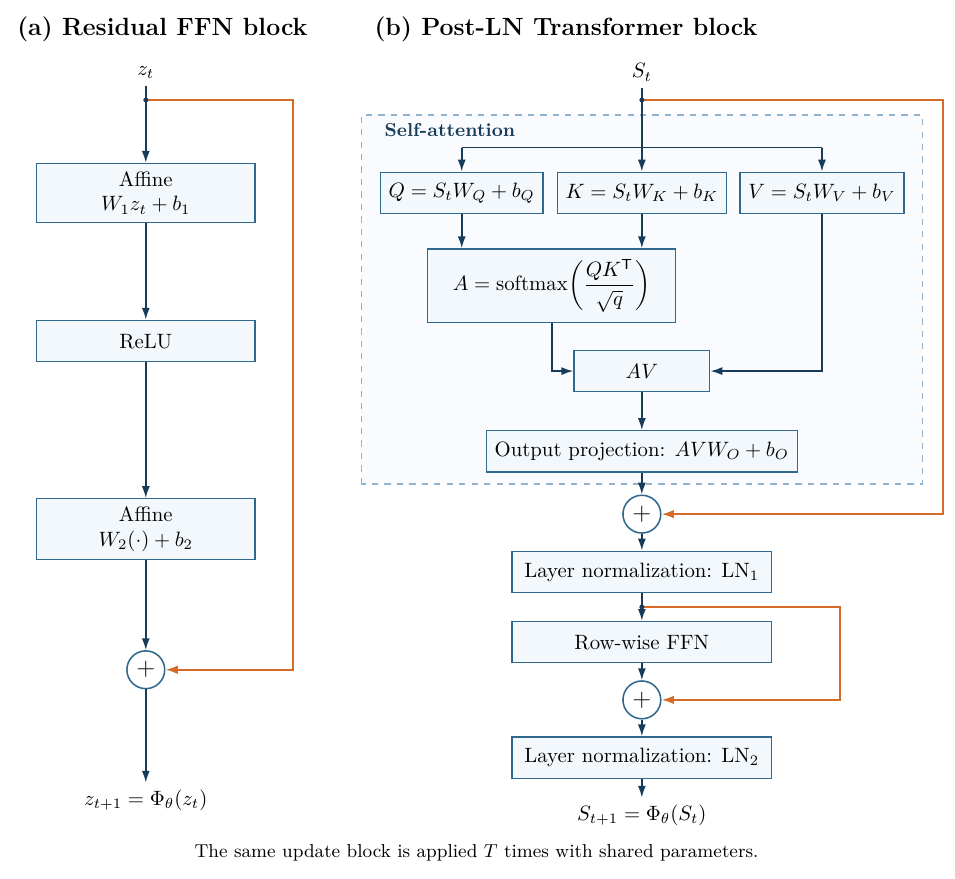}
\caption{Neural realizations of one shared update function. The residual
FFN applies an affine--ReLU--affine function with a residual connection. The
Transformer applies self-attention and a row-wise FFN, with layer normalization
after each residual update. Orange arrows show residual connections. In both
architectures, the displayed update is applied $T$ times with the same
parameters, which are counted once toward the parameter budget.}
\label{fig:neural-realizations}

\vspace{4pt}
\captionof{table}{Distinct trainable parameter counts. The looped count
$B_L^j(q,r)$ is independent of $T$; the untied count
$B_U^j(q,r,T)$ grows linearly with $T$. The Transformer uses $N=M+3$ tokens.}
\label{tab:parameter-counts}
\resizebox{\linewidth}{!}{%
\begin{tabular}{lcc}
\hline
Architecture & $B_L^j(q,r)$ & $B_U^j(q,r,T)$\\
\hline
Residual ReLU FFN
  & $q(d+3)+1+(2q+1)r$
  & $q(d+2)+1+T\{(2q+1)r+q\}$\\
Post-LN Transformer
  & $4q^2+q(d+2N+9)+1+(2q+1)r$
  & $q(d+2N)+1+T\{4q^2+(2q+1)r+9q\}$\\
\hline
\end{tabular}
}
\end{figure}

Under a parameter budget $B$, define
\[
\mathcal I_{U,j}(B)=\{(q,r,T):B_U^j(q,r,T)\le B\},\qquad
\mathcal F_{U,j}[B]=\bigcup_{(q,r,T)\in\mathcal I_{U,j}(B)}
\mathcal F_{U,j}(q,r,T).
\]
Here $j$ identifies the architecture;
the index $m$ identifies the likelihood model from
Section~\ref{sec:likelihood-examples}. The corresponding density classes are
$\mathcal P_{L,j}^m(q,r,T)=\{p_g^m:g\in\mathcal F_{L,j}(q,r,T)\}$
and $\mathcal P_{U,j}^m[B]=\{p_g^m:g\in\mathcal F_{U,j}[B]\}$.
For deterministic model tuning parameters $(q,r,T)$ satisfying
$B_L^j(q,r)\le B$, let $\widehat f_{L,n}^{m,j}$ be the sieve MLE of
Definition~\ref{def:main-sieve-mle} over $\mathcal F_{L,j}(q,r,T)$, with
average-log-likelihood tolerance $\varepsilon_n$. The untied competitor
may be any measurable estimator valued in $\mathcal F_{U,j}[B]$, allowing
data-dependent selection of $(q,r,T)$ and any fitting rule.

\clearpage
\subsection{Approximation error}
\label{sec:neural-approximation}
Approximation error measures how closely a function class can
represent the target before data are used to fit its coefficients. For
each architecture $j$ and likelihood model $m$, define the worst-case
squared Hellinger approximation error \(A_{L,j}^m(q,r,T) =\sup_{f_0\in\Holder}\inf_{g\in\mathcal F_{L,j}(q,r,T)} h^2(\mathbb P_{f_0}^m,\mathbb P_g^m)\). We first bound ordinary $L^2(\mu)$ approximation of the H\"older index and
then use Proposition~\ref{prop:model-geometry} to transfer the result to
all four likelihood models. The following theorem expresses the bound in terms of hidden width $r$ and
iteration count $T$, with $a=2s/d$. Besides the leading approximation term,
each bound includes a term arising from initialization of the approximation
construction. We call the condition
under which this term may be omitted the initialization condition.

\begin{theorem}[Looped neural approximation in Hellinger distance]
\label{thm:neural-approx}
Fix $d,s,H$ and $j\in\{\mathrm F,\mathrm{Tr}\}$. For $j=\mathrm{Tr}$, also fix the token partition and
$\varepsilon_{\rm LN}>0$. There exist constants $C,c>0$ and
architecture-specific thresholds $q_0^j,r_0^j,M_0^j$, depending only on
these fixed quantities, such that, for fixed $M_{\rm par}\ge M_0^j$,
integers $q\ge q_0^j$, $r\ge r_0^j$, and sufficiently large $T$,
every likelihood model $m$ of Section~\ref{sec:likelihood-examples} satisfies
\begin{equation}\label{eq:conditional-approx}
A_{L,j}^m(q,r,T)
\le C_m\left\{\left[\frac{\log(erT)}{rT}\right]^a
+e^{-cT^{1/h_0^j}}\right\},
\end{equation}
where $a=2s/d$, $h_0^{\mathrm F}=1$, $h_0^{\mathrm{Tr}}=2$, and
$C_m>0$ depends additionally on the fixed parameters of the likelihood model.
The exponential term may be omitted when $T\ge C\log(erT)$ for the
residual FFN and when $T\ge C\log^2(erT)$ for the Transformer.

For fixed admissible $q,r$, $B_L^j(q,r)$ is independent of $T$. Once the
initialization condition above holds, the bound simplifies to
$A_{L,j}^m(q,r,T)\lesssim_{q,r}\{\log(eT)/T\}^{a}$.
Thus any budget $B\ge B_L^j(q,r)$ admits all iteration counts,
with approximation error tending to zero as $T\to\infty$.
\end{theorem}

\noindent\textbf{How looping improves approximation.}
Theorem~\ref{thm:neural-approx} follows from the explicit FFN and Transformer
constructions in Theorems~\ref{thm:ffn-construction} and
\ref{thm:tr-construction}, respectively. Their central idea is to approximate
the target by local Taylor polynomials on a spatial grid and encode the
coefficients in bounded real parameters. Repeated use of the same update function
decodes and evaluates the encoded information; Supplementary Section~\ref{app:base-engine}
details this procedure, and Supplementary Appendix~\ref{app:transformer-proof} implements
it in the Transformer.
Above the initialization threshold, the proof of
Supplementary Lemma~\ref{lem:neural-index-approx} shows
that $T$ loop iterations with hidden width $r$ permit a grid with order
$rT/\log(erT)$ cells. The resulting finer local approximation gives the
squared-error bound $\{\log(erT)/(rT)\}^a$, with the number of trainable
coefficients independent of $T$.\looseness=-1

\subsection{Complexity of the looped and untied classes}\label{sec:architecture-unions}
Approximation requires a suitable candidate for each target; estimation
requires control of the complete class over which fitting takes place.
Supplementary Lemma~\ref{lem:neural-index-cover} bounds uniform covering
numbers of the neural index classes; Proposition~\ref{prop:model-geometry}
transfers these bounds to density-space complexity. For the loop sieve, Hellinger
bracketing entropy controls estimation error. For the budgeted untied
union defined in Section~\ref{sec:neural-estimators}, a Hellinger covering
bound combines with target packing to yield a lower bound that remains
valid after data-dependent architecture selection.
For any class $\mathcal C$ equipped with a pseudometric $d$,
$\mathcal N(\epsilon,\mathcal C,d)$ denotes its internal covering number at
radius $\epsilon$.

Fix a numerical constant $C_6\ge2$ and define
\begin{equation}\label{eq:commonG}
G(q,r,T)=\log(T+1)+(T+1)\log(C_6(q+1)(r+1)).
\end{equation}
Up to a constant, $G(q,r,T)$ bounds the logarithm of the output sensitivity
to coefficient changes after $T$ loop iterations. Greater sensitivity requires a finer coefficient grid, so the covering bound
depends on the iteration count even with a fixed number of coefficients.

\begin{proposition}[Hellinger complexity bounds]
\label{prop:architecture-cover}
Fix $M_{\rm par}>0$, $j\in\{\mathrm F,\mathrm{Tr}\}$, and a likelihood
model $m$ of Section~\ref{sec:likelihood-examples}. There exists a constant
$C>1$, depending only on the fixed problem and implementation parameters,
such that, for positive integers $q,r,T$ and $0<u<1$,
\begin{equation}\label{eq:conditional-entropy}
H(u,\mathcal P_{L,j}^m(q,r,T),h)
\le C B_L^j(q,r)\{\log(C/u)+G(q,r,T)\}.
\end{equation}
For all sufficiently large $B$ and $0<\epsilon<1$,
\begin{equation}\label{eq:law-union-cover}
\log\mathcal N(\epsilon,\mathcal P_{U,j}^m[B],h)
\le C\{B^2\log(eB)+B\log(C/\epsilon)\}.
\end{equation}

For fixed admissible $q,r$, $B_L^j(q,r)=O(1)$ and
$G(q,r,T)=O(T)$. Thus the looped-class entropy bound in
\eqref{eq:conditional-entropy} has the simpler form
$H(u,\mathcal P_{L,j}^m(q,r,T),h)\lesssim_{q,r} T+\log(C/u)$ for $0<u<1$.\looseness=-1
\end{proposition}

The proof is given in Supplementary Section~\ref{app:proofs-section3}.

Repeated computation therefore affects both sides of the loop's statistical
balance. At fixed admissible widths, its leading approximation bound decreases
as $\{\log(eT)/T\}^{a}$, whereas the entropy bound controlling estimation grows
linearly in $T$, up to the resolution term $\log(C/u)$. The next section
selects model tuning parameters to balance these effects.

\section{Risk under fixed and growing parameter budgets}
\label{sec:budget-consequences}
We apply the approximation and complexity bounds of Section~\ref{sec:examples-neural}
to compare looped and untied estimation under a common parameter budget.
With a sufficiently large fixed budget, a suitable choice of iteration count
allows the loop sieve MLE to attain the H\"older minimax rate up to logarithmic
factors, while the optimal worst-case risk over the untied family remains
bounded away from zero. We also identify growing-budget regimes in which
the loop-to-untied risk ratio tends to zero as the neural dimensions and
iteration count increase together. The analysis begins with a loop risk
upper bound and lower bounds for unrestricted estimation and the budgeted
untied family, followed by the fixed- and growing-budget comparisons.

\subsection{Risk bounds: estimation and representation}
\label{sec:model-risk-bounds}
For a loop sieve with model tuning parameters $(q,r,T)$, let
$\widehat f_{L,n}^{m,j}$ denote its sieve MLE. Define its worst-case squared
Hellinger risk and the best achievable worst-case squared Hellinger risk
over the budgeted untied family, respectively, by
\begin{equation}\label{eq:budget-risk-notation}
\begin{aligned}
R_{L,j,n}^m
&=\sup_{f_0\in\Holder}\mathbb E_{f_0}
h^2(\mathbb P_{f_0}^m,\mathbb P_{\widehat f_{L,n}^{m,j}}^m),\\
R_{U,j,n}^m(B)
&=\inf_{\widehat f:\,\widehat f\in\mathcal F_{U,j}[B]}
\sup_{f_0\in\Holder}\mathbb E_{f_0}
h^2(\mathbb P_{f_0}^m,\mathbb P_{\widehat f}^m).
\end{aligned}
\end{equation}
All measurable fitting and model-tuning rules valued in the budgeted union are allowed.

\begin{theorem}[Sieve-MLE Hellinger risk under a parameter budget]
\label{thm:model-risk}
Fix a likelihood model $m\in\{\mathrm G,\mathrm{La},\mathrm{Be},\mathrm E\}$
and an architecture $j\in\{\mathrm F,\mathrm{Tr}\}$. Let $n\ge2$ and $a=2s/d$.

\textup{\textbf{General width-dependent bound.}}
For any admissible $q,r,M_{\rm par},T$ from Theorem~\ref{thm:neural-approx},
let $\widehat f_{L,n}^{m,j}$ be a measurable sieve MLE over
$\mathcal F_{L,j}(q,r,T)$ with average-log-likelihood tolerance $\varepsilon_n$.
There exist constants $C,c>0$, depending only on the fixed model and
implementation parameters, such that, with
$e_{n,j}^2=1\wedge[C B_L^j(q,r)\{\log(en)+G(q,r,T)\}/n]$,
\begin{equation}\label{eq:conditional-risk}
R_{L,j,n}^m
\le C\left\{
\underbrace{
\left[\frac{\log(erT)}{rT}\right]^a+e^{-cT^{1/h_0^j}}
}_{\text{approximation}}
+\underbrace{e_{n,j}^2+\varepsilon_n}_{\text{fitting and optimization}}
\right\}.
\end{equation}

\textup{\textbf{Budget--iteration bound.}}
In particular, fix admissible $q,M_{\rm par}$ from Theorem~\ref{thm:neural-approx} and,
for sufficiently large $B$, choose the largest affordable hidden width
$r_B=\max\{r\in\mathbb N:B_L^j(q,r)\le B\}$.
Let $\widehat f_{L,n}^{m,j}$ be a measurable sieve MLE over
$\mathcal F_{L,j}(q,r_B,T)$ with average-log-likelihood tolerance
$\varepsilon_n\le n^{-1}$. If $T$ is sufficiently large for
Theorem~\ref{thm:neural-approx} to apply at $(q,r_B,M_{\rm par})$ and
satisfies its initialization condition, then
\begin{equation}\label{eq:budget-iteration-risk}
R_{L,j,n}^m\le C'\left\{
\underbrace{\left[\frac{\log(eBT)}{BT}\right]^a}_{\text{approximation}}
+\underbrace{\frac{B\{\log(en)+T\log(eB)\}}{n}}_{\text{fitting and optimization}}
\right\}.
\end{equation}
Here $C'>0$ may depend on the fixed $q$ and remaining problem and architecture
parameters, but is independent of $n,B,T$, and $f_0$.
\end{theorem}

For fixed $q$, Table~\ref{tab:parameter-counts} gives $r_B\asymp B$,
which converts the general bound into \eqref{eq:budget-iteration-risk}.
The latter holds for both fixed and sample-size-dependent budgets, provided
its conditions are satisfied. Suppressing logarithmic factors, it balances
approximation of order $(BT)^{-a}$ against estimation of order $BT/n$;
the balance is attained when $BT$ has order $n^{1/(a+1)}$.
At a fixed budget, this balance is achieved by increasing $T$ without
adding parameters. A larger shared block can achieve a prescribed
approximation accuracy with fewer iterations. At fixed $T$, increasing $B$
also increases the estimation bound. Repeated composition therefore balances approximation and estimation
within a fixed parameter dimension,
whereas additional untied depth requires additional coefficients.
Theorems~\ref{thm:fixed-budget} and~\ref{thm:growing-budget} specialize
Theorem~\ref{thm:model-risk} to fixed and growing budgets, respectively,
and combine its upper bound with the untied lower bound to establish
risk separation. The growing-width analysis uses the general bound
\eqref{eq:conditional-risk}.

The proof combines the approximation in Theorem~\ref{thm:neural-approx}, the
full-class entropy in Proposition~\ref{prop:architecture-cover}, and the
likelihood bridge in Proposition~\ref{prop:model-geometry} through
Corollary~\ref{cor:main-sieve-oracle}.
To judge whether the resulting rate can be improved, we first give a minimax
lower bound over all measurable distribution estimators. This benchmark
quantifies the sampling limitation common to the four likelihood models.\looseness=-1

\begin{theorem}[Unrestricted minimax benchmark for the likelihood examples]
\label{thm:model-minimax}
For every $m\in\{\mathrm G,\mathrm{La},\mathrm{Be},\mathrm E\}$, there is a
constant $c_{\mathrm{min}}>0$, depending only on the fixed likelihood-model and H\"older-class
parameters, such that, for all sufficiently large $n$,
\begin{equation}
\label{eq:model-minimax}
\inf_{\widehat P}\sup_{f_0\in\Holder}
\mathbb E_{f_0}h^2(\mathbb P_{f_0}^m,\widehat P)
\ge c_{\mathrm{min}} n^{-a/(a+1)},
\end{equation}
where the infimum is over all measurable distribution estimators.
\end{theorem}

Theorem~\ref{thm:model-minimax} applies independently of the parameterized
representation. Restricting the estimator to $\mathcal F_{U,j}[B]$ imposes an
additional limitation: the next proposition gives a lower bound determined
by the parameter budget that holds uniformly in sample size.

\begin{proposition}[Untied lower bound for the likelihood examples]
\label{prop:budget-lower}
Fix $m\in\{\mathrm G,\mathrm{La},\mathrm{Be},\mathrm E\}$ and
$j\in\{\mathrm F,\mathrm{Tr}\}$. With $R_{U,j,n}^m(B)$ defined in
\eqref{eq:budget-risk-notation}, there is a constant $c_{\mathrm U}>0$,
depending only on the fixed model, architecture, and H\"older-class parameters,
such that, uniformly over every
$n\ge1$ and all sufficiently large $B$,
\begin{equation}\label{eq:concretelower}
 R_{U,j,n}^m(B)
 \ge c_{\mathrm U}\{B^2\log(eB)\}^{-a}.
\end{equation}
\end{proposition}

Taken together, Theorem~\ref{thm:model-minimax} and
Proposition~\ref{prop:budget-lower} give, for sufficiently large $n$ and $B$, \(R_{U,j,n}^m(B)\gtrsim n^{-a/(a+1)}\vee\{B^2\log(eB)\}^{-a}\), representing sampling and representation limitations, respectively. The first term is a property of the statistical experiment. The second comes
from restricting the estimator's output to the budgeted untied family and
persists even with arbitrarily many observations. The proof of Proposition~\ref{prop:budget-lower} combines the budgeted cover
\eqref{eq:law-union-cover}, the target packing in
Proposition~\ref{prop:model-geometry}, and Theorem~\ref{thm:union}.
At a fixed budget, the representation term is positive for every sample
size. The next result constructs a loop estimator whose worst-case risk
tends to zero under the same budget.

\subsection{Main result: separation at a fixed parameter budget}\label{sec:fixed-budget}
We now fix the neural dimensions and choose $T_n$ to balance the approximation
and estimation terms in Theorem~\ref{thm:model-risk}. Combining the resulting
loop upper bound with Proposition~\ref{prop:budget-lower} gives a risk
separation under a common fixed parameter budget. Theorem~\ref{thm:model-minimax}
provides the benchmark for assessing the rate of the looped estimator.
\begin{theorem}[Fixed-budget risk separation]
\label{thm:fixed-budget}
Fix a likelihood model $m\in\{\mathrm G,\mathrm{La},\mathrm{Be},\mathrm E\}$,
one architecture $j\in\{\mathrm F,\mathrm{Tr}\}$, and admissible constants
$q\ge q_0^j$, $r\ge r_0^j$, $M_{\rm par}\ge M_0^j$. Let $\widehat f_{L,n}^{m,j}$
be the sieve MLE over the corresponding loop class and take
$\varepsilon_n\le n^{-1}$. Choose
\begin{equation}\label{eq:fixed-T}
T_n=\left\lceil n^{1/(a+1)}
             \{\log(en)\}^{a/(a+1)}\right\rceil .
\end{equation}
Then the parameter count $B_L^j(q,r)$ is constant in $n$, and
there is a constant $C_7>0$, possibly depending on the fixed admissible
$q,r,M_{\rm par}$ but independent of $n$ and the target, such that
\begin{equation}\label{eq:fixed-rate}
R_{L,j,n}^m
\le C_7 n^{-a/(a+1)}\{\log(en)\}^{a/(a+1)}.
\end{equation}

Let $B\ge B_L^j(q,r)$ be any sufficiently large fixed parameter budget.
Every estimator valued in the corresponding untied family obeys the positive lower bound
$R_{U,j,n}^m(B)\ge c_{\mathrm U}\{B^2\log(eB)\}^{-a}$ supplied by
Proposition~\ref{prop:budget-lower}, uniformly in $n$. Combining this with
\eqref{eq:fixed-rate} gives
\begin{equation}\label{eq:fixed-separation-ratios}
 \frac{R_{L,j,n}^m}{R_{U,j,n}^m(B)}\longrightarrow0.
\end{equation}
\end{theorem}

The iteration schedule $T_n$ in \eqref{eq:fixed-T} balances the estimation
term $T_n/n$ and the approximation term $(\log T_n/T_n)^a$ at fixed $(q,r)$.
Increasing $T_n$ changes the loop sieve while leaving its parameter count
fixed, yielding the minimax polynomial rate of
Theorem~\ref{thm:model-minimax} up to a logarithmic factor.
Under the same fixed budget, the untied risk remains bounded away from
zero, even with data-dependent tuning of width and depth.
Thus, vanishing worst-case risk requires a growing untied budget,
whereas the loop needs only more iterations.

\subsection{Extension: expanding sieves under a growing budget}\label{sec:growing-budget}
When $B=B_n$ grows, the untied lower bound in
Proposition~\ref{prop:budget-lower} also decreases, so separation requires a
comparison of the two rates. The fixed-dimensional loop of
Theorem~\ref{thm:fixed-budget} remains admissible and already gives separation
in the budget range considered below. We now allow both widths to grow
and choose their schedules, together with $T_n$, to balance the general
risk bound \eqref{eq:conditional-risk} in Theorem~\ref{thm:model-risk}.
The resulting estimator retains the minimax polynomial rate while fitting
over an expanding loop class, and its risk remains of smaller order than
the optimal risk over the growing untied family.
\begin{theorem}[Growing budget: jointly tuned rate and risk separation]
\label{thm:growing-budget}
Fix a likelihood model $m\in\{\mathrm G,\mathrm{La},\mathrm{Be},\mathrm E\}$
and an architecture $j\in\{\mathrm F,\mathrm{Tr}\}$ of
Section~\ref{sec:neural-estimators}, with a fixed token partition when
$j=\mathrm{Tr}$. Fix $M_{\rm par}\ge M_0^j$. Set $B=B_n\asymp n^{\gamma_B}$ with
\begin{equation}\label{eq:window}
0<\gamma_B<\frac{1}{2(a+1)}=\frac{d}{4s+2d}.
\end{equation}
Take $q_n\asymp\log(en)$ and
$r_n=\max\{r\in\mathbb N:B_L^j(q_n,r)\le B_n\}$. Define \(\ell_n=\log(C_6(q_n+1)(r_n+1))\), \(W_n=\{\frac{n\{\log(en)\}^{a}}{q_n\ell_n}\}^{1/(a+1)}\), \(T_n=\lceil W_n/r_n\rceil\). The resulting loop sieve has at most $B_n$ parameters and
$q_n,r_n,T_n\to\infty$. Its sieve MLE with $\varepsilon_n\le n^{-1}$ satisfies,
for all sufficiently large $n$,
\begin{equation}\label{eq:rate}
R_{L,j,n}^m
\le C_8 n^{-a/(a+1)}\{\log(en)\}^{3a/(a+1)},
\end{equation}
and
\begin{equation}\label{eq:growing-separation-ratios}
 \frac{R_{L,j,n}^m}{R_{U,j,n}^m(B_n)}\longrightarrow0.
\end{equation}
The constant $C_8>0$ depends only on the fixed model, architecture,
H\"older-class parameters, and constants in the schedules.
\end{theorem}

Under these schedules, $B_nT_n$ has order $n^{1/(a+1)}$ up to logarithmic factors.
The approximation and estimation terms in
\eqref{eq:conditional-risk} are both bounded by a constant multiple of the
right side of \eqref{eq:rate}. The polynomial exponent matches
\eqref{eq:fixed-rate}; the larger logarithmic factor reflects the growth of
$q_n$ and $\ell_n$ in the bound. In the budget range \eqref{eq:window}, the
loop upper bound is of smaller order than the untied lower bound in
Proposition~\ref{prop:budget-lower}, yielding
\eqref{eq:growing-separation-ratios}. Here the fitted loop sieve's parameter count grows with $n$.

\section{Discussion}\label{sec:discussion}

The main statistical finding is that repeated composition of a shared
function can improve approximation enough to yield vanishing risk without
increasing parameter count. Theorem~\ref{thm:fixed-budget} establishes this conclusion for
H\"older targets in four likelihood models. The loop approximation guarantee
improves as $T_n$ grows, while its complexity bound also increases.
Balancing these effects yields the minimax rate up to a logarithmic factor. Under the same fixed cap, the untied family retains positive worst-case
representation error.

\paragraph{Fixed and expanding sieves.}
The fixed-dimensional loop already gives the central separation.
Theorem~\ref{thm:growing-budget} shows that the conclusion also holds when state
width, hidden width, and iteration count grow together and likelihood fitting
ranges over the enlarged coefficient class. At polynomial order, hidden width
satisfies $r\asymp B/q$; the approximation bound has order $(BT)^{-2s/d}$ up to
logarithmic factors when $q$ is fixed or logarithmic. The fitted-class entropy
determines the statistical cost of this expansion. Sharper entropy and
approximation bounds may reduce the logarithmic factors.

\paragraph{Numerical findings.}
Supplementary Appendix~\ref{sec:numerical} compares single-block, looped,
and untied residual FFNs under common parameter budgets in regression
simulations, energy-based density estimation, and real-data prediction.
Looping improves performance in several settings, with gains varying across
targets, budgets, and tasks. The sample-size experiments keep the iteration count fixed;
the iteration experiments combine additional iterations with continued training.\looseness=-1

\paragraph{Parameters, computation, and precision.}
The budget counts distinct trainable scalar coefficients. Reusing the same
update function adds computation without adding coefficients. In the separation regime,
$T_n/B_n\to\infty$, whereas untied depth is at most $O(B_n)$; a dense residual
FFN loop uses order $B_nT_n$ parameter-associated operations per forward pass,
compared with order $B_n$ for an untied stack. Coefficient magnitudes remain
bounded, while the approximation constructions encode increasingly precise
Taylor-table codes and the FFN construction uses intermediate states whose magnitudes grow.
These requirements motivate jointly constraining parameter count, numerical
precision, optimization, and computation to characterize their tradeoff with
statistical risk.

\paragraph{Optimization.}
The statistical guarantees apply to an approximate maximizer of the empirical
likelihood, with tolerance $\varepsilon_n$. Reaching this tolerance is a
separate computational problem. Sharing couples coefficient changes across
all repeated updates, whereas untied updates can adjust their coefficients
separately. Comparing the computational effort required to reach this tolerance remains open.

\paragraph{Scope and future directions.}
The concrete comparisons cover the dense residual ReLU and post-LN Transformer
families specified in Section~\ref{sec:neural-estimators}, with bounded
coefficients, affine endpoints, and common block dimensions within each stack.
The Transformer has one softmax head, the specified position-dependent input
functions and special tokens, and a fixed positive normalization stabilizer.
Different normalization or masking conventions, variable-width stacks, growing
token counts, and compressed parameter descriptions require their own
approximation and complexity analysis. For another likelihood or evaluation
risk, the general framework requires the corresponding metric calibration and
estimation bound. The energy-based example accommodates normalization and an index defined modulo constants.

The schedules use known smoothness $s$; adaptation to unknown smoothness
remains open. The sharp untied budget threshold is another open question. For polynomial budgets, the proved lower bound gives
the necessary exponent condition $\gamma_B\ge d/(4s+2d)$ for attaining the
classical rate. Copying the shared block into separate untied blocks gives a
larger sufficient budget. Sharper untied approximation and lower bounds are
needed to close this gap.

\section*{Supplementary material}
Supplementary Appendices A--F, included below, contain general risk tools, proofs of the
main results, auxiliary likelihood and complexity bounds, the residual
FFN and Transformer approximation constructions, and numerical studies with experimental details.
\if0\anon\else
\section*{Funding}
This work was supported in part by NSF Grant DMS-1952539 and NIH Grants R01AG069895, R01AG065636, R01AG074858, and U01AG073079.
\fi
\bibliographystyle{agsm}
\bibliography{references}

\clearpage
\section*{Supplementary Appendices}
\phantomsection
\addcontentsline{toc}{section}{Supplementary Appendices}
\noindent This supplement develops the ingredients used to prove the main
results. Appendix~\ref{app:risk-tools} gives the general sieve-likelihood
upper bound and packing--cover lower bound. Appendix~\ref{sec:losses}
establishes the likelihood comparisons, target packing, and parameter
sensitivity bounds needed to apply these tools. Appendices~\ref{app:ffn}
and~\ref{app:transformer-proof} construct the residual FFN and Transformer
approximations at a fixed grid resolution. Appendix~\ref{app:general}
then chooses the resolution from the iteration budget, transfers index
approximation and covering bounds to Hellinger distance, and derives the
main risk bounds and their fixed- and growing-budget consequences.
Appendix~\ref{app:numerical-details} presents the numerical results;
Section~\ref{app:experimental-protocols} collects their experimental details.

\appendix
\setcounter{equation}{0}
\setcounter{table}{0}
\setcounter{figure}{0}
\renewcommand{\theequation}{S\arabic{equation}}
\renewcommand{\thetable}{S\arabic{table}}
\renewcommand{\thefigure}{S\arabic{figure}}
\renewcommand{\theHequation}{S\arabic{equation}}
\renewcommand{\theHtable}{S\arabic{table}}
\renewcommand{\theHfigure}{S\arabic{figure}}
\section{General risk and likelihood tools}\label{app:risk-tools}
\label{sec:risk-theory}
We use the statistical experiment and notation of
Section~\ref{sec:parameter-budget}. The following tools provide a
class-constrained lower bound for calibrated risks and a likelihood upper
bound for a fitted sieve.

\subsection{A packing--cover lower-bound lemma}
\label{sec:metric-comparison}
For a pseudometric $d$ on the target and candidate indices, let
$\mathcal N(\delta,\mathcal C,d)$ be the internal covering number, as in
Section~\ref{sec:untied-obstruction}, and let
$\mathcal M(\delta,\cF,d)$ be the supremum of the cardinalities of finite
subsets of $\cF$ with pairwise distances at least $\delta$.

\begin{lemma}[Packing--cover lower bound]\label{lem:packing-cover}
Let $\mathcal C\subseteq\mathcal V$ be a nonempty candidate class and
$R(f,g)\ge0$ an evaluation risk whose value at each estimator under
consideration is measurable. Suppose that, for constants $c_d,p>0$,
\[
 R(f,g)\ge c_d d(f,g)^p,\qquad f\in\cF,\quad g\in\mathcal C.
\]
If, for some $\delta>0$,
\[
 \mathcal M(8\delta,\cF,d)>\mathcal N(\delta,\mathcal C,d),
 \qquad \mathcal N(\delta,\mathcal C,d)<\infty,
\]
then there is an $f_*\in\cF$ such that $d(f_*,g)\ge\delta$ for every
$g\in\mathcal C$. Consequently, for every sample size $n\ge1$,
\begin{equation}\label{eq:packing-cover-lower}
 \inf_{\widehat g:\,\widehat g\in\mathcal C}
 \sup_{f\in\cF}\E_f R(f,\widehat g)\ge c_d\delta^p,
\end{equation}
where the infimum ranges over all measurable $\mathcal C$-valued estimators.
\end{lemma}
\begin{proof}
Choose an $8\delta$-packing $f_1,\ldots,f_M$ and an internal
$\delta$-cover $g_1,\ldots,g_N$ of $\mathcal C$ with $M>N$.
If each $f_j$ were within distance less than $\delta$ of some member of
$\mathcal C$, it would be within distance less than $2\delta$ of a cover
center. Two targets assigned to the same center would have distance less
than $4\delta$, contrary to their separation. The assignment would therefore
be injective, contradicting $M>N$. Hence some $f_*$ is at distance at least
$\delta$ from every candidate. For this target,
$R(f_*,\widehat g)\ge c_d\delta^p$ for every sample. Taking expectation,
the supremum over targets, and the infimum over estimators proves the bound.
\end{proof}
The lemma permits arbitrary data-dependent fitting and selection within
$\mathcal C$. Theorem~\ref{thm:union} uses squared Hellinger risk and $d=h$.

\subsection{A sieve-likelihood risk bound}
\label{sec:hellinger-complexity}\label{sec:likelihood-risk}
Use the density model and Hellinger distance of
Section~\ref{sec:parameter-budget}, and the one-sided likelihood quantity
$\rho_\gamma$ in
\eqref{eq:main-rho-gamma}. Fix measurable density versions on one common set
$\mathcal Z_\nu$ whose complement has $\nu$-measure zero. All pointwise
comparisons below use this set. Density ratios use $0/0=0$ and
$a/0=+\infty$ for $a>0$, and $\log0=-\infty$.

A bracket $[p^-,p^+]$ contains all densities $p$ satisfying
$p^-\le p\le p^+$ on $\mathcal Z_\nu$, where the endpoints are nonnegative
measurable functions with $\int p^+\,d\nu<\infty$. Its width is
$\|\sqrt{p^+}-\sqrt{p^-}\|_{L^2(\nu)}$; the endpoints need not be densities.
As in Section~\ref{sec:main-likelihood-risk}, $N(u,\mathcal P,h)$ counts
brackets of width at most $u$ and $H(u,\mathcal P,h)=\log N(u,\mathcal P,h)$.
In particular, a pointwise envelope
$|\sqrt p-\sqrt{p_j}|\le uL$, with nonnegative $L\in L^2(\nu)$, gives the
bracket
\[
 \left[\{(\sqrt{p_j}-uL)_+\}^2,(\sqrt{p_j}+uL)^2\right]
\]
of width at most $2u\|L\|_2$, where $t_+=\max\{t,0\}$.
Choosing a member of each nonempty bracket gives
$\mathcal N(u,\mathcal P,h)\le N(u,\mathcal P,h)$; an ordinary metric cover
alone does not supply the pointwise envelopes. Appendix~\ref{sec:losses}
constructs the required brackets for the likelihood examples.

For a deterministic sieve $\mathcal S_n\subseteq\mathcal V$, write
$\mathcal P_n=\{p_g:g\in\mathcal S_n\}$ and let $\widehat f_n$ be a measurable
approximate maximizer with average-log-likelihood tolerance $\varepsilon_n$,
as in Definition~\ref{def:main-sieve-mle} with $\mathcal C_n$ replaced by
$\mathcal S_n$. The next result combines Theorem~1 of
\citet{wongshen1995} with a comparison to one approximating sieve density.

\begin{theorem}[Sieve-MLE Hellinger bound]\label{thm:ws-hellinger}
For the sieve MLE and density class specified above, set
\[
\delta_n(f_0)=\inf_{q\in\mathcal P_n}\rho_{1/2}(p_{f_0},q).
\]
Let $C_1,C_2>0$ be the absolute constants in
Theorem~1 of \citet{wongshen1995}; they are independent of the sieve,
sample size, and target. Suppose that $e_n>0$ is deterministic,
$n^{-1}\le e_n^2\le2$, and, for every $e_n\le t\le\sqrt2$,
\begin{equation}\label{eq:ws-integral}
\int_{t^2/2^8}^{\sqrt2t}
\sqrt{H(u/C_1,\mathcal P_n,h)}\,du
\le C_2\sqrt n\,t^2.
\end{equation}
The resolution $e_n$ may depend on the selected sieve $\mathcal P_n$.
Then an absolute constant $C_3$ satisfies
\begin{equation}\label{eq:ws-mean}
\E_{f_0}h^2(\mathbb P_{f_0},\mathbb P_{\widehat f_n})
\le C_3\{\delta_n(f_0)+e_n^2+\varepsilon_n\}.
\end{equation}
The bound holds uniformly over any target set on which the right-hand side is
uniformly bounded.
\end{theorem}
\begin{proof}
Fix $f_0$ and put $p_0=p_{f_0}$ and
$L_n(q)=\prod_{i=1}^n q(Z_i)/p_0(Z_i)$. Theorem~1 of
\citet{wongshen1995}, under \eqref{eq:ws-integral}, gives absolute
constants $c_1,c_2>0$ such that
\begin{equation}\label{eq:ws-surface}
 \mathbb P_{f_0}^{*}\!\left\{
 \sup_{q\in\mathcal P_n:\,h(p_0,q)\ge t}L_n(q)
 \ge e^{-c_1nt^2}\right\}\le4e^{-c_2nt^2},
 \qquad e_n\le t\le\sqrt2.
\end{equation}
Outer probability permits a supremum over a general density class;
the selected estimator is measurable by hypothesis.

If $\delta_n(f_0)=\infty$, the conclusion is immediate. Otherwise,
fix $\eta>0$ and choose a deterministic $q_*\in\mathcal P_n$ such that
$\rho_{1/2}(p_0,q_*)\le\delta_n(f_0)+\eta$. Independence and the
definition of $\rho_{1/2}$ imply
\[
 \E_{f_0}L_n(q_*)^{-1/2}
 =\left\{1+\tfrac12\rho_{1/2}(p_0,q_*)\right\}^{n}
 \le\exp\{n(\delta_n(f_0)+\eta)/2\}.
\]
Markov's inequality therefore gives
\begin{equation}\label{eq:ws-comparator}
 \mathbb P_{f_0}\{L_n(q_*)\le
     e^{-c_1nt^2+n\varepsilon_n}\}
 \le\exp\!\left\{-\frac n2
 [c_1t^2-\varepsilon_n-\delta_n(f_0)-\eta]\right\}.
\end{equation}
The approximate maximum-likelihood property gives
$L_n(p_{\widehat f_n})\ge e^{-n\varepsilon_n}L_n(q_*)$.
If $h(p_0,p_{\widehat f_n})>t$, either the event in
\eqref{eq:ws-surface} or the event in \eqref{eq:ws-comparator}
must occur. For a sufficiently large absolute $K$, set
\[
 t^2=K\{\delta_n(f_0)+\eta+e_n^2+\varepsilon_n+x/n\},
 \qquad x\ge0.
\]
When $t^2<2$, the two probability bounds give a tail at most
$Ce^{-cx}$; when $t^2\ge2$, the event $\{h^2(p_0,p_{\widehat f_n})>t^2\}$ is empty.
Integrating this bound and using $n^{-1}\le e_n^2$ proves
\eqref{eq:ws-mean} with an additional $C\eta$. Letting $\eta\downarrow0$
completes the proof.
\end{proof}

\section{Proofs of the main results}\label{app:general}
This appendix gives the proofs of the results in
Sections~\ref{sec:parameter-budget}--\ref{sec:budget-consequences}, in their
order of appearance. Appendix~\ref{app:risk-tools} supplies the general
risk tools, Appendix~\ref{sec:losses} collects auxiliary likelihood and
complexity bounds, and Appendices~\ref{app:ffn} and
\ref{app:transformer-proof} give the neural constructions.

\subsection{Proofs for Section~\ref{sec:parameter-budget}}
\label{app:proofs-section2}
\begin{proof}[Proof of Corollary~\ref{cor:main-sieve-oracle}]
Apply Theorem~\ref{thm:ws-hellinger} with
$\mathcal S_n=\mathcal C_n$. Then
$\delta_n(f_0)=\delta(f_0;\mathcal C_n)$, and the bridge condition
\eqref{eq:main-rho-bridge} gives
$\sup_{f_0\in\cF}\delta_n(f_0)\le C_4A(\mathcal C_n)$.
Taking the supremum in \eqref{eq:ws-mean} proves
\eqref{eq:main-general-oracle} with $C_5=C_3\max\{1,C_4\}$.
\end{proof}

\begin{proof}[Proof of Theorem~\ref{thm:union}]
Set $W_B=B^\kappa\log^b(eB)$ and $\delta_B=c_0W_B^{-a/2}$, where
$0<c_0<1$ will be chosen below. Use the Hellinger pseudometric on indices,
$d(f,g)=h(\mathbb P_f,\mathbb P_g)$, so
\[
 \mathcal N(\delta,\mathcal F_U[B],d)
 =\mathcal N(\delta,\mathcal P_U[B],h).
\]
For all sufficiently large $B$, $\delta_B<1$, and
\eqref{eq:unioncover} with $\kappa,b\ge1$ gives
\[
 \log\mathcal N(\delta_B,\mathcal F_U[B],d)
 \le C'\{1+\log(1/c_0)\}W_B.
\]
Assumption~\ref{ass:complexity} supplies a target packing with logarithmic
cardinality at least $c_{\mathrm{pack}}c_0^{-2/a}W_B$ whenever
$\delta_B\le\delta_0$. Choose $c_0$ sufficiently small that
\[
 c_{\mathrm{pack}}c_0^{-2/a}>C'\{1+\log(1/c_0)\}.
\]
The packing is then larger than the candidate cover. Applying
Lemma~\ref{lem:packing-cover} with $R(f,g)=h^2(\mathbb P_f,\mathbb P_g)$,
$p=2$, $c_d=1$, and $\mathcal C=\mathcal F_U[B]$ gives the lower bound
$\delta_B^2=c_0^2W_B^{-a}$, proving \eqref{eq:unionlower}.
\end{proof}

\begin{proof}[Proof of Corollary~\ref{cor:main-separation}]
Divide the loop upper bound \eqref{eq:mainrisk} by the untied lower bound
\eqref{eq:unionlower}.
\end{proof}

\subsection{Proofs for Section~\ref{sec:examples-neural}}
\label{app:proofs-section3}
\begin{proof}[Proof of Proposition~\ref{prop:model-geometry}]
\emph{Conditional models.}
The calculations in Sections~\ref{app:gaussian}--\ref{app:bernoulli}
establish the two-sided
metric equivalence in part~(i). Lemma~\ref{lem:rho-bridge} establishes
part~(ii). The root-density envelope converts uniform index covers into
Hellinger brackets and internal Hellinger covers, proving part~(iii).
Finally, Lemma~\ref{lem:holder-packing} and the lower metric equivalence
give part~(iv), after a fixed rescaling of the packing radius, with
$a=2s/d$.

\emph{Energy-based generative model.}
Lemma~\ref{lem:ebm-geometry} proves parts~(i)--(iii), using
$d_{\mathrm E}=d_{\mathrm c}$ and $d_{\mathrm c}(f,g)\le
\|f-g\|_{L^2(\mu)}$. For part~(iv), use the zero-mean packing in
Lemma~\ref{lem:holder-packing}. Its elements satisfy
$d_{\mathrm c}(f_{j,\delta},f_{k,\delta})=
\|f_{j,\delta}-f_{k,\delta}\|_{L^2(\mu)}$.
The lower bound in \eqref{eq:ebm-geometry} transfers this packing to
Hellinger distance after a fixed rescaling of the radius, with exponent
$a=2s/d$.
\end{proof}

\noindent\textbf{Existence of the neural sieve MLE.}
For fixed $(q,r,T)$, the closed, bounded coefficient box is compact, and
the network output is continuous in its coefficients. The stabilized layer
normalizations preserve this continuity. In each of the four models, the
sample log likelihood is jointly measurable in the observations and
coefficients and is continuous in the coefficients. For the energy model,
continuity of the normalizing constant follows from bounded convergence,
since $|g|\le H$. The measurable maximum theorem therefore supplies a
measurable exact maximizer on the coefficient box, and hence an estimator
satisfying Definition~\ref{def:main-sieve-mle} for every
$\varepsilon_n\ge0$.

\begin{lemma}[Index approximation for looped neural classes]
\label{lem:neural-index-approx}
Fix $d,s,H$ and $j\in\{\mathrm F,\mathrm{Tr}\}$. For $j=\mathrm{Tr}$, also fix the token partition and
$\varepsilon_{\rm LN}>0$. There exist constants $C,c>0$ and
architecture-specific thresholds $q_0^j,r_0^j,M_0^j$, depending only on
these fixed quantities, such that, for fixed
$M_{\rm par}\ge M_0^j$, integers $q\ge q_0^j$, $r\ge r_0^j$, and
sufficiently large $T$, with $a=2s/d$, the following bound holds.

\begin{equation}\label{eq:neural-index-approx}
\sup_{f_0\in\Holder}\inf_{g\in\mathcal F_{L,j}(q,r,T)}
\|f_0-g\|_{L^2(\mu)}^2
\le C\biggl[\biggl(\frac{\log(erT)}{rT}\biggr)^a
+\exp\!\bigl(-cT^{1/h_0^j}\bigr)\biggr],
\end{equation}
where $h_0^{\mathrm F}=1$ and $h_0^{\mathrm{Tr}}=2$. The exponential term may
be omitted when $T\ge C\log(erT)$ for the residual FFN and when
$T\ge C\log^2(erT)$ for the Transformer.

For fixed admissible $q,r$, the initialization condition gives the simplified
bound $\sup_{f_0\in\Holder}\inf_{g\in\mathcal F_{L,j}(q,r,T)}
\|f_0-g\|_{L^2(\mu)}^2\lesssim\{\log(eT)/T\}^{a}$.
\end{lemma}

\begin{proof}[Proof of Lemma~\ref{lem:neural-index-approx}]
For $j\in\{\mathrm F,\mathrm{Tr}\}$, Theorems~\ref{thm:ffn-construction}
and \ref{thm:tr-construction} give squared index error at most $CK^{-2s}$
at every dyadic resolution $K$ whenever
\[
 T\ge C\{\log^{h_0^j}(eK)+(K^d/r)\log(eK)\},
 \qquad h_0^{\mathrm F}=1,\quad h_0^{\mathrm{Tr}}=2.
\]
Write $L=\log(erT)$. For sufficiently small constants $c_1,c_2>0$, set
\[
 J_0=\min\{c_1rT/L,\exp(c_2T^{1/h_0^j})\},\qquad
 K=2^{\lfloor d^{-1}\log_2J_0\rfloor}.
\]
For all sufficiently large $T$, uniformly over the admissible widths,
$K\ge2$ and $K^d\asymp J_0$. Moreover, $\log(eK)\le CL$ and
$\log(eK)\le C+c_2T^{1/h_0^j}/d$. The two iteration costs are therefore
at most $C'c_2^{h_0^j}T$ and $C''c_1T$, after absorbing the fixed additive
constant. Choosing $c_1,c_2$ sufficiently small makes their sum at most $T$.
The construction at this resolution consequently gives
\[
 \sup_{f_0\in\Holder}\inf_{g\in\mathcal F_{L,j}(q,r,T)}
 \|f_0-g\|_{L^2(\mu)}^2
 \le C\{[L/(rT)]^a+\exp(-cT^{1/h_0^j})\},
 \qquad a=2s/d.
\]
When $T\ge C L^{h_0^j}$, taking $J_0=c_1rT/L$ alone satisfies the same
iteration bound and removes the exponential term. This proves
\eqref{eq:neural-index-approx} and both initialization conditions.

For fixed admissible $q,r$, $\log(erT)/(rT)\lesssim\log(eT)/T$,
which gives the final simplified bound.
\end{proof}

\begin{proof}[Proof of Theorem~\ref{thm:neural-approx}]
Proposition~\ref{prop:model-geometry}(i) bounds squared Hellinger distance
by a constant times squared $L^2(\mu)$ distance in every likelihood model,
using $d_{\mathrm c}(f,g)\le\|f-g\|_{L^2(\mu)}$ for the energy model.
Applying this bound to the candidates from Lemma~\ref{lem:neural-index-approx} proves
\eqref{eq:conditional-approx}, with the same initialization conditions and
fixed-width simplification. Table~\ref{tab:parameter-counts} shows that
$B_L^j(q,r)$ is independent of $T$, giving the final budget statement.
\end{proof}

\begin{lemma}[Uniform covering of the neural index classes]
\label{lem:neural-index-cover}
Fix $M_{\rm par}>0$ and $j\in\{\mathrm F,\mathrm{Tr}\}$.
There exists a constant $C>1$, depending only on the fixed problem and
implementation parameters, such that the following bounds hold.

Let $G$ be defined by \eqref{eq:commonG}. For positive integers $q,r,T$ and
$0<\epsilon<1$,
\begin{equation}\label{eq:index-cover}
\log\mathcal N\bigl(\epsilon,\mathcal F_{L,j}(q,r,T),
\|\cdot\|_\infty\bigr)
\le C B_L^j(q,r)\{\log(C/\epsilon)+G(q,r,T)\}.
\end{equation}
For all sufficiently large $B$,
\begin{equation}\label{eq:architecture-union-cover}
\log\mathcal N(\epsilon,\mathcal F_{U,j}[B],\|\cdot\|_\infty)
\le C\{B^2\log(eB)+B\log(C/\epsilon)\}.
\end{equation}

\end{lemma}

\begin{proof}[Proof of Lemma~\ref{lem:neural-index-cover}]
Let $\Gamma$ denote the parameter-sensitivity bound in
Section~\ref{app:statistical}, \eqref{eq:sensitivity}: for parameter vectors
$\xi,\widetilde\xi$, it satisfies
$\|f_\xi-f_{\widetilde\xi}\|_\infty
\le\Gamma\|\xi-\widetilde\xi\|_\infty$.
We use its recurrence and the parameter-grid bound \eqref{eq:cover}. We first verify the required sensitivity
bounds on the full parameter boxes.

\emph{Residual FFN.}
We bound parameter sensitivity uniformly over the full coefficient box.
Put $L=C(q+1)^C(r+1)^C\ge2$. The residual recurrence gives $\|z_t\|_\infty+1\le L^{t+1}$. The state Lipschitz factor is at most $L$ and the parameter derivative is at most $L^{t+2}$. Thus \eqref{eq:sensitivity} gives $\Gamma\le(T+1)L^{C(T+1)}$, also for arbitrary untied sequences. Hence
\eqref{eq:commonG}
is valid, after changing constants.

The full predictor box satisfies \eqref{eq:cover}. Since
$\log\Gamma\le C G(q,r,T)$, this proves the residual-FFN instance of
\eqref{eq:index-cover}. The same sensitivity calculation applies to
arbitrary untied sequences and is used for
\eqref{eq:architecture-union-cover}.

\emph{Post-LN Transformer.}
Lemma~\ref{lem:postln-full-box-sensitivity} verifies the sensitivity
recurrence on the complete fitted box, including the affine initial state,
all four affine attention projections, both layer normalizations, and the
clipped affine output function. Its bound
\eqref{eq:postln-log-sensitivity} gives $\log\Gamma\le C G(q,r,T)$.
Equation~\eqref{eq:cover} therefore proves \eqref{eq:index-cover}.
The lemma also applies to every untied block sequence and supplies the
same sensitivity bound for \eqref{eq:architecture-union-cover}.

\emph{Budgeted union.}
For a budgeted untied class, the same calculations give
$\log\Gamma\le C(T+1)\log(eB)$ for every tuning triple in
$\mathcal I_{U,j}(B)$. Since $B_U^j(q,r,T)\le B$ and $T\le B$, the parameter
grid in \eqref{eq:cover} has logarithmic cardinality at most
\[
 C\{B^2\log(eB)+B\log(C/\epsilon)\}.
\]
There are at most $B^C$ admissible integer triples. Adding their logarithmic
cardinality proves \eqref{eq:architecture-union-cover}.
\end{proof}

\begin{proof}[Proof of Proposition~\ref{prop:architecture-cover}]
Apply Lemma~\ref{lem:neural-index-cover}. Proposition~\ref{prop:model-geometry}(iii) converts the loop
uniform cover into Hellinger brackets and converts the budgeted uniform
cover into an internal Hellinger cover. A fixed rescaling of the radius is
absorbed into $\log(C/u)$ or $\log(C/\epsilon)$, giving
\eqref{eq:conditional-entropy} and \eqref{eq:law-union-cover}.
For fixed admissible $q,r$, Table~\ref{tab:parameter-counts} and
\eqref{eq:commonG} give $B_L^j(q,r)=O(1)$ and $G(q,r,T)=O(T)$,
proving the simplified entropy bound.
\end{proof}

\subsection{Proofs for Section~\ref{sec:budget-consequences}}
\label{app:proofs-section4}
\begin{proof}[Proof of Theorem~\ref{thm:model-risk}]
Existence of a measurable exact sieve MLE was established in Section~\ref{app:proofs-section3}.
Theorem~\ref{thm:neural-approx} supplies the Hellinger approximation bound
\eqref{eq:conditional-approx}. Proposition~\ref{prop:architecture-cover}
gives \eqref{eq:conditional-entropy}, and
Lemma~\ref{lem:entropy-balance} verifies \eqref{eq:main-ws-integral} with
$e_{n,L}=e_{n,j}$ whenever $e_{n,j}<1$.
Proposition~\ref{prop:model-geometry}(ii) supplies the one-sided bridge
\eqref{eq:main-rho-bridge}. If $e_{n,j}=1$, the conclusion follows after
enlarging the constant because $h^2\le2$; otherwise
Corollary~\ref{cor:main-sieve-oracle} gives the first bound in
\eqref{eq:conditional-risk}.

For \eqref{eq:budget-iteration-risk}, fix $q,M_{\rm par}$ and take the
largest admissible width $r_B$ under budget $B$. Table~\ref{tab:parameter-counts}
gives $B_L^j(q,r)=\alpha_j(q)+(2q+1)r$, where
$\alpha_{\mathrm F}(q)=q(d+3)+1$ and
$\alpha_{\mathrm{Tr}}(q)=4q^2+q(d+2N+9)+1$.
Thus $r_B=\lfloor(B-\alpha_j(q))/(2q+1)\rfloor\asymp B$ for sufficiently
large $B$, and $r_B\ge r_0^j$. Consequently,
$\log(er_BT)/(r_BT)\lesssim\log(eBT)/(BT)$ and
$G(q,r_B,T)\lesssim T\log(eB)$ for $T\ge1$.
Using $B_L^j(q,r_B)\le B$ in the first risk bound, applying the
initialization refinement, and absorbing $\varepsilon_n\le n^{-1}$ into
$B\log(en)/n$ proves \eqref{eq:budget-iteration-risk}.
All comparison constants are uniform in $n,B,T$, and $f_0$.

For the fixed-budget specialization used in the Introduction, keep $B$
fixed in \eqref{eq:budget-iteration-risk}. Since
$\log(eBT)\lesssim_B\log(eT)$ and $\log(eB)$ is constant,
the bound becomes
$R_{L,j,n}^m\lesssim_B\{\log(eT)/T\}^{a}+(T+\log(en))/n$.
The iteration schedule \eqref{eq:fixed-T} then gives the displayed
fixed-budget rate in the Introduction.
\end{proof}

\begin{proof}[Proof of Theorem~\ref{thm:model-minimax}]
For $m\in\{\mathrm G,\mathrm{La},\mathrm{Be}\}$, use the packing in
Lemma~\ref{lem:holder-packing} at scale $\delta_n=K_n^{-s}$, where
$K_n\asymp n^{1/(2s+d)}$. Lemma~\ref{lem:rho-bridge} and Jensen's
inequality give
\[
\operatorname{KL}(\mathbb P_f^m\Vert\mathbb P_0^m)
\le \rho_{1/2}(p_f^m,p_0^m)\le C\|f\|_{L^2(\mu)}^2.
\]
The packing has logarithmic cardinality at least $cK_n^d$, and every
member satisfies $\|f\|_{L^2(\mu)}^2\le CK_n^{-2s}$. Choosing the fixed
multiplicative constant in $K_n$ sufficiently large makes the product-law
divergences a fixed small fraction of the logarithmic packing cardinality.
Fano's inequality and the lower metric equivalence in
Proposition~\ref{prop:model-geometry}(i) give the stated rate.

For $m=\mathrm E$, use the zero-mean packing of the same lemma.
Equation~\eqref{eq:ebm-geometry} gives
$\operatorname{KL}(\mathbb P_f^{\mathrm E}\Vert
\mathbb P_0^{\mathrm E})\le CK_n^{-2s}$ for every packing member.
The same Fano argument applies.
Assigning an arbitrary distribution estimator to its closest packing member
converts an incorrect assignment into Hellinger error at least half the
packing separation; estimates not dominated by the model's reference measure
are handled with a common dominating measure. In every model the resulting squared Hellinger
risk is at least
$cK_n^{-2s}\asymp n^{-2s/(2s+d)}=n^{-a/(a+1)}$, proving
\eqref{eq:model-minimax}.
\end{proof}

\begin{proof}[Proof of Proposition~\ref{prop:budget-lower}]
Equation~\eqref{eq:law-union-cover} gives \eqref{eq:unioncover} with
$\kappa=2$ and $b=1$ for the induced law class, while
Proposition~\ref{prop:model-geometry}(iv) verifies the target-packing
assumption. Theorem~\ref{thm:union} yields \eqref{eq:concretelower},
uniformly in $n$ and over all sufficiently large $B$.
\end{proof}

\begin{proof}[Proof of Theorem~\ref{thm:fixed-budget}]
Fix the family $j$ and its admissible $(q,r)$. The parameter count is
constant, and \eqref{eq:commonG} gives $G(q,r,T_n)=O(T_n)$. Hence the
fitted-sieve term $e_{n,j}^2$ is $O(T_n/n)$. With the choice
\eqref{eq:fixed-T},
\[
 \left\{\frac{\log(erT_n)}{rT_n}\right\}^{a}
 +\frac{T_n}{n}
 \le Cn^{-a/(a+1)}\{\log(en)\}^{a/(a+1)}.
\]
The architecture-specific initialization remainder is negligible and
$\varepsilon_n\le n^{-1}$. The general bound \eqref{eq:conditional-risk}
in Theorem~\ref{thm:model-risk} proves \eqref{eq:fixed-rate}.

At the fixed cap $B$, Proposition~\ref{prop:budget-lower} bounds the
untied risk below by $c_{\mathrm U}\{B^2\log(eB)\}^{-a}$. Dividing
\eqref{eq:fixed-rate} by this bound proves the limit in
\eqref{eq:fixed-separation-ratios}.
\end{proof}

\begin{proof}[Proof of Theorem~\ref{thm:growing-budget}]
By Table~\ref{tab:parameter-counts}, $B_L^j(q,r)$ is affine and strictly
increasing in $r$ with slope $2q+1$. Thus the maximizing choice of $r_n$
satisfies the parameter budget. Since $q_n\asymp\log n$ and
$B_n\asymp n^{\gamma_B}$, the counts
give $r_n\asymp n^{\gamma_B}/\log n$ for all sufficiently large $n$. The
condition $\gamma_B<1/(a+1)$ implies that $W_n/r_n$ diverges at a positive
polynomial rate. Thus rounding $T_n$ up changes its order by at most a
constant, $r_nT_n\asymp W_n$, and both initialization thresholds hold.

Because $\ell_n\asymp\log n$, the approximation term in
\eqref{eq:conditional-risk} is bounded by
\[
 C\left(\frac{\log n}{W_n}\right)^a.
\]
Moreover, $B_L^j(q_n,r_n)\asymp B_n$ and
$G(q_n,r_n,T_n)\asymp T_n\ell_n$, so the fitted-sieve term is
\[
 e_{n,j}^2\le C\frac{B_nT_n\ell_n}{n}
 \asymp C\frac{W_nq_n\ell_n}{n}.
\]
Since $q_n\ell_n\asymp(\log n)^2$, the definition of $W_n$
gives the order in \eqref{eq:rate} for both displays. This order tends to
zero, so $e_{n,j}<1$ eventually and the entropy lemma applies. The tolerance
is smaller, completing the upper bound. Proposition~\ref{prop:budget-lower}
at $B=B_n$ gives $R_{U,j,n}^m(B_n)\ge c_{\mathrm U}\{B_n^2\log(eB_n)\}^{-a}$.
Dividing the upper
bound by the untied lower bound produces the power
$-a/(a+1)+2a\gamma_B<0$, which dominates the logarithmic factors and proves
\eqref{eq:growing-separation-ratios}.
\end{proof}

\section{Auxiliary likelihood and complexity bounds}\label{sec:losses}
This appendix collects the likelihood calculations, H\"older target packing,
and parameter-sensitivity estimates used in Appendix~\ref{app:general}.
In the likelihood calculations, $f$ and $g$ take values in $[-H,H]$.

\subsection{Gaussian regression}\label{app:gaussian}
Suppose $Y=f(X)+\varepsilon$, where $X\sim\mu$ and
$\varepsilon\sim N(0,\sigma^2)$ is independent of $X$, with known
$\sigma^2>0$. Direct calculation gives
\[
h^2(\mathbb P_f,\mathbb P_g)
=2\int\left[1-\exp\left\{-\frac{(f-g)^2}{8\sigma^2}\right\}\right]d\mu.
\]
Because $|f-g|\le2H$, positive constants depending only on $H$ and
$\sigma$ bound this display above and below by $\|f-g\|_{L^2(\mu)}^2$. Maximum
likelihood over an index-function class is empirical least squares.

\subsection{Laplace regression}\label{app:laplace}
Suppose $X\sim\mu$ and $Y=f(X)+\varepsilon$, where $\varepsilon$ is
independent of $X$ and has density
$(2\tau)^{-1}\exp(-|u|/\tau)$ for fixed $\tau>0$. For
$v=|f(x)-g(x)|/\tau$, the conditional Hellinger affinity is
$(1+v/2)e^{-v/2}$, and hence
\[
h^2(\mathbb P_f,\mathbb P_g)
=2\int\{1-(1+v/2)e^{-v/2}\}\,d\mu.
\]
On the bounded range $0\le v\le2H/\tau$, this quantity is uniformly
equivalent to $\|f-g\|_{L^2(\mu)}^2$. Maximum likelihood is empirical absolute
deviation.

\subsection{Bernoulli binary-response model}\label{app:bernoulli}
Suppose $Y\mid X=x$ is Bernoulli with success probability
$\operatorname{sigmoid}\{f(x)\}$. On the bounded logit interval, the
derivatives of both square-root probabilities are bounded. The derivative
of the increasing coordinate is also bounded away from zero, so the mean
value theorem gives constants $c_{\mathrm{Be}},C_{\mathrm{Be}}>0$, depending
only on $H$, such that
\[
c_{\mathrm{Be}}\|f-g\|_{L^2(\mu)}^2\le h^2(\mathbb P_f,\mathbb P_g)
\le C_{\mathrm{Be}}\|f-g\|_{L^2(\mu)}^2.
\]
Maximum likelihood is empirical logistic-loss minimization.

For all three families, the square-root conditional densities are absolutely
continuous in the bounded index parameter and admit an $L^2$ envelope $L$
for their almost-everywhere derivative. Here the norm is under the joint
dominating measure $\mu\otimes\nu_Y$, where $\nu_Y$ is Lebesgue measure
for the regression families and counting measure for Bernoulli. Thus
\[
|\sqrt{p_t}-\sqrt{p_s}|\le L|t-s|,
\qquad \|L\|_{L^2(\mu\otimes\nu_Y)}<\infty.
\]
A uniform $\epsilon$-net of index functions therefore yields density brackets
with endpoints $(\sqrt{p_{g_j}}-\epsilon L)_+^2$ and
$(\sqrt{p_{g_j}}+\epsilon L)^2$, of Hellinger width at most
$2\epsilon\|L\|_{L^2(\mu\otimes\nu_Y)}$. Predictor covers therefore lift to Hellinger brackets
of the target density family and to Hellinger covers of the architecture
families with the same logarithmic order. The metric equivalences transfer
the loop approximation and the disjoint-bump packing from $L^2(\mu)$ to
Hellinger distance. The positive likelihood approximation condition is verified next.

\begin{lemma}[Positive likelihood approximation on the bounded index range]
\label{lem:rho-bridge}
For each of the Gaussian, Laplace, and Bernoulli conditional models,
there is a constant $C$ such that
\[
\rho_{1/2}(p_f,p_g)\le C\|f-g\|_{L^2(\mu)}^2
\le C h^2(\mathbb P_f,\mathbb P_g),\qquad
f,g:\mathcal X\to[-H,H].
\]
\end{lemma}
\begin{proof}
Condition on $X=x$ and put $\Delta=f(x)-g(x)$. In the Gaussian model,
\[
 \int p_f^{3/2}p_g^{-1/2}\,d\nu_Y
 =\exp\{3\Delta^2/(8\sigma^2)\}.
\]
For Laplace regression, with $v=|\Delta|/\tau$, the corresponding integral is
\[
 \{3e^{v/2}+e^{-3v/2}\}/4.
\]
For Bernoulli regression, write $b(t)=\log(1+e^t)$. The integral is
\[
 \exp\!\left\{b\!\left(\frac{3f-g}{2}\right)
       -\frac32b(f)+\frac12b(g)\right\}.
\]
Each display equals one at $\Delta=0$, has first derivative zero there,
and has uniformly bounded second derivative on the bounded index range.
Thus it is at most $1+C\Delta^2$. Integrating in $x$ and using
$\rho_{1/2}(p_f,p_g)=2\{\int p_f^{3/2}p_g^{-1/2}\,d(\mu\otimes\nu_Y)-1\}$
proves the first inequality. The second is the metric equivalence proved in
the three preceding subsections.
\end{proof}

\subsection{Energy-based generative models}\label{app:ebm}
Throughout this subsection, $\mu$ is uniform probability measure on
$[0,1]^d$ and $f,g$ take values in $[-H,H]$. Densities are defined by
\eqref{eq:ebm-model}. Then
$e^{-2H}\le p_f\le e^{2H}$.

\begin{lemma}[Geometry and likelihood approximation of bounded energy-based generative models]
\label{lem:ebm-geometry}
There exist constants $C,c>0$, depending only on $H$, such that
the Kullback--Leibler divergence $\operatorname{KL}$ satisfies
\begin{equation}\label{eq:ebm-geometry}
 c d_{\mathrm c}(f,g)^2
 \le h^2(\mathbb P_f,\mathbb P_g)
 \le \operatorname{KL}(\mathbb P_f\Vert\mathbb P_g)
 \le C d_{\mathrm c}(f,g)^2,
\end{equation}
and
\begin{equation}\label{eq:ebm-rho}
 \rho_{1/2}(p_f,p_g)\le C h^2(\mathbb P_f,\mathbb P_g).
\end{equation}
A uniform $u$-cover of a bounded index class induces density brackets
of Hellinger width at most $2\sinh(u)$ with the same cardinality.
\end{lemma}

\begin{proof}
Put $\delta=g-f$ and $f_t=f+t\delta$. Differentiation of the log
normalizer, justified by boundedness, gives
\[
 \frac{d^2}{dt^2}\log Z_{f_t}
 =\operatorname{Var}_{\mathbb P_{f_t}}(\delta).
\]
Consequently,
\[
 \operatorname{KL}(\mathbb P_f\Vert\mathbb P_g)
 =\int_0^1(1-t)\operatorname{Var}_{\mathbb P_{f_t}}(\delta)\,dt
 \le \tfrac12e^{2H}\inf_{c\in\R}\|\delta-c\|_2^2.
\]
For the lower bound, the logarithm is $e^H$-Lipschitz on
$[e^{-H},e^H]$, the range of the square-root densities. Thus
\[
 \begin{split}
 d_{\mathrm c}(f,g)
 &\le\|f-g-\log Z_f+\log Z_g\|_2\\
 &=\|\log p_f-\log p_g\|_2
 \le 2e^H\|\sqrt{p_f}-\sqrt{p_g}\|_2.
 \end{split}
\]
Together with $h^2\le\operatorname{KL}$, this proves
\eqref{eq:ebm-geometry}.

To prove \eqref{eq:ebm-rho}, put $t=p_f/p_g$. Then
$e^{-4H}\le t\le e^{4H}$ and $\int p_g(t-1)\,d\mu=0$, so
\[
 \rho_{1/2}(p_f,p_g)
 =2\int p_g\{t^{3/2}-1-\tfrac32(t-1)\}\,d\mu.
\]
The identity
$2\{t^{3/2}-1-\tfrac32(t-1)\}=(\sqrt t-1)^2(2\sqrt t+1)$
therefore gives
\[
 \rho_{1/2}(p_f,p_g)
 \le (2e^{2H}+1)h^2(\mathbb P_f,\mathbb P_g),
\]
which proves \eqref{eq:ebm-rho}.

If $\|f-f_j\|_\infty\le u$, then
$|\log Z_f-\log Z_{f_j}|\le u$ and
\[
 e^{-2u}p_{f_j}\le p_f\le e^{2u}p_{f_j}.
\]
The square-root width of this bracket is
$(e^u-e^{-u})\|\sqrt{p_{f_j}}\|_2=2\sinh(u)$.
\end{proof}

\subsection{H\"older target packing}\label{app:holder-complexity}
\begin{lemma}[Local packing of the H\"older class]\label{lem:holder-packing}
Fix $d\ge1$, $s>0$ and $H>0$, and let $\mu$ be uniform probability
measure on $[0,1]^d$. There exist constants $c,C,\delta_0>0$, depending
only on $d,s,H$, such that, for every $0<\delta\le\delta_0$, the class
$\Holder$ contains $f_{1,\delta},\ldots,f_{M_\delta,\delta}$ satisfying
\[
 \min_{j\ne k}\|f_{j,\delta}-f_{k,\delta}\|_{L^2(\mu)}\ge8\delta,
 \qquad
 \max_j\|f_{j,\delta}\|_{L^2(\mu)}\le C\delta,
 \qquad
 \log M_\delta\ge c\delta^{-d/s}.
\]
The functions can additionally be chosen so that
$\int f_{j,\delta}\,d\mu=0$ for every $j$.
\end{lemma}
\begin{proof}
Choose a nonzero smooth function $\psi$ supported strictly inside
$(0,1)^d$ with integral zero. Such a function is obtained by
differentiating a suitable compactly supported smooth function in one
coordinate. For an integer $m\ge1$ and
$\omega\in\{0,1\}^{m^d}$, define
\[
 f_\omega(x)=c_*m^{-s}
 \sum_{j\in\{0,\ldots,m-1\}^d}\omega_j\psi(mx-j).
\]
For sufficiently small fixed $c_*>0$, depending only on $d,s,H$ and
$\psi$, these functions belong to $\Holder$. Derivatives of order
$v\le k$ are bounded by $Cm^{v-s}$, and the order-$k$ derivatives
have uniformly bounded $\alpha$-H\"older seminorm. Separation of the
bump supports gives the same bound across distinct cells. Every
$f_\omega$ has integral zero, and disjointness of the supports gives
\[
 \|f_\omega-f_{\omega'}\|_{L^2(\mu)}^2
 =c_*^2m^{-2s-d}\|\psi\|_{L^2(\mathbb R^d)}^2
   \sum_j|\omega_j-\omega'_j|.
\]
A binary code with Hamming separation at least a fixed positive
fraction of $m^d$ has logarithmic cardinality at least $cm^d$ for
all sufficiently large $m$. The corresponding functions therefore
have pairwise $L^2(\mu)$ distances at least $c_{\mathrm{sep}}m^{-s}$,
while every function has norm at most $Cm^{-s}$, for a fixed
$c_{\mathrm{sep}}>0$.

For sufficiently small $\delta$, take
$m=\lfloor(c_{\mathrm{sep}}/(8\delta))^{1/s}\rfloor$.
Then $m\asymp\delta^{-1/s}$ and the pairwise distances are at least
$8\delta$. The norm bound and logarithmic cardinality become
$C\delta$ and $c\delta^{-d/s}$, respectively. Relabeling these
functions proves the lemma, including the zero-mean assertion.
\end{proof}

\subsection{Parameter sensitivity and architecture covers}\label{app:statistical}

Let $\xi,\widetilde\xi$ be two parameter vectors and put
$\delta=\|\xi-\widetilde\xi\|_\infty$. Write $e_t$ for a uniform
bound on the distance between their states after $t$ iterations.
Suppose that nonnegative constants, uniform on the full parameter box,
satisfy
\[
e_0\le c_E\delta,\qquad e_{t+1}\le a_te_t+b_t\delta,\qquad
|f_\xi-f_{\widetilde\xi}|\le c_Re_T+d_R\delta.
\]
For the untied class, assume the same bounds hold for every admissible
sequence of block parameters. Iteration gives
\begin{equation}\label{eq:sensitivity}
\Gamma=\max\left\{1,d_R+c_R\left(c_E\prod_{t<T}a_t+
\sum_{j<T}b_j\prod_{t=j+1}^{T-1}a_t\right)\right\},\qquad
\|f_\xi-f_{\widetilde\xi}\|_\infty\le\Gamma\delta.
\end{equation}
Let $P$ denote the total parameter count and
$\mathcal G=\{f_\xi:\xi\in[-M_{\rm par},M_{\rm par}]^P\}$ the corresponding
predictor class. Partition the $P$-dimensional parameter box into cells of
$\ell^\infty$ diameter at most $\epsilon/\Gamma$ and choose one parameter
vector from each nonempty cell. Their predictors form an internal cover, so
\begin{equation}\label{eq:cover}
\log\mathcal N(\epsilon,\mathcal G,\|\cdot\|_\infty)
\le P\log(1+4M_{\rm par}\Gamma/\epsilon).
\end{equation}
The same recurrence applies to shared and untied parameters, while their
parameter counts differ. Applied to the loop classes and the budgeted
unions, this calculation gives the two uniform covering bounds in
Lemma~\ref{lem:neural-index-cover}.
Proposition~\ref{prop:model-geometry}(iii) transfers these predictor covers
into Hellinger brackets and internal Hellinger covers.

The following lemma supplies explicit sensitivity bounds for the complete
post-LN Transformer class.

\begin{lemma}[Uniform sensitivity of the complete post-LN class]
\label{lem:postln-full-box-sensitivity}
Consider the Transformer of Section~\ref{sec:neural-estimators}, with
$x\in[0,1]^d$, positive integers $q,r,N$, parameter bound
$M_{\rm par}>0$, and fixed $\varepsilon_{\rm LN}>0$.
Write $m=M_{\rm par}$, $\varepsilon=\varepsilon_{\rm LN}$, and
$R=\max\{1,m(d+1),m(\sqrt q+1)\}$.
There exist $L_{\rm blk},K_{\rm blk}>0$, depending only on
$d,q,r,M_{\rm par},\varepsilon_{\rm LN}$, with the following properties.
For matrices, $\|\cdot\|_\infty$ denotes the maximum absolute entry.
For any two states $S,\widetilde S\in\mathbb R^{N\times q}$ with
$\|S\|_\infty,\|\widetilde S\|_\infty\le R$, and any two block
parameter vectors $\theta,\widetilde\theta$ in the complete prescribed box,
\begin{equation}\label{eq:postln-one-block-sensitivity}
 \|\Phi_\theta(S)-\Phi_{\widetilde\theta}(\widetilde S)\|_\infty
 \le L_{\rm blk}\|S-\widetilde S\|_\infty
      +K_{\rm blk}\|\theta-\widetilde\theta\|_\infty.
\end{equation}
Every initial state and every block output has norm at most $R$, so
\eqref{eq:postln-one-block-sensitivity} applies at every iteration.
If $\delta=\|\xi-\widetilde\xi\|_\infty$ is the distance between
two complete predictor parameter vectors and
$e_t=\sup_{x\in[0,1]^d}\|S_t(x)-\widetilde S_t(x)\|_\infty$, where
$S_t$ and $\widetilde S_t$ are the states after $t$ updates under
$\xi$ and $\widetilde\xi$, respectively, then
\begin{equation}\label{eq:postln-full-sensitivity-recurrence}
 e_0\le(d+1)\delta,\qquad
 e_{t+1}\le L_{\rm blk}e_t+K_{\rm blk}\delta,
 \qquad
 \|f_\xi-f_{\widetilde\xi}\|_\infty
 \le Nqm\,e_T+(NqR+1)\delta.
\end{equation}
These inequalities hold both for shared block parameters and for every
admissible sequence of untied block parameters. Consequently, for every
fixed numerical $C_6\ge2$, the sensitivity constant in
\eqref{eq:sensitivity} satisfies
\begin{equation}\label{eq:postln-log-sensitivity}
 \log\Gamma
 \le C\{\log(T+1)+(T+1)\log(C_6(q+1)(r+1))\},\qquad T\ge1,
\end{equation}
where $C$ depends only on $d,N,M_{\rm par},\varepsilon_{\rm LN}$.
The bounds hold with all reference-token coordinates treated as free
fitted parameters.
\end{lemma}

\begin{proof}
Throughout the proof, $\delta$ bounds all coefficient differences being
compared. We first bound layer normalization, then combine attention and
FFN perturbations into a one-block bound. Including the input and output
functions gives the recurrence; iterating it yields the bound on $\log\Gamma$.

\noindent\textbf{Step 1: Layer normalization.}
Put $\ell=m\sqrt{q/\varepsilon}$ and $\nu=\sqrt q+1$. Let
$P_q=I_q-q^{-1}\mathbf1\mathbf1^\top$, and define
\[
 h(z)=\frac{P_qz}{s(z)},\qquad
 s(z)=\{q^{-1}\|P_qz\|_2^2+\varepsilon\}^{1/2}.
\]
The row-wise layer normalization is
$\operatorname{LN}_{\gamma,\beta}(z)=\gamma\odot h(z)+\beta$.
Since $\|h(z)\|_2\le\sqrt q$, its output obeys
\begin{equation}\label{eq:postln-global-radius}
 \|\operatorname{LN}_{\gamma,\beta}(z)\|_\infty
 \le m(\sqrt q+1)
\end{equation}
for every $z\in\mathbb R^q$ and every admissible gain and bias.
The positive stabilizer gives the derivative
\[
 Dh(z)=s(z)^{-1}P_q
       -\frac{(P_qz)(P_qz)^\top}{q\,s(z)^3}.
\]
On the span of $P_qz$, its eigenvalue is
$\varepsilon/s(z)^3$; on the orthogonal directions within the centered
subspace it is $1/s(z)$, and it vanishes on the constant direction.
Since $s(z)\ge\sqrt\varepsilon$, every eigenvalue is at most
$\varepsilon^{-1/2}$. When $P_qz=0$, the derivative is
$\varepsilon^{-1/2}P_q$. Thus
$\|Dh(z)\|_{2\to2}\le\varepsilon^{-1/2}$ globally, including $q=1$.
Integration along the segment between $z$ and $\widetilde z$, followed
by $\|v\|_2\le\sqrt q\|v\|_\infty$, gives
\[
 \|h(z)-h(\widetilde z)\|_\infty
 \le\sqrt{q/\varepsilon}\,\|z-\widetilde z\|_\infty.
\]
Separating the changes in the input, gain, and bias therefore gives
\begin{equation}\label{eq:postln-joint-ln-bound}
 \|\operatorname{LN}_{\gamma,\beta}(z)
       -\operatorname{LN}_{\widetilde\gamma,\widetilde\beta}
                         (\widetilde z)\|_\infty
 \le\ell\|z-\widetilde z\|_\infty+\nu\delta,
\end{equation}
whenever each gain and bias differs by at most $\delta$.
This estimate requires no bound on the arguments of layer normalization.

\noindent\textbf{Step 2: Attention.}
Put $A=m(qR+1)$ and $\zeta=1+4\sqrt q\,A^2$, and define
\[
 a_{\rm att}=(qm)^2\zeta,\qquad
 b_{\rm att}=qm\zeta(qR+1)+qA+1.
\]
Set
$u=\|S-\widetilde S\|_\infty$ and let $\delta$ bound the differences
of all block coefficients. Each of the affine query, key, and value
projections has entries bounded in magnitude by $A$, and each difference
is bounded by
\[
 D=qm\,u+(qR+1)\delta.
\]
Indeed, separate the change in the state from the change in the
projection matrix and its bias. The scaled query--key logits
$Z=QK^\top/\sqrt q$ consequently satisfy
\[
 \|Z-\widetilde Z\|_\infty\le2\sqrt q\,A D.
\]
For softmax $\sigma$ on $\mathbb R^N$, its Jacobian obeys
\[
 \|D\sigma(z)v\|_1
 =\sum_i\sigma_i(z)
       \left|v_i-\sum_j\sigma_j(z)v_j\right|
 \le2\|v\|_\infty.
\]
Hence each row of the attention-weight matrix changes by at most
$4\sqrt q\,A D$ in $\ell^1$ norm. Writing $V_{\rm av}=\sigma(Z)V$, with
softmax applied separately to each row, the probability weights imply
\[
 \|V_{\rm av}\|_\infty\le A,\qquad
 \|V_{\rm av}-\widetilde V_{\rm av}\|_\infty
 \le A(4\sqrt q\,A D)+D=\zeta D.
\]
The fourth affine projection, applied to $V_{\rm av}$, therefore gives
\begin{align*}
 \|\operatorname{Attn}_\theta(S)
        -\operatorname{Attn}_{\widetilde\theta}(\widetilde S)\|_\infty
 &\le qm\zeta D+(qA+1)\delta\\
 &=a_{\rm att}u+b_{\rm att}\delta.
\end{align*}
All four projection matrices and all four projection biases have been
included. The probability-row bound is why no factor $N$ is needed in
this attention estimate.

\noindent\textbf{Step 3: FFN and the complete block.}
Put $a_{\rm ff}=qrm^2$ and $b_{\rm ff}=2rA+1$.
Adding the first residual and using \eqref{eq:postln-joint-ln-bound}
gives, for the first normalized states $Y,\widetilde Y$,
\[
 v:=\|Y-\widetilde Y\|_\infty
 \le\ell(1+a_{\rm att})u+(\ell b_{\rm att}+\nu)\delta.
\]
By \eqref{eq:postln-global-radius}, both these states have norm at most
$R$. For the row-wise FFN, the hidden affine values and their ReLU
outputs have norm at most $A$. ReLU is $1$-Lipschitz, so the change of a
hidden output is at most $qm v+(qR+1)\delta$. The second FFN affine
map then yields
\begin{align*}
 \|\operatorname{FFN}_\theta(Y)
       -\operatorname{FFN}_{\widetilde\theta}(\widetilde Y)\|_\infty
 &\le rm\{qm v+(qR+1)\delta\}+(rA+1)\delta\\
 &=a_{\rm ff}v+b_{\rm ff}\delta.
\end{align*}
Adding the second residual and applying the second layer-normalization
bound gives
\[
 \|\Phi_\theta(S)-\Phi_{\widetilde\theta}(\widetilde S)\|_\infty
 \le\ell(1+a_{\rm ff})v+(\ell b_{\rm ff}+\nu)\delta,
\]
which proves \eqref{eq:postln-one-block-sensitivity} with
\begin{align*}
 L_{\rm blk}&=\ell^2(1+a_{\rm ff})(1+a_{\rm att}),\\
 K_{\rm blk}&=\ell(1+a_{\rm ff})(\ell b_{\rm att}+\nu)
               +\ell b_{\rm ff}+\nu.
\end{align*}

\noindent\textbf{Step 4: Input, output, and iteration.}
We include the endpoints and the first application of the block. Every coordinate of data token $j$ at input has magnitude at most
$m(|I_j|+1)\le m(d+1)$, because $x\in[0,1]^d$. Each controller and
reference coordinate has magnitude at most $m$. Thus the initial state
has norm at most $R$, even though it has not been layer-normalized.
The same calculation with coefficient differences gives
$e_0\le(d+1)\delta$, including the free controller and reference rows.
Every subsequent block output satisfies
\eqref{eq:postln-global-radius}. The one-block bound therefore applies
at every time, even when the block parameter vector changes between applications within
the prescribed box. Since every individual block's parameter difference
is at most the max-norm difference of the full untied vector, the same
recurrence holds for shared and untied parameterizations.

Before clipping, the readout is an affine function of $Nq$ coordinates.
Separating its state and coefficient changes bounds its difference by
$Nqm\,e_T+(NqR+1)\delta$. The fixed clipping map $\Pi_H$ is
$1$-Lipschitz, so this also bounds the predictor difference. This proves
\eqref{eq:postln-full-sensitivity-recurrence}.

\noindent\textbf{Step 5: The logarithmic sensitivity bound.}
To prove \eqref{eq:postln-log-sensitivity}, we bound all constants by a
single polynomial majorant. Put
\[
 c=\max\{2,d+1,N,m,\varepsilon^{-1/2}\},\qquad
 P=32c^{16}(q+1)^8(r+1).
\]
Writing $\bar q=q+1$ and $\bar r=r+1$ only in the next bounds, the definitions give
\[
 R\le c^2\bar q^{1/2},\quad A\le c^3\bar q^{3/2},\quad
 \zeta\le5c^6\bar q^{7/2},\quad \ell\le c^2\bar q^{1/2},\quad
 \nu\le2\bar q^{1/2},
\]
\[
 1+a_{\rm att}\le6c^8\bar q^{11/2},\quad
 b_{\rm att}\le7c^9\bar q^6,\quad
 1+a_{\rm ff}\le2c^2\bar q\bar r,\quad
 b_{\rm ff}\le3c^3\bar q^{3/2}\bar r.
\]
In particular,
$L_{\rm blk}\le12c^{14}\bar q^{15/2}\bar r$ and
$K_{\rm blk}\le23c^{15}\bar q^8\bar r$, so
\[
 \max\{2,L_{\rm blk},K_{\rm blk},d+1,Nqm,NqR+1\}\le P.
\]
Iterating the recurrence gives
\[
 e_T\le\delta\left\{(d+1)L_{\rm blk}^T
                 +K_{\rm blk}\sum_{j=0}^{T-1}L_{\rm blk}^{j}\right\}
 \le\delta(T+1)P^{T+1}.
\]
Thus the constant from \eqref{eq:sensitivity} obeys
\[
 \Gamma\le(T+2)P^{T+2}
          \le(T+1)P^{2(T+1)},\qquad T\ge1.
\]
Finally, for any fixed numerical $C_6\ge2$,
\[
 \log P\le
 \left\{8+\frac{\log(32c^{16})}{\log2}\right\}
 \log(C_6(q+1)(r+1)).
\]
Taking logarithms proves \eqref{eq:postln-log-sensitivity}; in particular,
the fixed numerical $C_6$ in \eqref{eq:commonG} need not be changed.
\end{proof}

\begin{lemma}[Bracketing resolution for a complete parameter class]
\label{lem:entropy-balance}
Suppose $P\ge1$, $G\ge0$, and, for some constant $C>1$,
$H(u,\mathcal P,h)\le C P\{\log(C/u)+G\}$ for $0<u<1$.
For a sufficiently large constant $K$, put
$e_n^2=K P\{\log(en)+G\}/n$. If $e_n<1$, then
\eqref{eq:ws-integral} holds for $e_n\le t\le\sqrt2$.
\end{lemma}
\begin{proof}
Monotonicity of bracketing entropy extends its bound to bounded positive
$u$ by increasing the constant. On the integration interval, its square
root is at most
$C\sqrt{P\{\log(C/t^2)+G\}}$. The interval has length at most
$\sqrt2t$. Since $t\ge e_n\ge n^{-1/2}$, its integral is bounded by
\[
 C t\sqrt{P\{\log(en)+G\}}
 \le C K^{-1/2}\sqrt n\,t^2.
\]
Choosing $K$ sufficiently large proves the assertion for the fixed
Wong--Shen constants.
\end{proof}

\begin{lemma}[Untied budget-union entropy]\label{lem:untied-union-cover}
Suppose there is a constant $C>1$ such that every architecture in the untied budget union satisfies
$B_U\le B$, $T\le B$, the induced-law map satisfies
$h(\mathbb P_f,\mathbb P_g)\le C\|f-g\|_\infty$, and
\[
 \log\Gamma\le C(T+1)^\beta\log(eB),\qquad \beta\ge1,
\]
uniformly over its full parameter box. If
$|\mathcal I_U(B)|\le B^C$, then \eqref{eq:unioncover} holds with
$\kappa=\beta+1$ and $b=1$.
\end{lemma}
\begin{proof}
For each architecture, the metric transfer, \eqref{eq:cover}, and $T\le B$
give
\[
 \log\mathcal N(\epsilon,\mathcal P_U,h)
 \le C\{B^{\beta+1}\log(eB)+B\log(1/\epsilon)\}.
\]
Taking the union adds at most
$\log|\mathcal I_U(B)|\le C\log(eB)$, which is absorbed by the first
term. This proves the claim.
\end{proof}
In the FFN and post-LN examples, the dense counts imply $q,r\le B$.
Because each untied update contains at least one fitted scalar, $T\le B$.
Thus there are at most $B^3$
integer triples. Their sensitivity bounds have $\beta=1$, so the lemma
gives $\kappa=2$ separately for the complete FFN and Transformer unions.

\section{Residual FFN approximation construction}\label{app:ffn}
We construct a shared residual block that approximates a H\"older target
at a prescribed dyadic resolution. The following result is the
finite-resolution input to Lemma~\ref{lem:neural-index-approx}.

\begin{theorem}[Finite-resolution construction for looped residual FFNs]
\label{thm:ffn-construction}
Fix $d\ge1$, $s>0$ and $H>0$, and let $\mu$ be uniform probability measure
on $[0,1]^d$. There exist a constant $C>0$ and thresholds
$q_0^{\mathrm F},r_0^{\mathrm F},M_0^{\mathrm F}>0$, depending only on
$d,s,H$, with the following property. For every
$M_{\rm par}\ge M_0^{\mathrm F}$, integers $q\ge q_0^{\mathrm F}$,
$r\ge r_0^{\mathrm F}$, dyadic resolution $K=2^m$ with $m\ge1$, and
$f_0\in\Holder$, there is one admissible input-function, update-block, and output-function parameter
choice for the residual FFN of Section~\ref{sec:neural-estimators} whose
looped candidates $g_T\in\mathcal F_{L,\mathrm F}(q,r,T)$ satisfy
\begin{equation}\label{eq:ffn-resolution-error}
 \|g_T-f_0\|_{L^2(\mu)}^2\le CK^{-2s}
\end{equation}
for every integer
\begin{equation}\label{eq:ffn-resolution-time}
 T\ge C\{\log(eK)+(K^d/r)\log(eK)\}.
\end{equation}
The parameter choice may depend on $f_0,K,q,r$, but is the same for all
such $T$. Every scalar coefficient belongs to
$[-M_{\rm par},M_{\rm par}]$, and all endpoint and block coordinates are
counted in $B_L^{\mathrm F}(q,r)$, which is independent of $T$.
\end{theorem}

\noindent\textbf{Proof roadmap.}
We approximate $f_0$ on a grid of $K^d$ cells by local Taylor polynomials
and encode their quantized coefficients as scalar records.
Section~\ref{app:base-engine} constructs one fixed-dimensional shared
evaluator: it decodes a record, evaluates its polynomial, and uses a
fine-cell gate to retain the contribution from the cell containing $x$.
Each record costs $O(\log(eK))$ iterations. Section~\ref{app:ffn-selection}
adds a spatial selector that loads only the group of records containing
$x$. Choosing $N_{\rm grp}\asymp\min\{r,K^d\}$ groups makes the number
of records scanned comparable to $1+K^d/r$. Including initialization
therefore gives the iteration bound \eqref{eq:ffn-resolution-time}.
Taylor and numerical errors are $O(K^{-s})$ away from narrow cell-boundary
strips; clipping bounds the error on these strips, whose measure is
$O(K^{-2s})$. Combining the two regions proves
\eqref{eq:ffn-resolution-error}. The final proof in
Section~\ref{app:ffn-selection} verifies the width and coefficient bounds
and that the same parameter choice works for every larger iteration count.

\subsection{Shared Taylor-table evaluator}\label{app:base-engine}
This subsection constructs a shared residual block that reads encoded
cell records, evaluates their local Taylor polynomials, and uses a cell
gate to retain the contribution from the cell containing the input.
We establish its approximation error, state dimension, hidden width,
and number of iterations per record.

Fix $k=\lceil s\rceil-1$ and $\alpha=s-k\in(0,1]$. Index Taylor terms
by ordered coordinate lists
\[
 \mathcal W_k=\coprod_{\ell=0}^k\{1,\ldots,d\}^{\ell},\qquad
 M_{\rm Tay}=|\mathcal W_k|=\sum_{\ell=0}^k d^\ell,\qquad
 V=\sum_{w\in\mathcal W_k}|w|=\sum_{\ell=0}^k\ell d^\ell.
\]
The empty list indexes the constant term.
Put $B_{\rm eval}=32(M_{\rm Tay}+1)+1$. The evaluator dimensions are
\begin{equation}\label{eq:base-resources}
 q_0=9d+7M_{\rm Tay}+5V+29,\qquad
 r_0=(7d+7M_{\rm Tay}+5V+14)(8d+62).
\end{equation}
For $K=2^m$, $m\ge1$, use the integer precisions
\[
 b=J=(k+1)m,\qquad G=(2(k+1)+1)m+2.
\]
These choices give $2^{-b},4^{-J}=O_s(K^{-s})$. Define
\begin{equation}\label{eq:base-runtime}
 \tau_m=d(5+2m)+M_{\rm Tay}(5+2b)+V(3+2J)+2G+11.
\end{equation}

\begin{lemma}[Taylor-table engine]\label{lem:base}
Assume $M_{\rm par}\ge\max\{3B_{\rm eval},H\}$. There is one residual
FFN block with dimensions \eqref{eq:base-resources} and coefficients bounded
by $3B_{\rm eval}$ that has the following property. For every
$f_0\in\Holder$ and every ordered list $\mathcal J$ of
$1\le N_*\le K^d$ distinct fine-grid cells, initialize it with
$x\in[0,1]^d$ and the scalar code of their cell locations and quantized
Taylor coefficients defined in \eqref{eq:code}.
After $N_*\tau_m$ iterations it is in an absorbing terminal state.
Every evaluator data coordinate, including its timer, is bounded in
magnitude by two; the returned normalized output $u$ satisfies $|u|\le1$.
Its returned value $g_{\mathcal J}=Hu$ obeys
\[
 |g_{\mathcal J}(x)-f_0(x)|\le C_{d,s}HK^{-s}
\]
inside any listed cell at distance at least $2^{-G}$ from every face.
For a list of all cells, $\|g_{\mathcal J}-f_0\|_2^2\le C_{d,s}H^2K^{-2s}$.
The same block works for all targets and lists at fixed $d,s,K$;
the target and list enter only through the initial code. Its affine
input function and output function satisfy the stated parameter bound.
\end{lemma}

\subsubsection{Ordered Taylor records and scalar encoding}
For each cell we store its location and the coefficients of a local
polynomial. We first turn these numbers into binary digits, assemble one
record per cell, and then pack the records into a single scalar. The shared
block reads this scalar one bit at a time to recover the information needed
for polynomial evaluation. The encoding is chosen when constructing the
approximating network; during evaluation the network receives only the
input $x$ and the encoded records.

\noindent\textbf{Local Taylor approximation.}
Partition $[0,1]^d$ into cells of side $h=K^{-1}$. If $a_j$ is a cell's
lower corner, put $x_j=a_j+(h/2)\mathbf1$. For
$w=(i_1,\ldots,i_\ell)\in\mathcal W_k$, write
$\partial_w f=\partial_{i_1}\cdots\partial_{i_\ell}f$ and
$(x-x_j)^w=\prod_{t=1}^{\ell}(x_{i_t}-x_{j,i_t})$, with the usual
empty-product convention. The normalized polynomial is
\begin{equation}\label{eq:taylor}
 P_j(x)=\sum_{w\in\mathcal W_k}c_{j,w}(x-x_j)^w,\qquad
 c_{j,w}=\frac{\partial_w f_0(x_j)}{H|w|!}.
\end{equation}
This is the ordinary multivariate Taylor polynomial: an order-$\ell$
Fr\'echet derivative expands into all $d^\ell$ ordered coordinate choices,
each with factor $1/\ell!$. Merging equal monomials gives the usual
multi-index coefficients, but the evaluator does not perform that merge.
All $c_{j,w}$ belong to $[-1,1]$.

Taylor expansion along the segment from $x_j$ to an interior point $x$
in the cell gives
\begin{equation}\label{eq:taylor-error}
 |HP_j(x)-f_0(x)|\le d^kHh^{k+\alpha}=d^kHh^s.
\end{equation}
For $k=0$, this is the H\"older condition. For $k\ge1$, subtract the
order-$k$ derivative at $x_j$ in the integral remainder. Expanding the
directional derivative into its $d^k$ coordinate terms bounds its
variation by $d^kH\|x-x_j\|_\infty^{k+\alpha}$. The integral weight
has mass at most one. This argument includes $\alpha=1$ and uses no
regularity outside the cube.

\noindent\textbf{Quantizing the coefficients.}
The cell location lies on the grid and can be stored exactly. Its Taylor
coefficients need finite binary representations. For any coefficient
$c_{j,w}\in[-1,1]$, define
\[
 n_b(c_{j,w})=\min\{2^b-1,\lfloor2^{b-1}(c_{j,w}+1)\rfloor\},\qquad
 \widetilde c_{j,w}=2^{1-b}n_b(c_{j,w})-1.
\]
The integer $n_b(c_{j,w})$ is between zero and $2^b-1$, so it fits in
$b$ binary digits. Decoding this integer and applying the second formula
recovers $\widetilde c_{j,w}$, with
$|\widetilde c_{j,w}-c_{j,w}|\le2^{1-b}$ and
$\widetilde c_{j,w}\in[-1,1]$.

\noindent\textbf{The contents of one cell record.}
A record contains the following fields, in the order shown.
\begin{center}\small
\begin{tabular}{llr}
\toprule
Field & Stored values & Number of bits\\
\midrule
Location & $Ka_{j,1},\ldots,Ka_{j,d}$ & $dm$\\
Coefficients & $n_b(c_{j,w})$, in the fixed order of $\mathcal W_k$ & $M_{\rm Tay}b$\\
Continuation & One if another cell follows; zero for the last cell & $1$\\
\bottomrule
\end{tabular}
\end{center}
Each location integer lies in $\{0,\ldots,2^m-1\}$ and uses exactly $m$
bits; each coefficient integer uses exactly $b$ bits. A fixed-length
binary string for one integer is called a word, including zero padding
to fill its prescribed length. Within each word, the bits are stored
from least to most significant. Thus an $\ell$-bit integer
$n=\sum_{t=1}^{\ell}\eta_t2^{t-1}$ is stored in the order
$\eta_1,\ldots,\eta_\ell$, where $\eta_t\in\{0,1\}$.
For example, the integer one in a two-bit word is stored as $(1,0)$.

The lower corner determines the cell center $x_j$ and its faces because
$h$ is fixed. The polynomial factors are determined by $\mathcal W_k$,
so they do not need to be stored. Each record therefore has exactly
$dm+M_{\rm Tay}b+1$ bits. The common field lengths tell the decoder where
each location or coefficient word ends; the final bit tells it whether
to begin another cell record or halt.

\noindent\textbf{Packing the records into one scalar.}
Take the $N_*$ cells in the order specified by $\mathcal J$ and join
their records end to end. Here $N_*$ counts records, whereas
$L=N_*(dm+M_{\rm Tay}b+1)$ counts all their binary digits.
Write $\beta_i\in\{0,1\}$ for the $i$th digit of this complete sequence.
The digits include location bits, coefficient bits, and continuation bits.
Store them in the scalar
\begin{equation}\label{eq:code}
 c=\sum_{i=1}^{L}2\beta_i4^{-i},\qquad
 L=N_*(dm+M_{\rm Tay}b+1),\qquad \beta_i\in\{0,1\}.
\end{equation}
Thus the $i$th base-four digit of $c$ is $2\beta_i$, either zero or two.
The scalar $c$ stores the entire record list, rather than an individual
Taylor coefficient. Every remaining suffix, when shifted to the start
of the code, belongs to $[0,2/3]$.

\begin{samepage}
\noindent\textbf{Reading bits and recovering the stored values.}
If the first bit is zero, the code is at most
$\sum_{i\ge2}2\,4^{-i}=1/6$. If it is one, the code is at least $1/2$.
This gap lets the ReLU function
\[
 D(c)=3\ReLU(c-1/6)-3\ReLU(c-1/2)
\]
return the first bit exactly. After reading it, the update
$c^+=4c-2D(c)$ removes its base-four digit and shifts the remaining
digits forward. The same decoder can then read the next bit.
\end{samepage}

At the start of each word, reset an accumulator $a$ to zero. Read its
prescribed number of bits using the simultaneous updates
\[
 c^+=4c-2D(c),\qquad a^+=(a+D(c))/2,
\]
where both right-hand sides use the old code. For the word
$\eta_1,\ldots,\eta_\ell$ defined above, after $\ell$ steps the
accumulator is $a=\sum_{t=1}^{\ell}\eta_t2^{t-1}/2^\ell=n/2^\ell$.
Consequently, a location word gives
$a=Ka_{j,i}/2^m=a_{j,i}$ directly. A coefficient word gives
$a=n_b(c_{j,w})/2^b$, from which $2a-1=\widetilde c_{j,w}$.
The evaluator saves the decoded location and uses the coefficients in
their prescribed term order. After processing these words, it reads and
removes the continuation bit: one starts the next record and zero ends
the list.

\noindent\textbf{A one-cell example.}
Let $d=1$, $k=0$, and $m=b=2$. A cell with lower corner $a_j=1/4$
has location integer $Ka_j=1$, stored as $(1,0)$. Its only Taylor term
is constant. If its coefficient is $1/2$, then $n_b(1/2)=3$, stored as
$(1,1)$, and quantization recovers $1/2$ exactly. For a list containing
only this cell, the continuation bit is zero. The complete five-bit record is
\[
 (\underbrace{1,0}_{\text{location}},\quad
  \underbrace{1,1}_{\text{coefficient}},\quad
  \underbrace{0}_{\text{end}}),\qquad
 c=0.20220_{\text{(base 4)}}=\frac{69}{128}.
\]
The decoder returns these bits in order. The first word gives $a=1/4$;
after resetting the accumulator, the second gives $a=3/4$ and hence
$2a-1=1/2$. The final zero instructs the evaluator to halt.

The code is one counted input-function bias for the standalone evaluator
and one counted lookup coefficient per group for the spatial construction.

\subsubsection{Shared square and multiplication recurrences}
\begin{lemma}[Bounded square registers]\label{lem:square}
For $x\in[0,1]$, initialize $(v_0,w_0,S_0)=(x,1,x)$ and iterate
\begin{equation}\label{eq:square-loop}
 \begin{split}
 v_{j+1}&=v_j/2-\ReLU(v_j-w_j/2),\qquad w_{j+1}=w_j/4,\\
 S_{j+1}&=S_j-v_j/2+\ReLU(v_j-w_j/2).
 \end{split}
\end{equation}
For every $j\ge0$, $0<w_j=4^{-j}$, $0\le v_j\le w_j$, $0\le S_j\le1$,
and
\[
 w_j(S_j-x^2)=w_jv_j-v_j^2,\qquad
 0\le S_j-x^2\le4^{-j}/4.
\]
The same update coefficients apply at every iteration.
\end{lemma}
\begin{proof}
The displayed invariant holds at initialization. Substitution into
\eqref{eq:square-loop}, separately for $v_j\le w_j/2$ and
$v_j\ge w_j/2$, preserves it and gives $0\le v_{j+1}\le w_{j+1}$.
Consequently $S_j-x^2=v_j(w_j-v_j)/w_j\in[0,w_j/4]$.
Also $S_{j+1}=S_j-v_{j+1}\le S_j\le x\le1$, while the invariant
implies $S_j\ge x^2\ge0$. This proves all claims.
\end{proof}
The recurrence is a rescaled shared-coefficient form of the dyadic square
construction in \citet{yarotsky2017}.
For $u,v\in[-1,1]$, run two square channels for $J$ updates and put
\begin{equation}\label{eq:mult}
 \mathsf M_J(u,v)=S_J(|u+v|/2)-S_J(|u-v|/2).
\end{equation}
They share the register $w$. Absolute values are sums of two ReLUs, and
$uv=((u+v)/2)^2-((u-v)/2)^2$ gives
\begin{equation}\label{eq:mult-error}
 |\mathsf M_J(u,v)-uv|\le4^{-J}/2,\qquad |\mathsf M_J(u,v)|\le1.
\end{equation}
For $w=(i_1,\ldots,i_\ell)$, start at $\widetilde c_{j,w}$ and
successively apply $\mathsf M_J$ with factors $x_{i_v}-x_{j,i_v}$,
$v=1,\ldots,\ell$. Denote the final value by $\widetilde t_{j,w}(x)$;
for the empty list it equals $\widetilde c_{j,w}$.
Every intermediate value stays in $[-1,1]$,
since $x_i-x_{j,i}\in[-1,1]$ for all inputs in the cube. Hence
\[
 |\widetilde t_{j,w}(x)-c_{j,w}(x-x_j)^w|
 \le2^{1-b}+|w|4^{-J}/2.
\]
Summing over the ordered term list gives
\begin{equation}\label{eq:poly-error}
 \left|\sum_w\widetilde t_{j,w}(x)-P_j(x)\right|
 \le M_{\rm Tay}2^{1-b}+V4^{-J}/2\le C_{d,s}K^{-s}.
\end{equation}
Accumulate $\widetilde t_{j,w}/M_{\rm Tay}$ in a register $p$ and set
$v_j=\Pi_1(M_{\rm Tay}p)$. Every partial normalized sum has magnitude at
most one. Clipping cannot increase the error relative to $f_0/H$, so
\begin{equation}\label{eq:local-error}
 |Hv_j(x)-f_0(x)|\le C_{d,s}HK^{-s}
\end{equation}
for interior points of cell $j$. This calculation separates the analytic
Taylor remainder from the errors of the actual decoded coefficients and
shared multiplication recurrence.

\subsubsection{The evaluator as a finite-state program}
The evaluator implements a fixed finite-state program. Its purpose is simple:
for each encoded cell record, it decodes the cell location and Taylor
coefficients, evaluates the corresponding local polynomial, keeps the
contribution only if the input lies in that cell, adds the contribution to
the output accumulator, and then either advances to the next record or
halts. Thus one record is processed through the following five stages:
\begin{enumerate}[label=(\roman*),leftmargin=2.2em]
\item decode the cell location;
\item decode and evaluate the Taylor terms;
\item construct the cell-membership gate;
\item add the gated value to the output accumulator;
\item read the continuation bit and either repeat or halt.
\end{enumerate}
The point of the phase construction below is that the number of state
coordinates needed to implement these stages is fixed: increasing the
resolution changes only the number of repeated block applications.

\noindent\textbf{State variables and phases.}
The state has two parts. Data coordinates store numerical values, whereas
phase coordinates specify which elementary instruction is currently active.
Write $e_1,\ldots,e_P$ for the allocated phase coordinates, including an
absorbing halt phase. On every reachable state exactly one phase coordinate
equals one and all the others equal zero. The active phase determines both
the data update and the next phase; all assignments read the old state and
are updated simultaneously.

The vectors $x,a_j\in\mathbb R^d$ and the remaining scalar registers give
$D_{\rm data}=2d+15$ data coordinates in total.
\begin{center}\small
\setstretch{1}\renewcommand{\arraystretch}{1.15}
\begin{tabular}{lp{0.72\textwidth}}
\toprule
Register & Role\\
\midrule
$x$ & Input vector.\\
$a_j$ & Lower corner of the current cell.\\
$c$ & Unread remainder of the encoded record list.\\
$a$ & Accumulator for the binary word being decoded.\\
$t$ & Running product for the current Taylor term.\\
$p$ & Taylor sum scaled by $1/M_{\rm Tay}$.\\
$d_{\rm face}$ & Minimum signed distance to the current cell's faces.\\
$\psi$ & Cell gate controlling the local output contribution.\\
$v$ & Clipped polynomial value for the current cell.\\
$u$ & Accumulated output across processed cells.\\
$\beta$ & Continuation bit: one to continue, zero to halt.\\
$v_+,v_-$ & Intermediate states of the two square recurrences.\\
$S_+,S_-$ & Square approximations whose difference gives a product.\\
$w$ & Scale shared by the two square recurrences.\\
$\tau_{\rm eval}$ & Timer for a counted repetition.\\
\bottomrule
\end{tabular}
\end{center}
Initially, $x$ contains the input, $c$ contains the encoded record list,
the first phase coordinate equals one, and every other data and phase
coordinate is zero; in particular $u=0$.

\noindent\textbf{Primitive counted routines.}
All loops used by the evaluator have deterministic lengths known from
$d,s$ and the resolution parameters. A single three-phase counter implements
a repetition of any prescribed length $\ell$ without allocating $\ell$
different phases. For $\ell\ge1$, an initialization phase sets
$\tau_{\rm eval}=2^{-(\ell-1)}$; an operation phase performs one update;
and a check phase doubles the timer and returns to the operation phase
unless the old timer equals one. The branch signal
$\ReLU(2\tau_{\rm eval}-1)$ is zero at all earlier checks and one at the
last check. Hence the routine performs exactly $\ell$ operation/check pairs
and takes $1+2\ell$ iterations. A zero-length repetition exits immediately
after initialization.
A word-reading routine, denoted conceptually by $\mathsf{READ}(\ell)$,
adds one phase that first resets $a$ and then uses the counted decoder for
$\ell$ bits. It therefore uses four phases and $2+2\ell$ iterations.
Saving the decoded location or using the decoded coefficient is handled
by a separate one-step assignment.

A multiplication routine, denoted $\mathsf{MULT}(J)$, initializes the two
square channels, performs $J$ counted square updates, and finally assigns
$t=S_+-S_-$. It uses five phases and $3+2J$ iterations. These are the only
repeated arithmetic routines needed below.

\noindent\textbf{Stage (i): decoding the cell location.}
Reset the polynomial accumulator to $p=0$. For each coordinate
$i=1,\ldots,d$, run $\mathsf{READ}(m)$ and save the decoded value in
$a_{j,i}$. The coordinate list ends with one explicit identity step.
Thus the location stage uses $5d+1$ phase coordinates and
$d(3+2m)+1$ iterations after the initial reset of $p$.

\noindent\textbf{Stage (ii): evaluating the Taylor polynomial.}
Process the Taylor terms in the fixed order of $\mathcal W_k$. For a term
indexed by $w$, run $\mathsf{READ}(b)$, set $t=2a-1$, and then apply
$\mathsf{MULT}(J)$ successively to the $|w|$ factors
$x_i-a_{j,i}-h/2$ specified by $w$. For the empty list no multiplication
is needed. Each factor list ends with one identity step; then add
$t/M_{\rm Tay}$ to $p$. After the last Taylor term, the term list itself
ends with one additional identity step. The running term and every partial
polynomial accumulator remain in $[-1,1]$.

Consequently the Taylor stage uses
$7M_{\rm Tay}+5V+1$ phase coordinates and
\[
 M_{\rm Tay}(5+2b)+V(3+2J)+1
\]
iterations. The factor-list identity contributes the extra fixed step in
the $5+2b$ cost for each Taylor term, whereas the final $+1$ is the
identity step terminating the whole term list.
\noindent\textbf{Stage (iii): constructing the cell gate.}
First set $v=\Pi_1(M_{\rm Tay}p)$. In a separate phase, initialize $d_{\rm face}=1$.
Scan the $d$ lower faces and then the $d$ upper faces using
\[
 d_{\rm face}^+=d_{\rm face}-\ReLU(d_{\rm face}-\xi),\qquad
 \xi=x_i-a_{j,i}\ \text{or}\ a_{j,i}+h-x_i.
\]
Each update replaces the current value by the minimum of that value and
the signed distance to one face. The lower and upper faces form one ordered list of $2d$ faces,
which ends with one identity step. The final $d_{\rm face}$ is positive inside the cell and
nonpositive outside its interior.

Initialize $\psi=\ReLU(d_{\rm face})$ and perform $G$ counted updates
\[
 \psi^+=2\psi-\ReLU(2\psi-1).
\]
They yield
\[
 \psi_j(x)=\min\{1,2^G\ReLU(d_{{\rm face},j}(x))\}.
\]
Hence $\psi_j(x)=0$ outside the cell interior and $\psi_j(x)=1$ whenever
$x$ is at least $2^{-G}$ from every face. This stage uses $2d+7$ phase
coordinates and $2d+5+2G$ iterations.

\noindent\textbf{Stage (iv): accumulating the selected contribution.}
Update
\[
 u^+=u+\ReLU(v+\psi)-\ReLU(v-\psi)-\psi
     =u+\Pi_\psi(v).
\]
The increment is zero when $\psi=0$ and equals $v$ when $\psi=1$.
Distinct cells have disjoint localization supports, so every partial output
sum remains in $[-1,1]$. This stage uses one phase and one iteration.

\noindent\textbf{Stage (v): advancing or halting.}
Read the continuation bit $\beta=D(c)$ and remove it simultaneously through
$c^+=4c-2D(c)$. A separate marker-check phase reads the saved bit:
$\beta=1$ returns to the start of the record routine, retaining $u$ and
the unread suffix code, whereas $\beta=0$ enters the absorbing halt phase.
For valid record lists the marker is exactly zero or one. The halt phase
has zero residual, so all later block applications preserve the terminal
state. Together, marker read and marker check use two phases and two
iterations, and the absorbing halt state uses one additional phase.

The bounds above, Lemma~\ref{lem:square}, and the decoder bounds keep every
data register within $[-2,2]$ throughout the program. The next subsection
shows that this entire finite-state program is realized by one shared
residual ReLU block.
\subsubsection{Compiling the finite-state program into one residual ReLU block}
The preceding evaluator is a finite-state program whose elementary
instructions have a uniform algebraic form. The following lemma isolates
the compilation step used to turn such a program into one shared residual
ReLU block.

\begin{lemma}[Finite-state program compilation]\label{lem:finite-state-compiler}
Consider a finite-state program with $D$ data coordinates and $P$ one-hot
phase coordinates. Suppose that, on every reachable state, the active
phase $a$ updates each data coordinate according to
\[
 z_i^+=z_i+A_{ia}(z)
      +c_{ia,1}\ReLU\{\ell_{ia,1}(z)\}
      +c_{ia,2}\ReLU\{\ell_{ia,2}(z)\},
\]
where $A_{ia}$ and $\ell_{ia,\nu}$ are affine functions and unused hinges
are set to zero. Suppose also that every phase transition is a default
move, possibly redirected by one branch test whose ReLU output is exactly
zero or one on the reachable branch states.

Assume that, for some $B\ge1$, every ungated affine preactivation has
absolute value at most $2B$ on the reachable data box and every affine
and hinge output coefficient has absolute value at most $B$. Then the
program is realized exactly on its reachable states by a single residual
ReLU block
\[
 \Phi(y)=y+W_2\ReLU(W_1y+b_1)+b_2,\qquad b_2=0,
\]
on $D+P$ state coordinates, with hidden width at most
\[
 P(4D+2),
\]
and every block coefficient bounded in magnitude by $3B$. If one phase
has zero residual and transitions to itself, that phase is an absorbing
terminal state of the same block.
\end{lemma}
\begin{proof}
Let $e_a$ denote the one-hot indicator of phase $a$ and set
$\Lambda=2B$. For a feature belonging to phase $a$, add
$\Lambda(e_a-1)$ to its preactivation. If phase $a$ is active, this
addition is zero. If it is inactive, the preactivation is at most
$2B-\Lambda=0$, so the feature vanishes.

For one data coordinate in phase $a$, represent the signed affine residual
$A=A_{ia}(z)$ by
\[
 \ReLU\{A+\Lambda(e_a-1)\}
 -\ReLU\{-A+\Lambda(e_a-1)\}.
\]
This equals $A$ in the active phase and zero otherwise. Gate the at most
two hinge features in the same way. Hence four hidden units per data
coordinate suffice for every phase, with zero-padding when fewer are
needed.

It remains to update the phase coordinates. For phase $a$, one feature
$\ReLU(e_a)$ implements the default move and can also supply a
phase-specific constant residual. A second gated feature implements the binary redirect when the
branch is taken: its output cancels the default successor and
activates the alternative successor. Because the branch signal is
exactly zero or one on reachable branch states, these two units preserve
the one-hot phase encoding. Thus each phase uses at most $4D+2$ hidden
units, giving total width $P(4D+2)$.

Concatenating all phase-specific features yields one
affine--ReLU--affine residual block. The gating coefficient has magnitude
$\Lambda=2B$; combining it with an original affine coefficient or bias of
magnitude at most $B$ gives the bound $3B$. Taking zero residual and a
self-transition in the designated halt phase makes that state absorbing.
\end{proof}

\noindent\textbf{Application to the evaluator.}
Every instruction in the program of D.1.3 satisfies the lemma's
two-hinge condition. The decoder
\[
 D(c)=3\ReLU(c-1/6)-3\ReLU(c-1/2)
\]
uses two hinges; absolute value and clipping use two; a face-minimum
update uses one; the gate recurrence uses one; and the gated output
increment uses two. For example, in a face-comparison phase $a$, the single gated hinge
\[
 d_{\rm face}^+=d_{\rm face}
 -\ReLU\{d_{\rm face}-\xi+2B_{\rm eval}(e_a-1)\}
\]
takes the minimum with the current face distance when $e_a=1$
and leaves $d_{\rm face}$ unchanged in every other phase.
The branch tests are zero, $2\tau_{\rm eval}-1$, or $\beta$.
On every reachable check state their
ReLU outputs are exactly zero or one, so the phase transitions meet the
binary-redirect hypothesis.

By the register bounds from D.1.3, Lemma~\ref{lem:square}, and the decoder
bounds, every evaluator data coordinate stays in $[-2,2]$. With
$B=B_{\rm eval}$, all ungated affine preactivations are bounded in absolute
value by $2B_{\rm eval}$, and all affine and hinge output coefficients are
bounded by $B_{\rm eval}$. Lemma~\ref{lem:finite-state-compiler} therefore
gives one shared residual ReLU block with coefficient bound
$3B_{\rm eval}$.

\begin{samepage}
\noindent\textbf{Resource bookkeeping.}
The phase and iteration counts established in the previous subsection are
summarized below. The identity steps terminating the coordinate, factor,
term, and face lists are included explicitly.
\begin{center}\footnotesize
\begin{tabular}{lrr}
\toprule
Program group & Phase coordinates & Iterations per record\\
\midrule
Reset polynomial accumulator & $1$ & $1$\\
Coordinate words, saves, and final identity & $5d+1$ & $d(3+2m)+1$\\
Terms, additions, and list identities & $7M_{\rm Tay}+5V+1$
 & $M_{\rm Tay}(5+2b)+V(3+2J)+1$\\
Clip, face comparisons, and gate & $2d+7$ & $2d+5+2G$\\
Output accumulation & $1$ & $1$\\
Marker read and check & $2$ & $2$\\
Absorbing halt & $1$ & $0$\\
\midrule
Total & $7d+7M_{\rm Tay}+5V+14$ & $\tau_m$\\
\bottomrule
\end{tabular}
\end{center}
\end{samepage}
Therefore
\[
 P=7d+7M_{\rm Tay}+5V+14.
\]
Adding the $D_{\rm data}=2d+15$ data coordinates gives
\[
 q_0=D_{\rm data}+P
    =9d+7M_{\rm Tay}+5V+29,
\]
which is \eqref{eq:base-resources}. The hidden width is
\[
 r_0=P(4D_{\rm data}+2)
    =(7d+7M_{\rm Tay}+5V+14)(8d+62),
\]
again exactly as in \eqref{eq:base-resources}. Summing the iteration column
gives
\[
 \tau_m=d(5+2m)+M_{\rm Tay}(5+2b)+V(3+2J)+2G+11,
\]
which is \eqref{eq:base-runtime}. Thus the architecture size is fixed once
$d$ and $s$ are fixed, while increasing the resolution affects the
precision constants and the number of repeated applications, not the
number of distinct block parameters.

An all-zero evaluator phase vector leaves its initially bounded data
unchanged until the selector activates the first evaluator phase.
\subsubsection{Correctness, error, and precision}
\begin{proof}[Proof of Lemma~\ref{lem:base}]
Induct first over the bits of each word, then over its factor and term
lists, and finally over the cell records. The decoder removes exactly the
prescribed word and preserves its suffix code. The square invariant
bounds every arithmetic operation. The separate counted-loop initialization
and checks give their specified lengths; the marker returns to the same
record program or enters halt. The bounds established above justify all
inactive gates at each step. Thus the compiled block executes the complete program and returns
\[
 u(x)=\sum_{j\in\mathcal J}\Pi_{\psi_j(x)}(v_j(x))
\]
at time $N_*\tau_m$. At most one summand is nonzero. In a listed cell
whose face distances are at least $2^{-G}$, its gate is one and
\eqref{eq:local-error} applies. The halt state is fixed by the block,
so the entire terminal state persists under all additional iterations.

For the full grid, write $g_{\rm all}=g_{\mathcal J}$ when $\mathcal J$
lists every cell. Let $S_G$ be the union, inside the cube, of the open
strips of half-width $\delta=2^{-G}$ around all grid hyperplanes, including
the boundary. A union bound gives
\[
 \mu(S_G)\le2d(K+1)\delta\le dK^{-2s}.
\]
Indeed $G=(2(k+1)+1)m+2$ and $s\le k+1$. Off $S_G$ the containing cell
has gate one, giving pointwise error at most $C_{d,s}HK^{-s}$.
On $S_G$, both the returned value and target are bounded by $H$. Therefore
\begin{equation}\label{eq:grid-error}
 \|g_{\rm all}-f_0\|_2^2\le C_{d,s}H^2K^{-2s}.
\end{equation}
The input function uses input copies, the code, and phase constants of magnitude
at most one; the output-function coefficient is $H$. This proves the lemma.
\end{proof}
The code in \eqref{eq:code} is dyadic with denominator $2^{2L}$ and uses
at most $2L=O_{d,s}(N_*m)$ binary fractional bits. The mesh and timer
constants require $O_s(m)$ bits. Coefficient magnitudes remain uniformly bounded,
while the required precision grows with the resolution.

\subsection{Spatial selection and proof of Theorem~\ref{thm:ffn-construction}}\label{app:ffn-selection}
\label{app:joint}
The evaluator above scans every record in its input list. We now trade
additional width for fewer scanned records. The fine cells are partitioned
into coarse groups, a bounded-weight front end identifies the group
containing $x$ and loads only that group's encoded Taylor table, and a
one-block handoff starts the evaluator. The group codes are target-specific
lookup weights and are counted among the fitted coefficients. The final
width choice makes the number of groups comparable to
$\min\{r,K^d\}$, producing the runtime bound in
\eqref{eq:ffn-resolution-time}.

\subsubsection{A bounded-weight spatial selector}
The purpose of the front end is to use width to avoid scanning all
$K^d$ fine-cell records. Choose
\[
 u_{\rm grp}\in\{0,\ldots,m\},\qquad
 A=2^{u_{\rm grp}},\qquad
 N_{\rm grp}=A^d,
\]
and partition the fine grid into $N_{\rm grp}$ coarse cubes. Each coarse
cube contains $2^{d(m-u_{\rm grp})}$ fine cells. For
$a=(a_1,\ldots,a_d)\in\{0,\ldots,A-1\}^d$, write
\[
 j(a)=\sum_{i=1}^d A^{i-1}a_i
\]
for its group index and let $\gamma_a\in[0,2/3]$ be the scalar code of
the fine-cell records in that group.

The construction below has one objective: given $x$, load
$\gamma_a$ for the unique coarse cube containing $x$, and then emit a
single start pulse for the evaluator. The following lemma records this
property and the required resources.

\begin{lemma}[Bounded-weight spatial selector]\label{lem:spatial-selector}
Set
\[
 L_m=G+dm,\qquad C_*=2^{L_m},\qquad \ell=du_{\rm grp}.
\]
There is one residual ReLU front end with
\[
 q_{\rm front}=d+10,\qquad
 r_{\rm front}=10dA+10d+15N_{\rm grp}+60,
 \qquad M_{\rm front}=22,
\]
such that the following holds. If $x$ is at distance at least $C_*^{-1}$
from every coarse-grid face, then after exactly
\[
 L_m+du_{\rm grp}+2
\]
iterations the front end is in an absorbing halt state with
\[
 c_{\rm front}=\gamma_a,\qquad \eta_{\rm start}=1,
\]
where $a$ is the coarse-cell index of $x$. Before that lookup iteration,
both payload coordinates $c_{\rm front}$ and $\eta_{\rm start}$ and their
residual increments are zero. Every coefficient is bounded in magnitude
by $22$.
\end{lemma}
\begin{proof}
The proof has four steps: generate a common scale, decode the coarse-cell
coordinates, convert the normalized address to an integer address, and
perform a one-step table lookup.

\noindent\textbf{Step 1: generate a common scale with bounded coefficients.}
The front end uses data registers
\[
 U_1,\ldots,U_d,\ C,\ I,\ \tau_{\rm front},\
 c_{\rm front},\ \eta_{\rm start}
\]
and five phase coordinates for warm-up, address, scaling, lookup, and
halt. Write $e_{\rm w}$ and $e_{\rm a}$ for the warm-up and address
phase coordinates, respectively. The phase encoding is scaled: on every reachable state exactly one
phase coordinate equals the current scale $C$ and all the others are zero.

Initialize
\[
 U_i=x_i,\qquad C=1,\qquad
 \tau_{\rm front}=2^{-(L_m-1)},
\]
with the warm-up phase equal to one and all other front-end coordinates
zero. During warm-up use
\begin{align*}
 U_i^+&=2U_i,&
 C^+&=2C,&
 \tau_{\rm front}^+&=4\tau_{\rm front},\\
 e_{\rm w}^+&=2e_{\rm w}-2F,&
 e_{\rm a}^+&=e_{\rm a}+2F,&
 F&=\ReLU(2\tau_{\rm front}-C).
\end{align*}
Before warm-up update $j$, $C=2^j$ and
$\tau_{\rm front}/C=2^{j-(L_m-1)}$. Thus $F=0$ before the last update
and $F=C$ on the last one. The scaled one-hot invariant is therefore
preserved, and after exactly $L_m$ warm-up iterations,
\[
 U_i=C_*x_i,\qquad C=C_*,\qquad
 \tau_{\rm front}=2C_*,
\]
with the address phase active at scale $C_*$. Repeated doubling creates
the large spatial scale $C_*$ without any large trainable coefficient.

\noindent\textbf{Step 2: decode the coarse-cell coordinates.}
For one coordinate define
\[
 H_i=
 \sum_{j=1}^{A}
 \big[
 \ReLU(U_i-jC/A+1)-\ReLU(U_i-jC/A)
 \big].
\]
Suppose $x_i$ lies in coarse interval $a_i$ and is at least $C_*^{-1}$
from both of its faces. Since $U_i=C_*x_i$ and $C=C_*$, every summand
with $j\le a_i$ equals one and every summand with $j>a_i$ equals zero;
the endpoint term $j=A$ is also zero. Hence
\[
 H_i=a_i.
\]
One address iteration therefore writes
\begin{equation}\label{eq:address}
 I=\sum_{i=1}^d\frac{A^{i-1}}{N_{\rm grp}}H_i
   =\frac{j(a)}{N_{\rm grp}}.
\end{equation}
All coefficients $A^{i-1}/N_{\rm grp}$ have magnitude at most one.
\noindent\textbf{Step 3: convert to the integer address.}
Because $N_{\rm grp}=A^d=2^{du_{\rm grp}}=2^\ell$, the desired integer
address is
\[
 j(a)=2^\ell I.
\]
If $\ell\ge1$, the address phase initializes
$\tau_{\rm front}=2^{-(\ell-1)}C$ and activates the scaling phase.
Perform $\ell$ updates
\[
 I^+=2I,\qquad
 \tau_{\rm front}^+=2\tau_{\rm front},
\]
and exit when the old timer equals $C$, using the branch
$\ReLU(2\tau_{\rm front}-C)$. If $\ell=0$, skip this phase.
In either case, lookup starts with
\[
 I=j(a),\qquad \tau_{\rm front}=2C_*.
\]

\noindent\textbf{Step 4: lookup the encoded group.}
For an integer $j$ define the triangular hat
\[
 \chi_j(I)=
 \ReLU(I-j+1)-2\ReLU(I-j)+\ReLU(I-j-1).
\]
At integer arguments,
\[
 \chi_j(I)=\mathbf 1\{I=j\}.
\]
The lookup iteration therefore sets
\begin{equation}\label{eq:lookup}
 c_{\rm front}
 =\sum_a\gamma_a\chi_{j(a)}(I)
 =\gamma_a,
 \qquad
 \eta_{\rm start}=C/C_*=1.
\end{equation}
Each shift $j(a)$ is implemented as $(j(a)/C_*)C$, whose coefficient
has magnitude at most one. Lookup then enters the scaled halt phase, so
the front-end residual is zero at all later iterations.

It remains to compile these five phases into one residual ReLU block.
A front-end preactivation $A(z)$ is gated by adding $16(e-C)$, where the
active phase coordinate satisfies $e=C$. On every valid trajectory all
data coordinates are bounded by a fixed multiple of $C$, and every
inactive preactivation is at most $16C$ before gating. The relevant branch
outputs are always either zero or $C$. Hence the scaled one-hot invariant
is preserved and all inactive features vanish.

Allocate, for each phase, $2(d+5)$ signed-linear units, $2dA$ address
hinges, $3N_{\rm grp}$ lookup hinges, and two clock/branch units, with
inactive phase-specific output columns set to zero. Summing over the five
phases gives
\[
 q_{\rm front}=d+10,\qquad
 r_{\rm front}=10dA+10d+15N_{\rm grp}+60.
\]
All coefficients are bounded by $22$. The timeline is
\[
 L_m\ \text{warm-up steps}
 +1\ \text{address step}
 +du_{\rm grp}\ \text{scaling steps}
 +1\ \text{lookup step},
\]
which proves the claimed finishing time. Before lookup, both payload
coordinates and their residual increments are zero by construction.
\end{proof}
\subsubsection{One-block handoff and resource bounds}
Lemma~\ref{lem:spatial-selector} produces exactly the two quantities needed
to start the evaluator: the selected table code $\gamma_a$ and a one-time
start pulse. We now fuse the selector and evaluator without composing
their hidden layers.

Initialize the evaluator input registers to $x$ and all its remaining
data registers and phase coordinates to zero. It therefore remains
dormant until the selector loads the code and activates its first phase.
Let $\Phi_{\rm ev}$ and $\Phi_{\rm front}$ denote the two compiled residual
maps. Let $L$ copy the front-end code increment into the evaluator code
coordinate and the start-pulse increment into its first phase coordinate,
and be zero elsewhere. Use the single combined update
\begin{equation}\label{eq:ffn-fused-update}
 (z,y)^+=\big(\Phi_{\rm ev}(z)+L\{\Phi_{\rm front}(y)-y\},
                         \Phi_{\rm front}(y)\big).
\end{equation}
This is still one residual ReLU block: concatenate the two hidden layers
and copy the indicated front-end output rows to the corresponding evaluator
outputs. No hidden-layer composition and no multiplication of computed
states is used.

The handoff occurs only once. Its three regimes are
\[
\begin{array}{c|c|c}
\text{front-end stage} & \Phi_{\rm front}(y)-y & \text{evaluator state}\\
\hline
\text{before lookup} & 0\ \text{on payload rows} & \text{dormant}\\
\text{lookup} & (\gamma_a,1)\ \text{on payload rows} & \text{loaded and started}\\
\text{after lookup} & 0 & \text{runs its own recurrence}.
\end{array}
\]
Thus the evaluator cannot start early and cannot be loaded twice.
Selection takes $L_m+du_{\rm grp}+2$ steps by
Lemma~\ref{lem:spatial-selector}. The selected group contains
\[
 2^{d(m-u_{\rm grp})}=\frac{K^d}{N_{\rm grp}}
\]
fine-cell records, and D.1 requires $\tau_m$ iterations per record.
Therefore the exact finishing time is
\begin{equation}\label{eq:exact-joint-runtime}
 S_{m,u_{\rm grp}}
 =L_m+du_{\rm grp}+2
  +2^{d(m-u_{\rm grp})}\tau_m.
\end{equation}
The combined dimensions and a sufficient coefficient bound are
\begin{equation}\label{eq:width-cost}
 \begin{split}
 q_*&=q_0+d+10,\\
 r(u_{\rm grp})&=r_0+10d\,2^{u_{\rm grp}}+10d
                  +15\,2^{du_{\rm grp}}+60,\\
 B_*&=2(3B_{\rm eval}+22).
 \end{split}
\end{equation}
All parameter slots, including zeros, are included in the dense fitted
class. Extra coordinates and hidden units can be zero-padded without
changing the output or increasing the coefficient bound.

\subsubsection{Finite-resolution error and width selection}
\begin{proof}[Proof of Theorem~\ref{thm:ffn-construction}]
More width permits more coarse groups and hence fewer records to scan.
The total runtime is the selection time plus the number of records in the
selected group multiplied by the evaluation time per record.

Fix the target and $K=2^m$. Put
\[
 b_0=r_0+10d+60,\qquad b_1=10d+15,\qquad
 R(u_{\rm grp})=b_0+b_1\,2^{du_{\rm grp}}.
\]
For $d\ge1$, $2^{u_{\mathrm{grp}}}\le2^{du_{\mathrm{grp}}}$ and $r(u_{\mathrm{grp}})\le R(u_{\mathrm{grp}})$.
The following fixed thresholds suffice:
\begin{equation}\label{eq:resources}
 q\ge q_*=q_0+d+10,\qquad r\ge r_*=b_0+b_1,\qquad
 M_{\rm par}\ge M_*:=\max\{1,H,B_*\}.
\end{equation}
Choose the largest $u_{\mathrm{grp}}\in\{0,\ldots,m\}$ with $R(u_{\mathrm{grp}})\le r$.
If $u_{\mathrm{grp}}=m$, each group has one cell. Otherwise
$r<R(u_{\mathrm{grp}}+1)\le(b_0+2^db_1)2^{du_{\mathrm{grp}}}$. In either case the group count is
comparable to $\min\{r,K^d\}$, with constants depending only on $d,s$.
Since $\tau_m\le C_{d,s}m$, \eqref{eq:exact-joint-runtime} gives
\begin{equation}\label{eq:time-upper}
 S_{m,u_{\mathrm{grp}}}\le C_{d,s}m\left(1+\frac{K^d}{r}\right)
 \le C_{d,s}\left\{\log(eK)+\frac{K^d}{r}\log(eK)\right\}.
\end{equation}
Let $S_G$ be the fine-grid strip set from the evaluator proof.
Every coarse face is a fine-grid hyperplane and $C_*^{-1}\le2^{-G}$.
Hence outside $S_G$ the input has the coarse-face margins needed for
exact selection, and the loaded group contains its fine cell. The fused
block then returns the correct local value at time $S_{m,u_{\mathrm{grp}}}$ and preserves
it at all later iterations. On $S_G$, the final clipping always bounds
the error by $2H$, even for a nonvalid interpolated code. Thus
\begin{equation}\label{eq:joint-grid-error}
 \|g_T-f_0\|_2^2\le C_{d,s}H^2K^{-2s},\qquad T\ge S_{m,u_{\mathrm{grp}}}.
\end{equation}
Take $q_0^{\mathrm F}=q_*$, $r_0^{\mathrm F}=r_*$, and
$M_0^{\mathrm F}=M_*$. The input function is affine, the final output function is $Hu$
followed by clipping, and \eqref{eq:ffn-fused-update} is one shared
residual block. All entries are bounded by $M_*$ and are counted in
$B_L^{\mathrm F}(q,r)$. The group count, encoded tables, and precision
constants are chosen using $f_0,K,q,r$, before the eventual iteration
count is specified. Both the front end and evaluator halt on valid inputs;
clipping handles the exceptional set at every time. Enlarging the common
constant $C$ to include the fixed $H^2$ proves
\eqref{eq:ffn-resolution-error}--\eqref{eq:ffn-resolution-time} for all
admissible $T$ with one parameter choice.
\end{proof}
The inversion from an iteration budget to a resolution remains in the
proof of Lemma~\ref{lem:neural-index-approx} in Section~\ref{app:proofs-section3}.

% End FFN construction

\section{Transformer approximation construction}\label{app:transformer-proof}
We implement the shared Taylor evaluator of Appendix~\ref{app:ffn} in the
post-LN Transformer of Section~\ref{sec:neural-estimators}. Three additional
steps are needed. The input is distributed across tokens, so attention
must gather it and locate its cell group. Layer normalization changes the
state scale, so we encode the evaluator's values relative to a common
reference coordinate. Finally, the output function is affine, so we
correct the computed value before reading it from the normalized state.

\begin{theorem}[Finite-resolution construction for looped post-LN Transformers]
\label{thm:tr-construction}
Fix $d\ge1$, $s>0$ and $H>0$, the token partition of
Section~\ref{sec:neural-estimators}, and $\varepsilon_{\rm LN}>0$.
Let $\mu$ be uniform probability measure on $[0,1]^d$.
There exist a constant $C>0$ and thresholds
$q_0^{\mathrm{Tr}},r_0^{\mathrm{Tr}},M_0^{\mathrm{Tr}}>0$, depending only on
these fixed quantities, with the following property. For every
$M_{\rm par}\ge M_0^{\mathrm{Tr}}$, integers $q\ge q_0^{\mathrm{Tr}}$,
$r\ge r_0^{\mathrm{Tr}}$, dyadic resolution $K=2^m$ with $m\ge1$, and
$f_0\in\Holder$, there is one admissible input-function, update-block, and output-function parameter
choice for the post-LN Transformer whose looped candidates
$g_T\in\mathcal F_{L,\mathrm{Tr}}(q,r,T)$ satisfy
\begin{equation}\label{eq:tr-resolution-error}
 \|g_T-f_0\|_{L^2(\mu)}^2\le CK^{-2s}
\end{equation}
for every integer
\begin{equation}\label{eq:tr-resolution-time}
 T\ge C\{\log^2(eK)+(K^d/r)\log(eK)\}.
\end{equation}
The parameter choice may depend on $f_0,K,q,r$, but is the same for all
such $T$. Every scalar coefficient belongs to
$[-M_{\rm par},M_{\rm par}]$, and all endpoint and block coordinates are
counted in $B_L^{\mathrm{Tr}}(q,r)$, which is independent of $T$.
\end{theorem}

\noindent\textbf{Proof roadmap.}
The proof implements the Taylor evaluator from Appendix~\ref{app:ffn}
inside a post-LN Transformer in three stages.
Section~\ref{app:postln} uses paired homogeneous coordinates to represent
the evaluator's registers: layer normalization changes their common scale
while preserving the encoded native computation.
Section~\ref{app:router} uses attention to aggregate the input and find
its coarse-cell address, then loads the corresponding Taylor table and
runs the shared evaluator. The available hidden width permits
$N_{\rm grp}\asymp\min\{r,K^d\}$ groups, so table evaluation takes
$O((1+K^d/r)\log(eK))$ iterations.
Section~\ref{app:readout} corrects the remaining normalization scale so
that an affine output function recovers the target value.
Routing and terminal correction together cost $O(\log^2(eK))$, while
Taylor-table evaluation costs
$O((1+K^d/r)\log(eK))$. Outside small routing and cell-boundary
exceptional sets, Taylor evaluation and terminal correction have
$O(K^{-s})$ error; clipping controls the exceptional sets.
Section~\ref{app:tr-final-proof} combines these ingredients, checks the
parameter bounds, and proves \eqref{eq:tr-resolution-error}--\eqref{eq:tr-resolution-time}
for every $T$ beyond the finishing time.

\subsection{Homogeneous simulation through layer normalization}
\label{app:postln}
We distinguish the values used by the evaluator from their scaled
representation in a token. For a \emph{native state} $z\in\mathbb R^p$, set $q_{\mathrm{rep}}=2(p+1)$.
For $q\ge q_{\mathrm{rep}}$, let $L_q=\lfloor q/q_{\mathrm{rep}}\rfloor$ and
define $\mathcal I_q(z)$ by repeating $(z,1,-z,-1)$ exactly $L_q$ times
and padding with zeros to dimension $q$. For $\alpha>0$, we call
$\alpha\mathcal I_q(z)$ its \emph{physical token state}.
Three reference quantities have different roles.
The fixed $+1$ entry of the paired encoding has physical value $\alpha$;
it supplies homogeneous affine biases. The native register $C$ belongs to
the spatial selector and may grow during initialization. The two reference
\emph{tokens}, whose native key coordinates are $+1$ and $-1$, generate the
attention comparison signal. Neither $C$ nor a reference token is the
homogeneous reference coordinate.

\subsubsection{Exact simulation of a native residual update}
Let $\theta$ collect all attention, FFN, and layer-normalization parameters
in the block
\begin{equation}\label{eq:tr-postln-block}
Y=\operatorname{LN}_1\{S+\operatorname{Attn}(S)\},\qquad
\Phi^{\mathrm{postLN}}(S)
=\operatorname{LN}_2\{Y+\operatorname{FFN}(Y)\}.
\end{equation}
Both normalizations have fixed stabilization $\varepsilon_{\mathrm{LN}}>0$
and trainable gains and biases. The block has
\begin{equation}\label{eq:tr-block-count}
P_{\theta}^{\mathrm{postLN}}(q,r)=4q^2+(2q+1)r+9q
\end{equation}
dense parameters. The input function with local inputs and the output function are counted in
\eqref{eq:tr-dense-counts} below. Throughout the construction the parameter bound
$M_{\rm par}$ is fixed above a threshold depending only on the fixed model
quantities, including $\sqrt{1+\varepsilon_{\mathrm{LN}}}$.

\begin{lemma}[Homogeneous state representation]\label{lem:homogeneous}
Let a residual function on $\mathbb R^p$ have the form
\begin{equation}\label{eq:tr-native-residual}
\Phi^{\rm res}(z)=z+V\operatorname{ReLU}(Uz+b).
\end{equation}
For $q\ge q_{\mathrm{rep}}$, the paired encoding defined above satisfies
\[
\lambda_q=\frac{2L_q}{q},\qquad
\frac1{q_{\mathrm{rep}}}\le\lambda_q\le\frac2{q_{\mathrm{rep}}}.
\]
There is a post-LN block with zero attention output and the same FFN
hidden width as $\Phi^{\rm res}$ such that
\begin{equation}\label{eq:tr-homogeneous-simulation}
S=\alpha\mathcal I_q(z),\quad\alpha>0
 \quad\Longrightarrow\quad
 \Phi^{\mathrm{postLN}}(S)
 =\alpha'\mathcal I_q(\Phi^{\rm res}(z)),\quad\alpha'>0.
\end{equation}
The representation does not increase the magnitudes of the copied
FFN coefficients. If the entries of $U,V,b$ have magnitude at most $M$,
the block is admissible whenever
$M_{\rm par}\ge\max\{M,\sqrt{1+\varepsilon_{\rm LN}}\}$.
\end{lemma}
\begin{proof}
Set both normalization gains to $\sqrt{1+\varepsilon_{\mathrm{LN}}}$
and their biases to zero. The coordinates of $\mathcal I_q(z)$ have
mean zero, so the first normalization produces
$\widetilde\alpha\mathcal I_q(z)$ for some $\widetilde\alpha>0$.
Each hidden unit reads $Uz$ from the first copy and represents its
bias $b$ as a coefficient on that copy's reference coordinate.
Its preactivation is therefore $\widetilde\alpha(Uz+b)$.
Positive homogeneity of ReLU gives the desired residual increment
on every positive state copy and its negative on every negative copy.
The reference and padding coordinates receive zero increments.
The intermediate state is exactly
$\widetilde\alpha\mathcal I_q(\Phi^{\rm res}(z))$, which again has mean zero.
The second normalization changes only its positive common scale.
All actual FFN biases are zero, and copying the output columns
does not change their coefficient magnitudes.
\end{proof}

Ratios to a reference coordinate identify the represented state in
the proof. The network output will instead be obtained by the affine
function constructed below. Repeating the paired representation keeps
$\lambda_q$ between fixed positive constants as $q$ increases;
all repeated matrix entries are included in the dense parameter count.
Throughout the homogeneous construction, every constant in an active
preactivation is represented as a multiple of the current reference
coordinate; consequently, each clipping and threshold operation is
unchanged by the positive common scale.

\subsubsection{Scale recurrence and homogeneous instructions}
For a represented token $\alpha\mathcal I_q(z)$, set
$d_\alpha=\alpha^{-2}$ and
$\rho_{\rm LN}=\varepsilon_{\rm LN}/(1+\varepsilon_{\rm LN})$.
Since the coordinate mean is zero and the mean square is
$\alpha^2\lambda_q(\|z\|_2^2+1)$, one normalization gives exactly
\begin{equation}\label{eq:tr-scale-recurrence}
 d_\alpha'=
 \frac{\lambda_q(\|z\|_2^2+1)+\varepsilon_{\rm LN}d_\alpha}
 {1+\varepsilon_{\rm LN}}.
\end{equation}
When attention is active, the first normalization uses the native state
\emph{after} its attention increment. The second uses the state after the
row-wise FFN. The projections below write paired increments and leave the
homogeneous reference unchanged, so both normalized states retain the same
representation. Thus the FFN instructions operate on native values even
when physical token scales differ.

The native instructions below use simultaneous assignments: every
right-hand side and transition test reads the old state. For a phase $e$
whose active value is $C=1$, a signed affine residual $A(z)$ uses
\[
 \ReLU\{A(z)+\Lambda(e-C)\}
 -\ReLU\{-A(z)+\Lambda(e-C)\},
\]
and a hinge $\ReLU(B(z))$ uses
$\ReLU\{B(z)+\Lambda(e-C)\}$. Active instructions are exact; inactive
ones vanish when $|A|,|B|<\Lambda C$. Every numerical constant in these
forms multiplies the homogeneous reference coordinate, not the selector
register $C$. The selector's own phases instead have active value $C$,
as in Section~\ref{app:joint}. These gates use only one ReLU hidden layer.

\subsection{Attention routing and table evaluation}\label{app:router}
The homogeneous representation lets us describe the computation in native
values. We first gather the input in a controller token, then use attention
comparisons to determine its coarse-cell address and load the corresponding
Taylor code. The attention and FFN matrices remain fixed throughout.
In the routing formulas below, input coordinates are numbered from $0$ to
$d-1$: $x_i$ denotes the $(i+1)$st coordinate of the original input vector.
The local and aggregate input registers use the same convention; coordinate
$i$ corresponds to register $U_{i+1}$ of the selector.

\subsubsection{Token protocol, native registers, and parameter counts}
Use the fixed partition $I_1,\ldots,I_M$ of the input coordinates, with
$M$ data tokens, one controller token, and two reference tokens. Thus
\begin{equation}\label{eq:tr-token-count}
 N=M+3.
\end{equation}
The data-token input function at position $j$ reads only $x_{I_j}$, and the
other three input functions are constant. Counting all affine endpoints,
projections, biases, and normalization gains gives
\begin{equation}\label{eq:tr-dense-counts}
\begin{split}
 P_{\rm end}&=q(d+2N)+1,\\
 P_{\rm block}&=4q^2+(2q+1)r+9q,\\
 B_L^{\mathrm{Tr}}&=P_{\rm end}+P_{\rm block},\qquad
 B_U^{\mathrm{Tr}}=P_{\rm end}+TP_{\rm block}.
\end{split}
\end{equation}
These are the full dense counts, including entries assigned special values
by the construction.

Use the $q_0$ native coordinates of the evaluator in
Section~\ref{app:base-engine}. The attention-addressed front end below uses
$d+7$ further coordinates; reserve three zero coordinates so their total
allocation is $q_*=q_0+d+10$, as in Appendix~\ref{app:ffn}.
This front end receives its address from attention, loads the evaluator directly, and clears its four clocks
after lookup. Its realization and count are specified below.
Reserve the 15 data and 18 phase coordinates for the terminal correction in
Section~\ref{app:readout}. The additional attention and routing coordinates
are listed below. Here $x^{\rm loc},x^{\rm agg}\in\mathbb R^d$;
each indexed routing coordinate is allocated for $i=0,\ldots,d-1$.
\begin{center}\small
\begin{tabular}{lp{0.68\textwidth}}
\toprule
Coordinate & Role\\
\midrule
$k$ & Key coordinate; the two reference tokens use $+1$ and $-1$.\\
$a_{\rm att}$ & Attention comparison signal read by the FFN.\\
$e_{\rm boot}$ & Phase coordinate for input initialization.\\
$x^{\rm loc}$ & Local input entries, with zeros outside the token partition.\\
$x^{\rm agg}$ & Input entries gathered by attention.\\
$u_i$ & Centered binary remainder for input coordinate $i$.\\
$w_i$ & Saturating comparator used to determine the next bit.\\
$b_i$ & Binary prefix accumulated for coordinate $i$.\\
$p_i$ & Weight of the next binary digit.\\
$\tau_{{\rm cmp},i}$ & Timer for comparisons within one bit decision.\\
$\tau_{{\rm level},i}$ & Timer for the bits extracted from coordinate $i$.\\
$e_{C,i},e_{M,i},e_{H,i}$ & Phase coordinates for comparison, commit, and handoff.\\
\bottomrule
\end{tabular}
\end{center}
The total native dimension and encoding threshold are
\begin{equation}\label{eq:tr-native-dimension}
 p_{\rm nat}=q_*+11d+36,\qquad q_{\rm hom}=2(p_{\rm nat}+1).
\end{equation}
Apply Lemma~\ref{lem:homogeneous} with $p=p_{\rm nat}$ at any
$q\ge q_{\rm hom}$, so $1/q_{\rm hom}\le\lambda_q\le2/q_{\rm hom}$.
The constant $q_{\rm hom}$ is independent of $q,r,K$. All tokens start
with physical scale one and native $C=1$. Data tokens contain their local
input coordinates; the reference tokens have $k=1$ and $k=-1$.
The controller contains $e_{\rm boot}=1$, the front-end timer prescribed
in Appendix~\ref{app:ffn}, and the terminal-correction code prescribed below. All other
native coordinates are zero. These initializations obey the local-input
protocol and are affine in each token's allowed inputs.

\subsubsection{Fixed attention projections and input aggregation}
For token $j$ with physical state $\alpha_j\mathcal I_q(z_j)$, choose
projections whose query, key, and nonzero value channels are
\begin{equation}\label{eq:tr-fixed-projections}
\begin{split}
 Q_j&=\alpha_j\Big(\sum_i u_{j,i}\Big)\mathbf1_q,\qquad
 K_j=\frac{\alpha_j k_j}{\sqrt q}\mathbf1_q,\\
 V_j&=\alpha_j(k_j,Nx^{\rm loc}_{j,0},\ldots,
                         Nx^{\rm loc}_{j,d-1},0,\ldots,0).
\end{split}
\end{equation}
Each projection reads the first positive copy. The output projection writes
value channel zero into $a_{\rm att}$ and the next $d$ channels into
$x^{\rm agg}$, with opposite signs in negative copies. All biases are zero;
every projection entry has magnitude at most $\max\{N,1\}$. In particular,
$Q_jK_l^\top/\sqrt q=\alpha_j\alpha_l(\sum_i u_{j,i})k_l$,
including the prescribed head scaling.

At boot all queries vanish. Uniform attention therefore gathers the complete
input $x$ into every token's aggregate registers, while the reference values
cancel in the signal channel. The first normalization changes only the
common scale of each row. The controller's boot instruction copies the
aggregate into the Taylor input and selector registers $U_i$.
At every FFN application, in every token, unconditional signed-linear
updates clear $x^{\rm loc},x^{\rm agg},a_{\rm att}$. All assignments read
the old state, so these clears do not erase the input before it is copied
or the signal before the comparator uses it. Subsequently, local-input
values remain zero, and only the two reference tokens supply nonzero values.

\subsubsection{Uniform comparison signal}
The reference rows differ only in the sign of $k$, so they retain a common
physical scale $\eta_t$. After boot, their native states consist of
$C=1$ and $k=\pm1$; data rows retain only $C=1$.
Before the first normalization, the reference native squared norm including
the homogeneous reference is $3+\|x\|^2$, and after clearing it is three.
Since $q_{\rm hom}\ge2(d+3)$, the scale recurrence gives
\begin{equation}\label{eq:tr-reference-scale}
 3/q_{\rm hom}\le\eta_t^{-2}\le1,\qquad
 1\le\eta_t\le\sqrt{q_{\rm hom}/3}.
\end{equation}
The controller has scale $0<\alpha_t\le\sqrt{q_{\rm hom}}$, since its
homogeneous reference remains fixed and
\eqref{eq:tr-scale-recurrence} preserves $\alpha_t^{-2}\ge1/q_{\rm hom}$.

Only one controller remainder $u_i$ is active at a time. Writing it as $u$,
the two reference logits are $\pm\alpha_t\eta_tu$ and the other logits are
zero. The signal in native coordinates, unchanged by the first normalization,
is therefore
\begin{equation}\label{eq:routesignal}
 a_t(u)=\frac{\eta_t}{\alpha_t}F_N(\alpha_t\eta_tu),\qquad
 F_N(t)=\frac{2\sinh t}{N-2+2\cosh t}.
\end{equation}
For $t\ge0$, $F_N(t)\ge(2/N)\tanh t$, and $F_N'(t)$ is the variance
of a softmax-distributed variable in $\{-1,0,1\}$, hence is at most one.
Put $C_{\rm att}=q_{\rm hom}/\sqrt3$. For $|u|\le1$,
$|\alpha_t\eta_tu|\le C_{\rm att}$. Monotonicity of $\tanh(t)/t$ and
\eqref{eq:tr-reference-scale} consequently give
\begin{equation}\label{eq:signalbound}
\begin{split}
 \operatorname{sign}a_t(u)&=\operatorname{sign}u,\qquad
 c_{\rm sig}|u|\le|a_t(u)|\le(q_{\rm hom}/3)|u|,\\
 c_{\rm sig}&=\frac2N\frac{\tanh C_{\rm att}}{C_{\rm att}}>0.
\end{split}
\end{equation}
These constants do not depend on the iteration, grid resolution, or padded
width. No lower bound on the controller scale is needed for this comparison.

\subsubsection{Complete binary-routing program}
For $K=2^m$, choose $A=2^h$ with $0\le h\le m$ and let
$N_{\rm grp}=A^d$. The desired coarse address and normalized prefixes are
\begin{equation}\label{eq:tr-coarse-address}
 j_i(x)=\min\{A-1,\lfloor Ax_i\rfloor\},\qquad
 I(x)=\sum_{i=0}^{d-1}A^i j_i(x),\qquad b_i=j_i(x)/A.
\end{equation}
For $h\ge1$, choose a comparison length $L$ as specified below and write
$\zeta_k=2^{-(k-1)}$. Initializing coordinate $i$ means setting
\begin{equation}\label{eq:tr-route-initialization}
 (u_i,w_i,\tau_{{\rm cmp},i},\tau_{{\rm level},i},b_i,p_i)
 =(2x_i-1,0,\zeta_L,\zeta_h,0,1)
\end{equation}
and activating $e_{C,i}$. Unused coordinates already have $w_i=b_i=0$,
so only $u_i,p_i$ and the two timers require nonzero assignments.
The boot instruction performs this initialization for $i=0$ and clears
$e_{\rm boot}$. In that instruction $x_0$ is read from the old aggregate.

In comparison phase $C_i$, the FFN uses the current native attention signal:
\begin{equation}\label{eq:tr-route-comparison}
 w_i^+=\Pi_1(2w_i+a_{\rm att}),\qquad
 \tau_{{\rm cmp},i}^+=2\tau_{{\rm cmp},i}.
\end{equation}
It stays in $C_i$ unless the old timer equals one, in which case it moves
to commit phase $M_i$. Thus exactly $L$ comparisons precede every commit.
The commit assignments are
\begin{equation}\label{eq:tr-route-commit}
\begin{aligned}
 u_i^+&=\Pi_1(2u_i-w_i),&
 b_i^+&=b_i+\ReLU((p_i+w_i-1)/2),&p_i^+&=p_i/2,\\
 w_i^+&=0,&\tau_{{\rm cmp},i}^+&=\zeta_L,&
 \tau_{{\rm level},i}^+&=2\tau_{{\rm level},i}.
\end{aligned}
\end{equation}
It returns to $C_i$ unless the old level timer equals one, in which case
it moves to handoff phase $H_i$. This performs exactly $h$ commits.
At $H_i$, set $u_i,w_i,\tau_{{\rm cmp},i},\tau_{{\rm level},i},p_i$ to
zero and remove $e_{H,i}$. If $i<d-1$, retain $b_i$, initialize coordinate
$i+1$ by \eqref{eq:tr-route-initialization}, and activate $e_{C,i+1}$.
For $i=d-1$, instead assign the old prefixes to the selector index,
\begin{equation}\label{eq:tr-normalized-address}
 I\longleftarrow\sum_{i=0}^{d-1}\frac{A^{i+1}}{N_{\rm grp}}b_i
                  =I(x)/N_{\rm grp},
\end{equation}
clear every $b_i$, and activate the selector warm-up phase $e_{\rm w}=1$.
All coefficients in this address assignment are at most one. These
simultaneous assignments preserve the prefixes until they have been read.

For $h=0$, boot copies the input and passes directly to $e_{\rm w}$ with
index zero; no routing timer is used. In either case, at handoff all
routing registers in the controller are zero. Its input, front-end timer,
and terminal-correction code remain intact. The query coordinates are then zero, the
local values are zero, and the equal-scale reference values cancel.
Attention is therefore exactly zero in every subsequent iteration without
changing its matrices.

\subsubsection{Correctness and exceptional inputs}
For a fixed remainder with $|u_i|\ge\delta$, every comparison signal has
the same sign and magnitude at least $c_{\rm sig}\delta$. Starting from
$w_i=0$, induction in \eqref{eq:tr-route-comparison} gives
\[
 |w_i^{(j)}|\ge\min\{1,c_{\rm sig}\delta(2^j-1)\}.
\]
Once it reaches one it stays there until the commit. Thus $L$ comparisons
return $w_i=\operatorname{sign}u_i$ when
$c_{\rm sig}\delta(2^L-1)\ge1$. For this exact sign, the remainder update
is the centered binary shift, and the prefix increment is $p_i/2$ for
sign $+1$ and zero for sign $-1$. Hence the initialized program returns
the prefixes in \eqref{eq:tr-coarse-address} after $h$ bits.

Exclude inputs for which one of the $dh$ ideal centered binary remainders
has magnitude below $\delta$, together with the measure-zero endpoint
faces $x_i=1$. Under uniform measure, each remainder event has measure
$\delta$, so the exceptional set satisfies
$\mu(\mathcal B_{\rm route})\le dh\delta$. Choose
\begin{equation}\label{eq:tr-route-accuracy}
 \delta=K^{-2s}/(dh),\qquad
 L=\left\lceil\log_2\left(1+\frac1{c_{\rm sig}\delta}\right)\right\rceil.
\end{equation}
Then $\mu(\mathcal B_{\rm route})\le K^{-2s}$ and
$L=O_{d,s,N}(m+\log(h+1))$. Off this set, induction over commits justifies
the use of the ideal remainders and makes the selected address exact.
On the exceptional set the phases and timers still run for their prescribed
lengths: $|u_i|,|w_i|\le1$, $p_i=2^{-j}$ after $j$ commits, and every
prefix increment lies in $[0,p_i/2]$. Thus $b_i\in[0,1]$ and both timers
are at most two throughout routing. The program takes exactly
$1+d\{h(L+1)+1\}$ blocks, including boot and handoffs.

\subsubsection{ReLU features and routing resource count}
Use the homogeneous gates from Section~\ref{app:postln} with
$\Lambda=64(q_{\rm hom}+1)$. During routing $C=1$ and
$|a_{\rm att}|\le q_{\rm hom}/3$. At later stages the routing coordinates
are zero and the selector coordinates are $O(C)$. Thus every inactive
affine or hinge feature is suppressed. Reference and data tokens have
zero phases, so their native states remain fixed after unconditional
clearing. Every feature reads the same old state.

A signed affine residual requires two units. A clipped assignment
$y^+=\Pi_1(v)$ has residual
$-1-y+\ReLU(v+1)-\ReLU(v-1)$ and requires four units. A counter test uses
\[
 \ReLU\{2\tau_{{\rm cmp},i}-1+\Lambda(e_{C,i}-C)\}
 \quad\hbox{or}\quad
 \ReLU\{2\tau_{{\rm level},i}-1+\Lambda(e_{M,i}-C)\}.
\]
These are exact zero-one indicators at the active checks. The first alone
moves $C_i$ to $M_i$ when needed. At a commit, one $\ReLU(e_{M,i})$ unit
makes the default move to $C_i$, and the second test redirects it to $H_i$.
An unconditional handoff uses one $\ReLU(e_{H,i})$ unit.

Comparison uses four units for its clipped update, two for its timer,
and one for the conditional phase move, at most eight in total.
Commit uses four for the remainder, one for the prefix hinge, eight for
bit-weight, comparator, and timer updates, and two for the phase move,
at most sixteen. Handoff clears five data registers with ten units and
uses one phase unit. The table below counts the hidden units needed to
implement routing within the shared block; unused units are zero-padded:
\begin{center}\small
\begin{tabular}{lr}
\toprule
Update group & Hidden units\\
\midrule
Boot input copies and first-coordinate initialization & $4d+9$\\
Comparison, commit, and handoff clears & $(8+16+11)d$\\
Initialization of each following coordinate & $8(d-1)$\\
Final index assignment and prefix clearing & $2+2d$\\
Unconditional local-input, aggregate, and signal clearing & $4d+2$\\
\midrule
Total additional routing allocation & $53d+5$\\
\bottomrule
\end{tabular}
\end{center}
Boot counts two units for each of the $2d$ input copies, two for each
of $u_0,p_0$ and its two timers, and one phase transition. The next-coordinate
initializations each use eight units because their comparator and prefix
are already zero. Final index assignment and prefix clearing are gated
by $H_{d-1}$ and share its already counted phase transition.

\subsubsection{Handoff to the Taylor evaluator}
Routing has supplied the group address. The front end now loads that
group's code and starts the evaluator from Appendix~\ref{app:ffn}.
The front end uses the following $d+7$ coordinates; it loads the code
directly into the evaluator and activates its first phase.
\begin{center}\small
\begin{tabular}{lp{0.68\textwidth}}
\toprule
Coordinate & Role\\
\midrule
$U_1,\ldots,U_d$ & Input copies scaled together with $C$.\\
$C$ & Common scale for the selector data and active phase.\\
$I$ & Group address, rescaled to an integer before lookup.\\
$\tau_{\rm front}$ & Timer for warm-up and address scaling.\\
$e_{\rm w}$ & Warm-up phase: repeated doubling of the state scale.\\
$e_{\rm a}$ & Phase that initializes the address-scaling timer.\\
$e_{\rm s}$ & Phase that scales the address to its integer value.\\
$e_{\rm l}$ & Lookup phase that loads the code and starts evaluation.\\
\bottomrule
\end{tabular}
\end{center}
Initially $C=1$, $U_i=x_i$, and
$\tau_{\rm front}=2^{-(L_m-1)}$, with $L_m=G+dm$ from Appendix~\ref{app:ffn}.
Routing has already written the normalized address into $I$.
Warm-up uses the same bounded-weight doubling updates as
Section~\ref{app:joint}, leaving $I$ unchanged, and produces $C=C_*=2^{L_m}$.
The next phase only initializes the scaling timer and advances its clock;
it does not evaluate \eqref{eq:address} or reset $I$. For $dh\ge1$ it sets
$\tau_{\rm front}=2^{-(dh-1)}C$ and enters $e_{\rm s}$; for $h=0$ it goes
directly to lookup. The scaling phase performs $dh$ updates
$I^+=2I$, $\tau_{\rm front}^+=2\tau_{\rm front}$ and exits at the old
condition $\tau_{\rm front}=C$. Thus the lookup address is the integer
$I(x)$ on valid routing inputs.

For $j=0,\ldots,N_{\rm grp}-1$, let $\gamma_j$ be the base-four code
of the Taylor records in the coarse group with integer address $j$,
using the encoding in \eqref{eq:code}. The lookup hinges are
\[
 H_j(I,C)=\ReLU(I-jC/C_*+1)-2\ReLU(I-jC/C_*)
                         +\ReLU(I-jC/C_*-1).
\]
In that phase, add $\sum_{j=0}^{N_{\rm grp}-1}\gamma_j H_j(I,C)$ directly to the initially
zero evaluator code. Remove $e_{\rm l}$ and initialize the evaluator's
first phase to one using $C_*^{-1}\ReLU(e_{\rm l})$. Simultaneously rescale
\begin{equation}\label{eq:tr-selector-rescale}
 (C,U_i,I,\tau_{\rm front})
 \longleftarrow C_*^{-1}(C,U_i,I,\tau_{\rm front}).
\end{equation}
Every right-hand side uses the old state, so rescaling does not change
the selected code or first-phase activation. It uses $2(d+3)$ additional
signed-affine units, with residual coefficients $-(1-C_*^{-1})$.
Afterwards $C=1$, $|U_i|,|I|\le1$, $|\tau_{\rm front}|\le2$, and all
four front-end clocks are zero. The three reserved coordinates stay zero.
The evaluator's first instruction resets its polynomial accumulator;
its separate timer $\tau_{\rm eval}$ is initialized only when the first
counted word routine is reached. This is the program in
Section~\ref{app:base-engine}, including its explicit identity steps.
Outside $\mathcal B_{\rm route}$, it evaluates $K^d/N_{\rm grp}$ valid
records in $K^d\tau_m/N_{\rm grp}$ iterations. At $h=0$ the one full-grid
code is loaded in the same way.

This direct front end uses $2d+6$ warm-up units, three timer/transition
units, five address-scaling units, and $3N_{\rm grp}+1$ lookup/handoff
units. Their sum is $3N_{\rm grp}+2d+15$. The rescaling units are counted
with the terminal-correction allocation $C_{\rm add}=2q_*+2d+184$ below. Together with
the evaluator width $r_0$ from \eqref{eq:base-resources} and the $53d+5$
routing units, the sufficient hidden width is
\begin{equation}\label{eq:tr-implementation-width}
 r_{\rm impl}=r_0+C_{\rm add}+3N_{\rm grp}+55d+20.
\end{equation}
This count corresponds to the direct attention-addressed construction above.
The same allocation covers $h=0$ with inactive units zero-padded.
Front-end gates use $\Lambda(e-C)$ and evaluator gates use their own
$\Lambda_{\rm eval}(e-1)$ before homogenization. After handoff, the front
end's auxiliary registers do not enter evaluator features. These
bounds therefore suppress inactive instructions in all subsequent phases.

\subsection{Terminal correction and affine output function}
\label{app:readout}
Taylor evaluation has produced $u\in[-1,1]$, whose rescaled value $Hu$
approximates the target. It remains to read this value from the physical
token: an affine output function cannot divide by the token's changing
reference coordinate. We therefore transform $u$ within the shared block,
clear the auxiliary coordinates, and let normalization converge to a scale
at which an affine readout recovers $Hu$ to the required accuracy.

\subsubsection{Approximation of the inverse normalization map}
\label{sec:analytic-readout}
The required native transformation is
\begin{equation}\label{eq:tr-inverse-normalization}
 g_{\rm inv}(u)=\frac{u}{\sqrt{1-u^2/2}},\qquad
 \frac{\sqrt2g_{\rm inv}(u)}{\sqrt{g_{\rm inv}(u)^2+2}}=u.
\end{equation}
For $|u|\le1$, $|g_{\rm inv}(u)|\le\sqrt2$ and
\[
 g_{\rm inv}(u)=u\sum_{j=0}^{\infty}c_j(u^2/2)^j,
 \qquad c_j=\binom{2j}{j}4^{-j}\in[0,1].
\]
The integers $J,b,\ell$ below control the precision of the terminal correction
and are distinct from the
Taylor-word parameters of Appendix~\ref{app:ffn}. Compute
$t=S_\ell(|u|)/2\in[0,1/2]$ using Lemma~\ref{lem:square}.
Round $c_j/4$ down to a $b$-digit dyadic value $a_j$, and encode
$a_J,\ldots,a_0$ with continuation digits in one base-four scalar.
This code is a counted input-function bias of magnitude at most $2/3$ and uses
$O((J+1)b)$ binary digits. Starting at $p_{J+1}=0$, evaluate
\begin{equation}\label{eq:tr-readout-horner}
 p_j=a_j+\mathsf M_\ell(t,p_{j+1}),\qquad
 v=4\mathsf M_\ell(u,p_0).
\end{equation}
Since $a_j\le1/4$, $t\le1/2$, and the multiplication error is at most
$4^{-\ell}/2\le1/8$, induction gives
$|p_j|\le1/4+|p_{j+1}|/2+1/8\le7/8<1$ whenever
$|p_{j+1}|\le1$. All multiplication inputs therefore remain admissible.

Put $t_0=u^2/2$ and let $P_j$ be the exact truncated Horner values with
coefficients $c_j/4$ and argument $t_0$. Then $0\le P_j\le1/2$ and
$|t-t_0|\le4^{-\ell}/8$. Thus
\[
 |p_j-P_j|\le\tfrac12|p_{j+1}-P_{j+1}|
                  +2^{-b}+\tfrac9{16}4^{-\ell}.
\]
Summing this contracting recursion, allowing for the final multiplication,
and bounding the series tail gives
\begin{equation}\label{eq:tr-readout-accuracy}
 |v-g_{\rm inv}(u)|\le8\,2^{-b}+7\,4^{-\ell}+2^{-J}.
\end{equation}
For $K=2^m$, choose
\[
 b=\lceil sm\rceil+6,\qquad
 \ell=\lceil sm/2\rceil+4,\qquad J=\lceil sm\rceil+6.
\]
The error is $O_s(K^{-s})$ and is less than $1/8$, so $|v|<2$.
The phase implementation below uses 15 data and 18 phase registers and,
including selector rescaling and the Taylor-to-output handoff, at most
\begin{equation}\label{eq:tr-readout-width}
 C_{\rm add}=2q_*+2d+184
\end{equation}
additional hidden units. Its runtime is at most
\begin{equation}\label{eq:tr-readout-time}
 (J+1)(2b+2\ell+8)+4\ell+12=O_s(\log^2(eK)).
\end{equation}
All assignments and phase tests below are simultaneous and read the old
state; unlisted coordinates are unchanged. The timer $\tau_{\rm read}$ is
separate from $\tau_{\rm eval},\tau_{\rm front}$ and both routing timers.

\subsubsection{Coordinates and initialization}
The table below describes the 15 additional data coordinates and the
18 phase coordinates used for terminal correction. Symbols reused from
Appendix~\ref{app:ffn}, including $c,a,t,p,v$, denote separate coordinates
local to this subsection. The Taylor output $u$ remains in its original
coordinate and is unchanged until the terminal transition.
\begin{center}\small
\begin{tabular}{lp{0.68\textwidth}}
\toprule
Coordinate & Role\\
\midrule
$c$ & Unread remainder of the encoded correction coefficients.\\
$a$ & Accumulator for the coefficient being decoded.\\
$t$ & Approximation to $u^2/2$, used in the Horner recursion.\\
$p$ & Current value of the Horner recursion.\\
$v_s,S_s,w_s$ & Square-recurrence state, approximation, and scale for $|u|$.\\
$v_+,v_-$ & Intermediate states of the two multiplication channels.\\
$S_+,S_-$ & Square approximations whose difference gives a product.\\
$w$ & Scale shared by the multiplication channels.\\
$\tau_{\rm read}$ & Timer for coefficient decoding and square recurrences.\\
$\beta$ & Continuation bit for the correction coefficient list.\\
$v$ & Corrected value approximating $g_{\rm inv}(u)$.\\
$e_1,\ldots,e_{18}$ & Phase coordinates for the correction program below.\\
\bottomrule
\end{tabular}
\end{center}

Initially, $c$ contains the coefficient code and the other 14 data
coordinates are zero. Add 18 phase coordinates $e_1,\ldots,e_{18}$,
all initially zero. One hidden unit transfers the Taylor terminal
phase into $e_1$. Let
\[
 D(c)=3\ReLU(c-1/6)-3\ReLU(c-1/2),\qquad
 \zeta_k=2^{-(k-1)}.
\]
The code contains $a_J,\ldots,a_0$, each with $b$ digits read from
least to most significant. A continuation digit follows every word;
all are one except the last, which is zero.

\subsubsection{Square evaluation and coefficient decoding}
The table below specifies how the shared block computes $t$ and decodes
one coefficient, listing each phase's updates and transition.
The notation $h_s=\ReLU(v_s-w_s/2)$ denotes an expression in the
current state, not an additional coordinate.
\begin{center}\small
\begin{tabular}{cll}
\toprule
Phase & Assignments & Next phase\\
\midrule
1 & $v_s=S_s=|u|,\ w_s=1,\ \tau_{\rm read}=\zeta_\ell$ & 2\\
2 & $v_s^+=v_s/2-h_s,\ S_s^+=S_s-v_s/2+h_s$ & 3\\
  & $w_s^+=w_s/4$ & \\
3 & $\tau_{\rm read}^+=2\tau_{\rm read}$ & 4 if $\tau_{\rm read}=1$; otherwise 2\\
4 & $t=S_s/2,\ p=a=0,\ \tau_{\rm read}=\zeta_b$ & 5\\
5 & $c^+=4c-2D(c),\ a^+=(a+D(c))/2$ & 6\\
6 & $\tau_{\rm read}^+=2\tau_{\rm read}$ & 7 if $\tau_{\rm read}=1$; otherwise 5\\
\bottomrule
\end{tabular}\end{center}
Absolute values are represented by $|x|=\ReLU(x)+\ReLU(-x)$.
At the checks of a length-$k$ iteration, the old counter values are
$\zeta_k,\ldots,1/2,1$. Thus precisely $k$ iteration/check pairs
are performed. The final check writes two into the counter, which
is reset by the following initialization.

\subsubsection{Multiplication and Horner updates}
Each multiplication uses four phases. For inputs $a_*,b_*$ that
remain fixed until the final assignment, put
$w_+=(a_*+b_*)/2$, $w_-=(a_*-b_*)/2$ and
$h_\pm=\ReLU(v_\pm-w/2)$. The table below gives the four-phase
multiplication routine used in both the Horner update and final output.
\begin{center}\small
\begin{tabular}{cll}
\toprule
Phase & Assignments & Next phase\\
\midrule
$i$ & $v_\pm=S_\pm=|w_\pm|,\ w=1,\ \tau_{\rm read}=\zeta_\ell$ & $i+1$\\
$i+1$ & $v_\pm^+=v_\pm/2-h_\pm,\ S_\pm^+=S_\pm-v_\pm/2+h_\pm$ & $i+2$\\
      & $w^+=w/4$ & \\
$i+2$ & $\tau_{\rm read}^+=2\tau_{\rm read}$ & $i+3$ if $\tau_{\rm read}=1$; otherwise $i+1$\\
$i+3$ & $y=\lambda(S_+-S_-)+a_{\rm add}$ & specified below\\
\bottomrule
\end{tabular}\end{center}
The two uses of this routine have the following inputs, outputs, and
transitions:
\begin{center}\small
\begin{tabular}{cccccc}
\toprule
$i$ & $(a_*,b_*)$ & $y$ & $\lambda$ & $a_{\rm add}$ & Next phase\\
\midrule
7 & $(t,p)$ & $p$ & 1 & $a$ & 11\\
14 & $(u,p)$ & $v$ & 4 & 0 & 18\\
\bottomrule
\end{tabular}\end{center}
The two sets of phases use the same data coordinates. They implement
$p\leftarrow a_j+\mathsf M_\ell(t,p)$ and
$v\leftarrow4\mathsf M_\ell(u,p)$, respectively.
The coefficient four occurs only in the final assignment and does
not change the multiplication input domain.

\subsubsection{Continuation and terminal transition}
The remaining phases either start the next coefficient or finish the
output calculation and clear the auxiliary coordinates:
\begin{center}\small
\begin{tabular}{cll}
\toprule
Phase & Assignments & Next phase\\
\midrule
11 & $c^+=4c-2D(c),\ \beta=D(c)$ & 12\\
12 & none & 13 if $\beta=1$; otherwise 14\\
13 & $a=0,\ \tau_{\rm read}=\zeta_b$ & 5\\
18 & all old coordinates except $C$ set to zero; & none\\
   & all new data coordinates except $v$ set to zero & \\
\bottomrule
\end{tabular}\end{center}
At phase 18, $C=1$ and the other 17 new phase coordinates vanish.
All extra routing coordinates of the controller have already been cleared;
the terminal instructions need only clear the $q_*$ original coordinates
and the terminal-correction registers listed in the table.
Removing $e_{18}$ leaves only $C=1$ and $v$. All residual updates
subsequently vanish, and the normalization convergence established
below applies.

\subsubsection{ReLU realization and resource bounds}
For a phase coordinate $e$ and an assignment
\[
 z_i\longleftarrow F(z)+\sum_\nu a_\nu\ReLU(G_\nu(z)),
\]
where $F,G_\nu$ are affine, use the preactivations
\[
 \pm(F(z)-z_i)+\Lambda(e-C),\qquad
 G_\nu(z)+\Lambda(e-C).
\]
The first pair represents the signed affine residual; each additional
hinge uses one unit. At the active phase, $e=C=1$ and the update is
exact. At every other phase, the $O(C)$ bounds on the relevant
coordinates make these units inactive for the same sufficiently large
$\Lambda=64(q_{\rm hom}+1)$.
Every feature reads the same old state, so all phases are represented
in one hidden layer.

An unconditional transition uses one $\ReLU(e)$ unit with output
coefficients $-1$ on the old phase and $1$ on the next phase.
A conditional transition adds one feature to this default move:
\[
 \ReLU(2\tau_{\rm read}-1+\Lambda(e-C))
 \quad\hbox{or}\quad
 \ReLU(\beta+\Lambda(e-C)).
\]
On the admissible counter and continuation values, these are exact
zero-one indicators. The table below counts the hidden units for these
phases to establish \eqref{eq:tr-readout-width}, using two units per signed
affine residual and counting repeated hinges separately.
\begin{center}\footnotesize
\begin{tabular}{lrr}
\toprule
Phases & Units per phase & Total\\
\midrule
Square (1--4) & $13,9,4,9$ & 35\\
Coefficient decoding (5--6) & $9,4$ & 13\\
Horner multiplication (7--10) & $21,15,4,3$ & 43\\
Continuation (11--13) & $9,2,5$ & 16\\
Final multiplication (14--17) & $21,15,4,3$ & 43\\
Terminal transition (18) & $2(q_*-1)+2\cdot14+1$ & $2q_*+27$\\
\bottomrule
\end{tabular}\end{center}
For example, phase 1 uses four units for each absolute-value
assignment, two for each constant assignment and one for the phase
transition, totaling $4+4+2+2+1=13$.
The first five rows sum to 150. Adding the $2(d+3)$ units for
\eqref{eq:tr-selector-rescale} and one terminal-to-output transition gives
$2q_*+2d+184$ units in total and proves \eqref{eq:tr-readout-width}.
The 15 data and 18 phase coordinates are precisely those reserved in
\eqref{eq:tr-native-dimension}; no further native coordinates are needed.

Each coefficient requires at most $2b+2\ell+5$ iterations.
Initialization, final multiplication and the two terminal transitions
require at most another $4\ell+6$. This proves the slightly looser
bound \eqref{eq:tr-readout-time}. Neither the coordinate count nor the number
of phases depends on $b,\ell,J$.

\subsubsection{Normalization convergence and affine output function}
The correction has produced $v\approx g_{\rm inv}(u)$. We now show that
the remaining normalization steps make it readable by an affine output.
After the terminal transition, the controller's only nonzero native
coordinates are $C=1$ and $v$; its homogeneous reference remains one.
All native updates and attention increments vanish. In
\eqref{eq:tr-scale-recurrence}, $\|z\|_2^2+1=v^2+2$, so, writing
$d_{\alpha,0}$ for the inverse-square scale after this transition,
\begin{equation}\label{eq:tr-terminal-scale}
 d_*=\lambda_q(v^2+2),\qquad
 d_{\alpha,j}=d_*+\rho_{\rm LN}^{2j}(d_{\alpha,0}-d_*).
\end{equation}
On valid inputs, routing data are uniformly bounded, selector data grow
at most as $CK^{2(k+1)+d+1}$, and evaluator and terminal-correction data are bounded.
The contracting scale recurrence therefore gives
$1/q_{\rm hom}\le d_{\alpha,0}\le CK^{C_{d,s}}$, uniformly over the padded
widths. Since $|v|<2$ and $\lambda_q\ge1/q_{\rm hom}$, at most
$C_{d,s,\varepsilon_{\rm LN}}\log(eK)$ additional blocks ensure
$|d_{\alpha,j}/d_*-1|\le CK^{-s}$ and also at most $1/2$.

Choose the affine output function to read the first positive copy of the controller's
$v$ coordinate with weight
\begin{equation}\label{eq:tr-affine-readout}
 w_q=H\sqrt{2\lambda_q}.
\end{equation}
At the limiting scale its value is
$H\sqrt2v/\sqrt{v^2+2}$. The map
$v\mapsto\sqrt2v/\sqrt{v^2+2}$ is 1-Lipschitz, so
\eqref{eq:tr-inverse-normalization} and \eqref{eq:tr-readout-accuracy}
bound its error relative to $Hu$ by $CHK^{-s}$. The finite-scale error has
the same order by \eqref{eq:tr-terminal-scale}. Further iterations contract
the scale discrepancy, so this bound persists. The output-function weight is
uniformly bounded and included in the dense endpoint count.

\subsection{Proof of Theorem~\ref{thm:tr-construction}}
\label{app:tr-final-proof}
\begin{proof}
Fix $f_0\in\Holder$ and $K=2^m$. Take
$q_0^{\mathrm{Tr}}=q_{\rm hom}$ and
$r_0^{\mathrm{Tr}}=r_0+C_{\rm add}+55d+23$.
For each $r\ge r_0^{\mathrm{Tr}}$, choose the largest $0\le h\le m$
whose width in \eqref{eq:tr-implementation-width} is at most $r$.
Since this width is a fixed additive constant plus $3\cdot2^{dh}$,
bounded consecutive ratios give
$N_{\rm grp}=2^{dh}\asymp_{d,s}\min\{r,K^d\}$.
The constructions above fit every $q\ge q_{\rm hom}$ by repeated paired
encoding and every larger hidden width by zero padding.

All native instructions are affine residuals and ReLU hinges, homogenized
as in Lemma~\ref{lem:homogeneous}. They coexist in one FFN hidden layer.
The initialization, routing, evaluation, and terminal-correction phases are disjoint;
their inactive features vanish under the bounds verified above.
It suffices to take
\[
 M_0^{\mathrm{Tr}}=\max\{3B_{\rm eval},\,64(q_{\rm hom}+1)+2,
                        \,N,\,\sqrt{1+\varepsilon_{\rm LN}},\,2H\}.
\]
This bounds every coefficient, including the encoded tables, homogeneous
bias columns, normalization gains, and affine output function. The same fixed block
performs every phase, and all parameter slots are counted in
$B_L^{\mathrm{Tr}}(q,r)$ from \eqref{eq:tr-dense-counts}.

For $h\ge1$, routing takes $1+d\{h(L+1)+1\}=O_{d,s,N}(m^2)$ blocks;
for $h=0$ it takes only boot. Selector initialization and address scaling
cost $O_{d,s}(m)$. The $K^d/N_{\rm grp}$ Taylor records cost
$O_{d,s}((K^d/N_{\rm grp})m)$ blocks by Lemma~\ref{lem:base}.
The terminal correction costs $O_s(m^2)$ and normalization settling costs
$O_{d,s,\varepsilon_{\rm LN}}(m)$. Consequently, a common finishing time
satisfies
\begin{equation}\label{eq:trruntime}
 T_{\rm req}\le C\left\{\log^2(eK)+\frac{K^d}{r}\log(eK)\right\}.
\end{equation}
The first term also covers the saturation case of one record per group.

Let $\mathcal B$ be the union of the routing exceptional set and the
fine-cell localization strips of Appendix~\ref{app:ffn}. Then
$\mu(\mathcal B)\le CK^{-2s}$. On its complement, routing selects the
correct integer address, the selector loads a valid table, and the Taylor
output satisfies $|Hu-f_0(x)|\le CHK^{-s}$. The terminal correction adds
at most $CHK^{-s}$ for every $T\ge T_{\rm req}$, and clipping cannot increase
this error. On $\mathcal B$, the final clipped error is always at most $2H$.
For the resulting candidate $g_T$,
\[
 \|g_T-f_0\|_{L^2(\mu)}^2
 \le CH^2K^{-2s}+4H^2\mu(\mathcal B)\le CH^2K^{-2s}.
\]
The table codes, phase timers, block coefficients, and affine endpoints
were fixed using only $f_0,K,q,r$. After evaluation and terminal correction, further
iterations only contract the remaining scale discrepancy on valid inputs;
clipping remains available elsewhere. Enlarging $C$ to absorb the fixed
$H^2$ proves \eqref{eq:tr-resolution-error}--\eqref{eq:tr-resolution-time}
for every admissible $T$ with this same parameter choice.
The selection of $K$ from a prescribed iteration budget is made only in
Section~\ref{app:proofs-section3}, in the proof of
Lemma~\ref{lem:neural-index-approx}.
\end{proof}

\section{Numerical illustrations and experimental details}
\label{app:numerical-details}\label{sec:numerical}

The numerical studies illustrate the parameter-budget comparison with trained
residual FFNs. We call the fitted model with one residual update block applied
once ($T=1$) \emph{Single}; it includes the input and output functions and uses
the same block family as the other methods. \emph{Loop} repeatedly applies
one fitted block, while \emph{Untied} composes separately parameterized blocks.
Single is the common $T=1$ special case of Loop and Untied. Comparing Loop
with Single illustrates the effect of looping with shared parameters;
comparing Untied with Single illustrates the effect of greater depth with
separately parameterized blocks. Each comparison uses a common cap $B$
on all trainable input-function, block, output-function, and bias coefficients; admissible
architectures can use fewer than $B$ coefficients. Computation may differ across the three methods under the same cap.
The Transformer results are theoretical.

Sections~\ref{sec:regression-results}--\ref{sec:real-results} present
known-function regression, energy-based generative models, and real-data
prediction. Section~\ref{app:experimental-protocols} collects the
experimental protocols.

Regression is evaluated by mean squared error (MSE), classification by
predictive log loss, and density estimation by squared Hellinger distance.
For bounded Gaussian regression, Proposition~\ref{prop:model-geometry}
relates integrated squared error to Hellinger risk. Expected predictive log
loss equals a method-independent conditional entropy plus the
covariate-averaged Kullback--Leibler divergence between the true and fitted
conditional probabilities. This divergence bounds the corresponding squared
Hellinger distance from above. The multiclass tasks extend the empirical
comparison beyond the theoretical binary-response model.
These experiments describe finite-sample behavior under their fitting
procedures; the theorems concern worst-case risk of full-class sieve MLEs.\looseness=-1

\subsection{Regression model simulation}\label{sec:regression-results}
We use localized smooth bumps, a smooth recurrence, and scaled Friedman1
in dimensions two, three, and five, respectively. Inputs are uniform on
$[0,1]^d$, and $Y=f_0(X)+\varepsilon$ with independent, mean-zero Gaussian
noise of standard deviation $0.02$ for bumps and recurrence and $0.05$ for
Friedman1. Training uses squared-error loss and Adam with validation-based
tuning; test MSE is measured against the known $f_0$.
At the prespecified cap $B=128$, Loop has the smallest mean MSE for all
three families (Table~\ref{tab:numerical-simulations}). The largest relative
reductions occur for the recurrence: 46.5\% against Single and 70.4\%
against Untied. Table~\ref{tab:paired-simulation-contrasts}
reports the paired differences; uncertainty is greater for the recurrence
comparison with Untied and both Friedman1 comparisons than for smooth bumps.

\begin{table}[!htbp]
\centering
\small\setlength{\tabcolsep}{5pt}
\begin{tabular}{lrrr}
\toprule
Simulation & Single & Loop & Untied \\
\midrule
Smooth bumps & $2.9193\;(0.063)\times 10^{-3}$ & {\boldmath $2.6412\;(0.11)\times 10^{-3}$} & $3.0873\;(0.087)\times 10^{-3}$ \\
Smooth recurrence & $2.6431\;(0.18)\times 10^{-4}$ & {\boldmath $1.4133\;(0.26)\times 10^{-4}$} & $4.7724\;(2.6)\times 10^{-4}$ \\
Scaled Friedman1 & $1.9482\;(0.72)\times 10^{-3}$ & {\boldmath $1.6124\;(0.73)\times 10^{-3}$} & $3.2363\;(0.67)\times 10^{-3}$ \\
\bottomrule
\end{tabular}
\caption{MSE against known $f_0$ at $B=128$: means over eight paired replications, with standard errors (SEs) in parentheses. Boldface marks the smallest mean in each row.}
\label{tab:numerical-simulations}
\end{table}

\begin{figure}[!t]
\centering
\includegraphics[width=\textwidth]{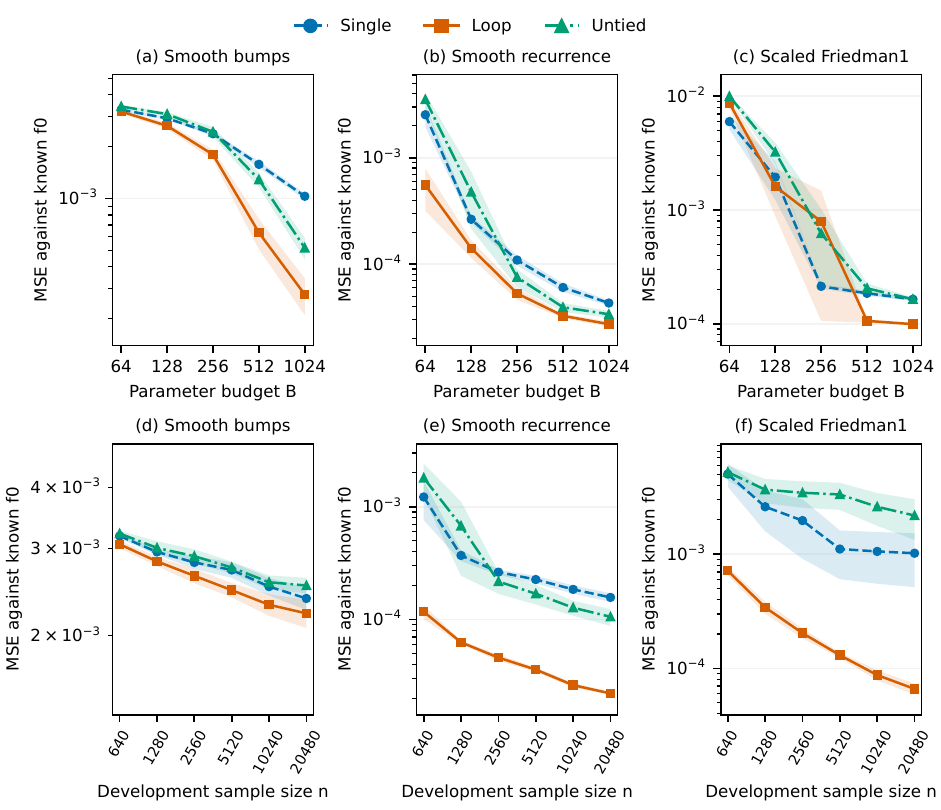}
\caption{Effects of parameter cap $B$ (top) and development sample size $n$ (bottom). Columns: smooth bumps, smooth recurrence, and scaled Friedman1. Points show mean $f_0$ MSE over eight paired replications; shading indicates one SE. Both axes are logarithmic.}
\label{fig:numerical-resources}
\end{figure}

\noindent\textbf{Effects of the parameter budget and sample size.}
Figure~\ref{fig:numerical-resources} varies the parameter cap and sample size.
The upper row uses 4096 training and 1024 validation observations, with
$B\in\{64,128,256,512,1024\}$ and $T=16$ for Loop. Loop has the smallest mean MSE
at every displayed budget for bumps and recurrence; for Friedman1,
Single performs better at $B=64,256$, and Untied also performs better at
$B=256$. The lower row fixes $B=128$ and the parameterized modules while
varying the development sample size from 640 to 20480, split 80/20 between
training and validation. Mean MSE decreases for all three methods, with
Loop attaining the smallest mean at every displayed sample size.
Each row uses its own frozen configuration and continuation protocol;
larger budgets or samples can entail more training updates.
The sample-size curves keep the iteration count fixed, whereas
Theorem~\ref{thm:fixed-budget} increases $T_n$ with $n$.
Sections~\ref{app:counts}--\ref{app:continuation} give the configurations
and update limits.

\noindent\textbf{Loop iterations with continued training.}
Figure~\ref{fig:numerical-iterations} increases Loop's iteration count from
one to 64 at $B=128$, with fixed state and hidden widths and continued
training. Mean MSE falls by 36.9\%, 80.6\%, and 41.6\% for bumps,
recurrence, and Friedman1, respectively. The recurrence gains mainly from
the first increase and then fluctuates around a lower plateau; the other
two curves decrease throughout. These curves reflect joint increases in
iteration count and cumulative training, with a fixed parameter count.

\begin{figure}[!htbp]
\centering
\includegraphics[width=\textwidth]{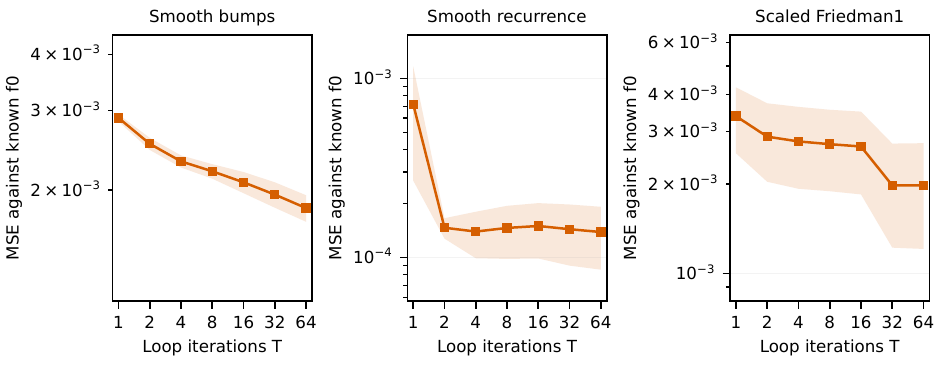}
\caption{Loop iterations with continued training at $B=128$. Curves show mean $f_0$ MSE with one-SE shading over eight replications, using the same seeds at every $T$. Both axes are logarithmic.}
\label{fig:numerical-iterations}
\end{figure}

\subsection{Energy-based generative-model simulation}\label{sec:ebm-results}
We next evaluate density estimation with a fixed, smooth four-mode target
distribution in the energy-based family, constructed independently of the
fitted architectures. Using three
development seeds, we screen parameter caps
$B\in\{96,128,160,192\}$ and training sample sizes
$n_{\rm tr}\in\{1024,2048,4096\}$.  Within each setting and method,
validation negative log likelihood (NLL) selects the architecture, optimizer
setting, and checkpoint.
A setting enters confirmation only if Loop has the smallest mean development
squared Hellinger error and wins at least two of the three paired seeds.  Before any
confirmation sample is generated, this rule freezes the single strongest
eligible setting, $B=160$ and $n_{\rm tr}=2048$. Confirmation evaluates the selected setting on fresh samples.\looseness=-1

\begin{table}[!htbp]
\setstretch{1}
\centering
\small
\caption{Energy-based simulation at the development-selected setting $B=160$, $n_{\rm tr}=2048$: mean (SE) over eight paired confirmation replications. Squared Hellinger errors use midpoint quadrature. Boldface marks the smallest mean in each row.}
\label{tab:ebm-budget-confirmation}
\begin{tabular}{lrrr}
\toprule
Metric & Single & Loop & Untied \\
\midrule
Hellinger$^2$ & $0.01511\;(0.0020)$ & $\mathbf{0.008919\;(0.0014)}$ & $0.01339\;(0.0017)$ \\
Test NLL & $-0.03623\;(0.0040)$ & $\mathbf{-0.04884\;(0.0027)}$ & $-0.03996\;(0.0033)$ \\
\bottomrule
\end{tabular}
\end{table}

Loop's mean squared Hellinger error is 0.00892, compared with 0.01511 for
Single and 0.01339 for Untied, a reduction of 33.4\% relative to the
second-lowest mean at the selected operating point. The test NLL means have
the same ordering. In
Figure~\ref{fig:ebm-budget-confirmation}, the top row averages the eight
fitted densities for each method.  The bottom row pools 4096 exact
rejection draws from each fitted model, so every method panel contains 32768
generated observations.  Analytic-density panels share one color scale, and
sample panels use one fixed histogram, smoothing bandwidth, and color scale.

\begin{figure}[t]
\centering
\includegraphics[width=\textwidth]{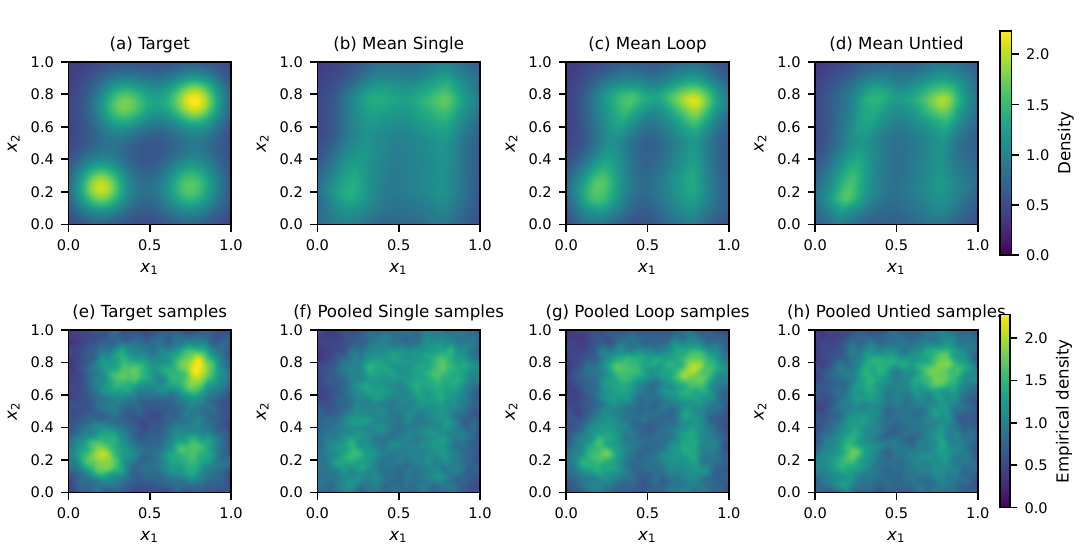}
\caption{Four-mode density estimation at $B=160$, $n_{\rm tr}=2048$. Top: target and mean fitted densities. Bottom: smoothed empirical densities from 32768 target draws or equally pooled exact rejection samples from the eight fitted models per method.}
\label{fig:ebm-budget-confirmation}
\end{figure}

The target specification, integration grids, and fitting settings are given
in Section~\ref{app:ebm-protocols}.

\subsection{Real-data experiments}\label{sec:real-results}
We next examine how parameter sharing affects prediction on real data.
We compare Single, Loop, and Untied residual FFNs on nine regression and
four multiclass classification data sets, using the same parameter cap for
the three methods within each task. We evaluate regression by test MSE on the original response scale
and classification by predictive log loss.

Table~\ref{tab:real-comparisons} summarizes the results. Loop has the smallest displayed mean on seven of the nine regression tasks and
three of the four classification tasks. The clearest mean margins occur for
SGEMM and RT-IoT. Several other orderings are small relative to the reported
method-specific SEs: for example, the Wall Robot means for Loop and Single are
0.5673 and 0.5680, with SEs 0.029 and 0.036.

\begin{table}[!htbp]
\setstretch{1}
\centering
\small
\setlength{\tabcolsep}{3.7pt}
\caption{Real-data test performance: mean (SE) over eight paired replications. Boldface marks the smallest mean in each row.}
\label{tab:real-comparisons}
\resizebox{\linewidth}{!}{%
\begin{tabular}{llrrrr}
\toprule
Source & Metric & $B$ & Single & Loop & Untied \\
\midrule
\multicolumn{6}{l}{\textit{Regression}} \\
Airfoil & MSE & 64 & $9.5361\;(0.17)$ & $\mathbf{8.9303\;(0.23)}$ & $9.7895\;(0.27)$ \\
Power plant & MSE & 32 & $17.937\;(0.10)$ & $18.146\;(0.36)$ & $\mathbf{17.916\;(0.095)}$ \\
Fish toxicity & MSE & 32 & $0.94739\;(0.024)$ & $\mathbf{0.94596\;(0.015)}$ & $0.95003\;(0.026)$ \\
White wine & MSE & 64 & $\mathbf{0.59428\;(0.0066)}$ & $0.59577\;(0.0063)$ & $0.59469\;(0.0066)$ \\
Concrete & MSE & 128 & $36.730\;(0.81)$ & $\mathbf{35.633\;(0.85)}$ & $36.944\;(0.42)$ \\
California housing & MSE & 64 & $0.36494\;(0.0017)$ & $\mathbf{0.36079\;(0.0024)}$ & $0.36365\;(0.0012)$ \\
Abalone & MSE & 64 & $4.4706\;(0.016)$ & $\mathbf{4.4297\;(0.019)}$ & $4.4652\;(0.0098)$ \\
SGEMM & MSE & 64 & $9.2061\;(0.83)\times10^{3}$ & $\mathbf{6.2200\;(0.15)\times10^{3}}$ & $8.4769\;(0.72)\times10^{3}$ \\
SARCOS & MSE & 512 & $14.765\;(0.22)$ & $\mathbf{14.005\;(0.19)}$ & $14.224\;(0.21)$ \\
\multicolumn{6}{l}{\textit{Multiclass classification}} \\
Wall Robot & log loss & 1024 & $0.5680\;(0.036)$ & $\mathbf{0.5673\;(0.029)}$ & $0.6283\;(0.044)$ \\
Maternal Risk & log loss & 512 & $0.7294\;(0.015)$ & $\mathbf{0.7201\;(0.013)}$ & $0.7280\;(0.013)$ \\
RT-IoT & log loss & 4096 & $0.04230\;(0.0016)$ & $\mathbf{0.03782\;(0.0012)}$ & $0.04095\;(0.0011)$ \\
Room Occupancy & log loss & 1024 & $\mathbf{0.04824\;(0.023)}$ & $0.04968\;(0.027)$ & $0.05580\;(0.027)$ \\
\bottomrule
\end{tabular}
}
\end{table}

Table~\ref{tab:paired-real-contrasts} summarizes regression comparisons
using the paired replications. The SGEMM differences are large relative to their paired SEs;
several other contrasts at the selected budgets have substantial uncertainty.

Section~\ref{app:real-data-protocols} gives the data splits, preprocessing,
and fitting protocols.

\subsection{Experimental protocols}\label{app:experimental-protocols}
This section collects the simulation functions, architectures, fitting
procedures, and data sources for the results in Sections~\ref{sec:regression-results}--\ref{sec:real-results}.

\subsubsection{Simulation functions, data streams, and error metric}\label{app:simulation-protocols}
For each family, inputs are independent and uniform on $[0,1]^d$.
Training and validation responses are $Y=f_0(X)+\varepsilon$, with
independent Gaussian noise. The three functions are as follows.

\paragraph{Smooth bumps}
Let $k=3$, $a_0=0.75$, and let $c_j\in\{-1,1\}$ be independent
equiprobable signs for the nine cells $j\in\{0,1,2\}^2$.
For $x\in[0,1]^2$, put $j_\ell=\min\{\lfloor kx_\ell\rfloor,k-1\}$
and $t_\ell=4(kx_\ell-j_\ell-1/2)$. Define
\[
 \phi(t)=
 \begin{cases}
 \exp\!\left(\displaystyle\sum_{\ell=1}^2
                 \left[1-\frac{1}{1-t_\ell^2}\right]\right),
           & |t_1|<1,\ |t_2|<1,\\
 0,        & \text{otherwise},
 \end{cases}
 \qquad
 f_0(x)=\frac{a_0}{k}\,c_j\phi(t).
\]
The compact bumps are smooth, with supremum bound $0.25$ and
Lipschitz bound $0.75(32/e)\sqrt{2}$. The noise standard deviation (SD) is $0.02$.
The cell count $k$ is unrelated to the model's dense parameter count.

\paragraph{Smooth recurrence}
For $x\in[0,1]^3$, set $u_0=x_3-1/2$ and iterate
\[
 u_{j+1}=0.25u_j+0.1\tanh\{2u_j+3(x_1-1/2)\}
                +0.08\sin\{2\pi(x_2+u_j)\},
 \qquad j=0,\ldots,11.
\]
The target is $f_0(x)=u_{12}$ and the noise SD is $0.02$.
The state remains in $[-1/2,1/2]$. A conservative Lipschitz bound is
\[
 1+\frac{\sqrt{0.3^2+(0.16\pi)^2}}
              {1-(0.45+0.16\pi)}.
\]

\paragraph{Scaled Friedman1}
For $x\in[0,1]^5$, the target is
\[
 f_0(x)=0.5\sin(\pi x_1x_2)+(x_3-1/2)^2+0.5x_4+0.25x_5.
\]
This is \href{https://scikit-learn.org/stable/modules/generated/sklearn.datasets.make_friedman1.html}{Friedman1}
divided by 20. Its supremum is bounded by $1.5$, its Lipschitz constant
by $\sqrt{2(\pi/2)^2+1+0.25+0.0625}$, and the noise SD is $0.05$.
Each family's displayed supremum-plus-Lipschitz bound is below 32,
consistent with the $s=1$ smoothness setting.

For $B$ and $T$, each seed spawns separate random streams for function
coefficients, training inputs, validation inputs, test inputs, and
the corresponding noises. A family and seed share nested input
prefixes across resource values. The $n$ study retains its own frozen
random-stream implementation, likewise using nested fitting prefixes
and a fixed independent test sample within each seed. Each model is
evaluated on $N_{\rm test}=32768$ independent uniform test inputs using
\[
 \widehat{\operatorname{MSE}}_{f_0}
  =\frac{1}{N_{\rm test}}\sum_{i=1}^{N_{\rm test}}
     \{\widehat f(X_i^{\rm test})-f_0(X_i^{\rm test})\}^2.
\]
Thus synthetic test scores exclude response noise. All means are
arithmetic across eight paired replications, one per seed. The reported SE is
the sample SD, computed with denominator seven, divided by $\sqrt{8}$.
Plot shading shows the mean plus or minus one SE, and table entries use
mean (SE).
Target centering and scaling use only the current training responses;
test predictions are transformed back to the original target scale.
Synthetic inputs remain on their original uniform scale.

\subsubsection{Architectures and dense coefficient accounting}\label{app:counts}
The input function is affine from $\mathbb R^d$ to $\mathbb R^q$. Regression and
density-estimation models use an affine scalar output function. A $C$-class model uses
an affine output function from $\mathbb R^q$ to $\mathbb R^C$, followed by the softmax
map to class probabilities. Every residual block has the form
\[
 \Phi_\theta(z)=z+W_2\ReLU(W_1z+b_1)+b_2,
 \qquad W_1\in\mathbb R^{r\times q},\quad
 W_2\in\mathbb R^{q\times r}.
\]
Single executes one block once. Loop executes one learned block $T$
times. Untied executes $T$ independently parameterized blocks of
common dimensions. The networks contain no attention or normalization
layers. For a scalar output, the actual dense coefficient counts are
\begin{align*}
 P_{\rm Single}=P_{\rm Loop}
   &=q(d+2)+1+(2q+1)r+q,\\
 P_{\rm Untied}
   &=q(d+2)+1+T\{(2q+1)r+q\}.
\end{align*}
For a $C$-class output, the corresponding counts are
\begin{align*}
 P^{(C)}_{\rm Single}=P^{(C)}_{\rm Loop}
   &=q(d+1)+C(q+1)+(2q+1)r+q,\\
 P^{(C)}_{\rm Untied}
   &=q(d+1)+C(q+1)+T\{(2q+1)r+q\}.
\end{align*}
Counts include every trainable dense matrix entry and bias, even when
initialized to zero. A cap $B$ is an upper bound on these counts, while
iteration count and total optimization cost are tracked separately. Table~\ref{tab:numerical-configurations}
lists the frozen simulation configurations. The smooth-bump modules
differ between the $B$, $T$, and $n$ studies; the two rows of
Figure~\ref{fig:numerical-resources} therefore retain distinct protocols.

\begin{table}[!ht]
\centering
\small\setlength{\tabcolsep}{3pt}
\textbf{(a) Simulation protocols}\par\medskip
\resizebox{\linewidth}{!}{%
\begin{tabular}{>{\raggedright\arraybackslash}p{10mm}>{\raggedright\arraybackslash}p{29mm}>{\raggedright\arraybackslash}p{36mm}>{\raggedright\arraybackslash}p{32mm}>{\raggedright\arraybackslash}p{36mm}}
\toprule
Study & Parameter cap & Training / validation observations & Loop iterations & Maximum cumulative updates \\
\midrule
$B$ & $64,128,256,$ $512,1024$ & $4096/1024$ & $16$ & $36000$ \\
$T$ & $128$ & $4096/1024$ & $1,2,4,8,$ $16,32,64$ & $60000$ \\
$n$ & $128$ & $n_{\rm tr}=2^9,\ldots,2^{14}$; $n_{\rm val}=n_{\rm tr}/4$ & $16$ & $48000$ \\
\bottomrule
\end{tabular}
}
\par\bigskip
\textbf{(b) Module dimensions and actual dense parameter counts}\par\medskip
\resizebox{\linewidth}{!}{%
\begin{tabular}{>{\raggedright\arraybackslash}p{10mm}>{\raggedright\arraybackslash}p{31mm}>{\raggedright\arraybackslash}p{34mm}>{\raggedright\arraybackslash}p{34mm}>{\raggedright\arraybackslash}p{34mm}}
\toprule
Study & Simulation & Single $(q,r,T,P)$ & Loop $(q,r,T,P)$ & Untied $(q,r,T,P)$ \\
\midrule
$B$ & Smooth bumps & $(4,11,1,120)$ & $(4,11,16,120)$ & $(6,2,3,121)$ \\
$B$ & Smooth recurrence & $(3,15,1,124)$ & $(2,23,16,128)$ & $(2,5,4,119)$ \\
$B$ & Scaled Friedman1 & $(4,10,1,123)$ & $(5,7,16,118)$ & $(4,3,3,122)$ \\
\midrule
$T$ & Smooth bumps & \textemdash & $(3,16,T,128)$ & \textemdash \\
$T$ & Smooth recurrence & \textemdash & $(2,23,T,128)$ & \textemdash \\
$T$ & Scaled Friedman1 & \textemdash & $(5,7,T,118)$ & \textemdash \\
\midrule
$n$ & Smooth bumps & $(3,16,1,128)$ & $(2,23,16,126)$ & $(4,3,3,110)$ \\
$n$ & Smooth recurrence & $(3,15,1,124)$ & $(2,23,16,128)$ & $(2,5,4,119)$ \\
$n$ & Scaled Friedman1 & $(4,10,1,123)$ & $(5,7,16,118)$ & $(4,3,3,122)$ \\
\bottomrule
\end{tabular}
}

\caption{Frozen simulation configurations. Here $q$ is the state dimension,
$r$ the residual hidden width, $T$ the iteration count, and $P$ the actual
dense coefficient count. Panel (b) gives the $B$ study at its cap-128 anchor
and the fixed modules for the $T$ and $n$ studies; in the $T$-study rows,
$T\in\mathcal T=\{1,2,4,8,16,32,64\}$. Maximum cumulative updates refer to
one selected checkpoint ancestry, including initialization stages. Per-fit
caps and continuation procedures are specified below.}
\label{tab:numerical-configurations}
\end{table}

\subsubsection{Validation selection and common optimization settings}
The $B$ and $T$ studies use diagnostic seeds 7700--7701 and confirmation
seeds 7800--7807. For smooth bumps, each learner and curve receives
three candidate state dimensions, obtained by adding $0,1,2$ to its
development dimension. One dimension is chosen by validation loss
averaged over both diagnostic seeds and all values of that curve.
The recurrence and Friedman1 retain their development dimensions.
These choices are frozen before confirmation, and test evaluation begins at
the confirmation stage. The cap-$128$ table rule is also recorded
before confirmation test scoring; its models are exactly those used
at $B=128$ in the upper row of Figure~\ref{fig:numerical-resources}.
The $n$ study retains its separate diagnostic at seeds 7501--7502
and its confirmation seeds 7600--7607.

At the first stage, each learner evaluates two initializations and
four hyperparameter combinations: learning rate $0.001$ or $0.003$
and weight decay $0$ or $0.001$. At every subsequent stage, the four
settings start from the preceding validation-selected checkpoint.
Selection uses the current validation MSE and includes the step-zero
checkpoint. Each fit receives its full update allowance; the selected
checkpoint can occur earlier. Adam is reset for every fit, with
cosine learning-rate decay to 5\% of its starting value, batch size
128, and validation evaluation every 200 updates. Gradient norm is
capped at five; coefficients are clamped to $[-128,128]$ after each
update. Predictions are clipped to $[-8,8]$ on the training-target
standardized scale. Computation uses deterministic 32-bit floating-point
operations on a central processing unit (CPU).

For $B$ and $T$, the outgoing residual matrix is initialized with multiplier
$0.05/T$ and its bias is zero. The outgoing matrix and bias use Adam
learning rate equal to the base rate divided by the architecture's
actual iteration count; other layers use the base rate. This rule
applies to all three learners. The independently frozen $n$ protocol
uses multiplier $0.05$ and the base rate uniformly across parameter groups.

\subsubsection{Continuation paths and cumulative update budgets}\label{app:continuation}
\paragraph{Parameter-cap path}
The displayed caps are $64,128,256,512,1024$, with 4096 training and
1024 validation observations and 16 Loop iterations. Each path
first trains at cap 64 using 512 training and 128 validation observations,
then uses the full fitting sample at cap 64 before proceeding to
larger caps. State dimensions are fixed. Hidden width is the largest
positive integer satisfying the cap; baseline depths follow the
development allocation, reduced only when needed for feasibility at
the initial cap. Increasing width pads the old weights and sets new
outgoing coefficients to zero; additional untied blocks are initialized
as identities. This embedding preserves the previous function.
When training-target normalization changes, the output function is remapped
to preserve the unclipped predictions on the original response scale. The
clipped predictions are also preserved wherever clipping is inactive under
both normalizations. Each fit has at most
6000 updates. The longest selected checkpoint ancestry has at most
18000 updates at the cap-$128$ table point and 36000 at cap 1024.

\begin{table}[!htbp]
\setstretch{1}
\centering
\small
\begin{tabular}{lrr}
\toprule
Scenario & Loop minus Single & Loop minus Untied \\
\midrule
Smooth bumps & $-2.78\;(0.98)\times10^{-4}$ & $-4.46\;(0.61)\times10^{-4}$ \\
Smooth recurrence & $-1.23\;(0.15)\times10^{-4}$ & $-3.36\;(2.8)\times10^{-4}$ \\
Scaled Friedman1 & $-3.36\;(5.1)\times10^{-4}$ & $-1.62\;(1.1)\times10^{-3}$ \\
\bottomrule
\end{tabular}
\caption{Paired MSE differences at the prespecified parameter cap 128 for the simulations in Table~\ref{tab:numerical-simulations}. Entries are mean (SE) of the eight within-seed differences, with seeds paired across methods. Negative values favor Loop.}
\label{tab:paired-simulation-contrasts}
\end{table}

\paragraph{Iteration path}
The iteration grid is $1,2,4,8,16,32,64$ at cap 128. Loop's fixed
$(q,r)$ values are $(3,16)$, $(2,23)$, and $(5,7)$ for the three
families. Each path begins at $T=1$ with 512 training and 128 validation
observations, continues at $T=1$ with 4096 training and 1024 validation
observations, then increases iterations.
On changing $T$, the outgoing residual coefficients are multiplied
by $T_{\rm old}/T_{\rm new}$. For nonlinear blocks this provides an
approximate continuation of the preceding function. Fits at $T\le16$ receive 6000 updates each; those at 32 and
64 receive 12000 each. The longest selected path at $T=64$ is bounded
by 60000 cumulative updates. The corresponding Single and Untied fits use
their fixed architecture depths but receive the same stage counts,
hyperparameter counts, and per-stage update caps. Figure~\ref{fig:numerical-iterations}
reports Loop's performance as iteration count and cumulative training
budget increase together.

\paragraph{Sample-size path}
At fixed cap 128 and the modules in
Table~\ref{tab:numerical-configurations}, the training sizes are
$512,\allowbreak 1024,\allowbreak 2048,\allowbreak 4096,\allowbreak 8192,\allowbreak 16384$. Validation size is one quarter of
training size, so the plotted total sample sizes are
$640,\allowbreak 1280,\allowbreak 2560,\allowbreak 5120,\allowbreak 10240,\allowbreak 20480$. All three learners use the same
nested fitting samples within each seed. At larger samples, the
output function is remapped for the new training-target normalization,
preserving unclipped original-scale predictions before further fitting.
Per-fit caps are 6000 updates for training sizes at most 4096 and
12000 for 8192 and 16384. The longest selected path therefore has
at most 48000 updates. This experiment uses its independently frozen modules
and optimization rule.

Maximum cumulative updates sum the upper bounds along one selected
checkpoint ancestry, including initialization stages. They exclude
rejected hyperparameter and initialization candidates. Actual selected
paths can be shorter because validation may select an earlier
checkpoint. Each individual fit is capped at 12000 updates. Total computation
over all candidate fits is a separate resource recorded in the logs.

\subsubsection{Energy-based generative-model simulation protocol}\label{app:ebm-protocols}
The experiment has support $[0,1]^2$. The target log-density index, equal to
the negative energy under the convention of \eqref{eq:ebm-model}, is a sum of
four broad Gaussian radial basis terms
centered at $(0.20,0.23)$, $(0.74,0.24)$, $(0.34,0.73)$, and $(0.78,0.76)$,
with scales between 0.16 and 0.18, plus a shallow central depression.  Its
index lies in $[0.0634,1.814]$ on a $512^2$ reference grid, inside the
bound $[-3,3]$. Training updates compute the log normalizer on a fixed
$32\times32$ midpoint grid; validation NLL uses a $128\times128$ midpoint
grid. The development screen crosses
$B\in\{96,128,160,192\}$ with
$n_{\rm tr}\in\{1024,2048,4096\}$ using seeds 9000--9002, while holding the
target fixed.  Candidate state dimensions are $q\in\{4,8,12\}$, Loop
iterations are $T\in\{2,4,8\}$, and Untied depths are
$T\in\{2,4\}$.  Learning rates are $\{0.001,0.003\}$ and weight decays are
$\{0,0.001\}$ for every method.  Development uses 300 warm and 2500
refinement updates.

The development screen retains settings where Loop has the smallest mean
development squared Hellinger error and at least two of three paired wins, then ranks
eligible settings by relative advantage. The top-ranked setting,
$B=160,n_{\rm tr}=2048$, is frozen before confirmation.
Confirmation seeds 9100--9107 use 2048 training draws, 2048 validation draws,
65536 test draws, 800 warm updates, and 8000 refinement updates.  Development
keeps three structures per method.  Within confirmation, validation NLL
selects structure, learning rate, weight decay, and checkpoint before test
samples or Hellinger scores are computed. Test squared Hellinger error is
approximated on a $256\times256$ midpoint grid. The scenario and
model-selection rules are fixed before confirmation. Sample-panel pooling and smoothing
affect only visualization.

\subsubsection{Real-data protocols}\label{app:real-data-protocols}
Table~\ref{tab:real-data-dimensions} records the dimensions and parameter
caps for exactly the real-data results displayed in
Table~\ref{tab:real-comparisons}. The dense coefficient formulas in
Section~\ref{app:counts} determine which architectures satisfy each cap $B$, using the scalar output function
for regression and the $C$-logit output function for classification.

\begingroup
\footnotesize
\setlength{\tabcolsep}{4pt}
\begin{longtable}{lrrrr}
\caption{Dimensions and parameter caps for the real-data experiments in
Table~\ref{tab:real-comparisons}.  A dash in the class column denotes a
regression task.}\label{tab:real-data-dimensions} \\
\toprule
Data set & Observations & $d$ & Classes & $B$ \\
\midrule
Airfoil & 1,503 & 5 & -- & 64 \\
Power plant & 9,568 & 4 & -- & 32 \\
Fish toxicity & 908 & 6 & -- & 32 \\
White wine & 4,898 & 11 & -- & 64 \\
Concrete & 1,030 & 8 & -- & 128 \\
California housing & 20,640 & 8 & -- & 64 \\
Abalone & 4,177 & 10 & -- & 64 \\
SGEMM & 241,600 & 14 & -- & 64 \\
SARCOS & 48,933 & 21 & -- & 512 \\
Wall Robot & 5,456 & 24 & 4 & 1,024 \\
Maternal Risk & 1,014 & 6 & 3 & 512 \\
RT-IoT & 30,000 & 93 & 12 & 4,096 \\
Room Occupancy & 10,129 & 19 & 4 & 1,024 \\
\bottomrule
\end{longtable}
\endgroup

Each regression task uses a test set fixed across methods and replications.
SARCOS retains its 4,449 official test observations, with matching input groups
removed from the training pool; the other tasks hold out 20\% of input groups.
For each seed, the remaining groups are split 75/25 for training and validation.  Imputation and standardization are fitted on the training
split only. Source-specific caps were selected during development by maximizing
the smaller relative validation improvement of Loop over Single and Untied.
Caps and candidate shortlists were fixed before confirmation.
The development candidate state dimensions are
$q\in\{1,2,3,4,6,8,14\}$; Loop uses $T\in\{2,4,8,16\}$ and Untied uses depths
$\{1,2,3,4,8,16\}$. For every candidate state dimension and iteration count,
the residual hidden width $r$ is the largest positive integer whose full dense
coefficient count does not exceed the row's cap. Confirmation uses three
shortlisted architectures per method and data set, except for Fish toxicity
(Single: two), Concrete (Single and Loop: two each), and SARCOS (Untied: eight).
For every shortlisted architecture, the learning rate is selected
from $\{0.001,0.003\}$ and weight decay from $\{0,0.001\}$.  Training uses a
batch size of 128, 1,000 warm updates, and at most 6,000 refinement updates,
with validation every 100 updates.  Adam uses cosine learning-rate decay to 5\%
of its initial value, gradient-norm clipping at 5, a dense-coefficient clamp to
$[-128,128]$, and a prediction clip of 8 on the standardized response scale.
Within each frozen shortlist, validation MSE selects the architecture,
optimizer setting, and checkpoint.
Table~\ref{tab:real-comparisons} reports test MSE on the original response scale
as mean (SE) over eight paired confirmation replications, corresponding to
seeds 8200--8207.

\begin{table}[!htbp]
\setstretch{1}
\centering
\small
\begin{tabular}{lrr}
\toprule
Data set & Loop minus Single & Loop minus Untied \\
\midrule
Airfoil & $-0.606\;(0.32)$ & $-0.859\;(0.32)$ \\
Power plant & $0.209\;(0.3)$ & $0.231\;(0.31)$ \\
Fish toxicity & $-1.43\;(18)\times10^{-3}$ & $-4.07\;(21)\times10^{-3}$ \\
White wine & $1.49\;(3.6)\times10^{-3}$ & $1.08\;(3.3)\times10^{-3}$ \\
Concrete & $-1.1\;(0.69)$ & $-1.31\;(1.1)$ \\
California housing & $-4.15\;(2.7)\times10^{-3}$ & $-2.86\;(2.8)\times10^{-3}$ \\
Abalone & $-0.0409\;(0.012)$ & $-0.0355\;(0.016)$ \\
SGEMM & $-2.99\;(0.76)\times10^{3}$ & $-2.26\;(0.63)\times10^{3}$ \\
SARCOS & $-0.761\;(0.11)$ & $-0.219\;(0.15)$ \\
\bottomrule
\end{tabular}
\caption{Paired regression MSE differences for the validation-selected budgets in Table~\ref{tab:real-comparisons}. Entries are mean (SE) of the eight within-seed differences. Negative values favor Loop.}
\label{tab:paired-real-contrasts}
\end{table}

For each classification task, a class-stratified split holds out approximately
20\% of input groups for testing, fixed across methods and replications.
For each seed, the remaining groups are split 75/25 for training and validation;
preprocessing uses training data only. Exact duplicate feature rows remain
in one split group.  Wall
Robot uses all 5,456 observations and 24 sensor variables, with contiguous
blocks of 60 readings grouped together.  Room Occupancy uses all 10,129 rows;
the date and time variables are converted to a day ordinal and daily sine and
cosine terms, and contiguous blocks of 120 rows define groups.  Maternal Risk
uses all 1,014 rows and groups duplicate feature rows.  RT-IoT uses a
prespecified stratified sample of 30,000 rows (sampling seed 20261101) and also
groups exact duplicates.

For Wall Robot, development seeds 9700--9703 and confirmation seeds 9800--9807
use candidate state dimensions $q\in\{2,4,6,8,12,16,24\}$, Loop iterations
$T\in\{2,4,8\}$, and Untied depths $\{2,3,4\}$.  Development uses 300 warm and
2,200 refinement updates; confirmation uses 600 warm and 5,000 refinement
updates.  For Room Occupancy, Maternal Risk, and RT-IoT, development seeds
9900--9903 and confirmation seeds 10000--10007 use
$q\in\{2,4,8,12,16\}$, the same Loop and Untied depth grids, and the same
maximal feasible hidden-width rule, with 300 warm plus
1,800 refinement updates in development, and 600 warm plus 4,500 refinement
updates in confirmation.  Both protocols use batch size 256, validation every
100 updates, learning rates $\{0.001,0.003\}$, and weight decays
$\{0,0.001\}$.  Within each method, validation log loss selects the candidate structure,
optimizer setting, and checkpoint. Table~\ref{tab:real-comparisons}
reports held-out predictive log loss. Each table
entry is mean (SE) across eight paired confirmation replications, one per seed.

\subsubsection{Measured-source links}
\begin{itemize}
\item Airfoil: \url{https://archive.ics.uci.edu/dataset/291}.
\item Power plant: \url{https://archive.ics.uci.edu/dataset/294/combined+cycle+power+plant}.
\item Fish toxicity: \url{https://archive.ics.uci.edu/dataset/504}.
\item White wine: \url{https://archive.ics.uci.edu/dataset/186}.
\item Concrete: \url{https://archive.ics.uci.edu/dataset/165/concrete+compressive+strength}.
\item California housing: \url{https://scikit-learn.org/stable/modules/generated/sklearn.datasets.fetch_california_housing.html}.
\item Abalone: \url{https://archive.ics.uci.edu/dataset/1/abalone}.
\item SGEMM: \url{https://archive.ics.uci.edu/dataset/440/sgemm+gpu+kernel+performance}.
\item SARCOS: \url{https://gaussianprocess.org/gpml/data/}.
\item Wall Following Robot Navigation: \url{https://archive.ics.uci.edu/dataset/194/wall+following+robot+navigation+data}.
\item Room Occupancy Estimation: \url{https://archive.ics.uci.edu/dataset/864/room+occupancy+estimation}.
\item Maternal Health Risk: \url{https://archive.ics.uci.edu/dataset/863/maternal+health+risk}.
\item RT-IoT2022: \url{https://archive.ics.uci.edu/dataset/942/rt+iot2022}.
\end{itemize}

% End numerical_details

\end{document}